\RequirePackage{fix-cm}
\documentclass[]{article}
\usepackage{arxiv}
\usepackage[english]{babel}
\usepackage{lipsum}
\usepackage[export]{adjustbox}
\usepackage{graphicx} 
\usepackage{subcaption}
\usepackage{comment}
\usepackage{verbatim}
\usepackage{algorithm}
\usepackage{algpseudocode}
\usepackage{makecell}
\usepackage[numbers,sort&compress]{natbib}

\usepackage{placeins}
\usepackage[table]{xcolor}
\usepackage{bm}
\definecolor{mygray}{gray}{0.9} 
\usepackage{booktabs} 
\usepackage{csquotes}
\usepackage{graphicx}
\usepackage{multirow}
\usepackage{listings}
\usepackage{xcolor}
\definecolor{abstract_background}{RGB}{235,235,235}
\usepackage{float}
\usepackage{color, colortbl}
\definecolor{ao(english)}{rgb}{0.0, 0.5, 0.0}
\usepackage{xr-hyper}
\usepackage{hyperref}

\hypersetup{colorlinks,linkcolor={ao(english)},citecolor={ao(english)},urlcolor={red}} 
\usepackage{stackrel,amssymb}
\usepackage{amsfonts,amsmath,amstext,amsbsy,amsthm}

\theoremstyle{definition}
\newtheorem{prop}{Proposition}

\newtheorem{remark}{Remark}[section]
\usepackage{cleveref}
\usepackage{amsmath}
\definecolor{LightCyan}{rgb}{0.88,1,1}
\definecolor{LightRed}{rgb}{1,0.7,0.7}

\def\bu{\bm{u}}

\def\bx{\bm{x}}

\newcommand{\R}{\mathbb{R}}

\DeclareMathOperator*{\argmin}{argmin}

\title{Next Generation Equation-Free Multiscale Modelling of Crowd Dynamics via Machine Learning}

\author{ \textbf{Hector Vargas Alvarez\textcolor{teal}{$^{1}$},} \textbf{Dimitrios G. Patsatzis\textcolor{teal}{$^{1}$},}
\textbf{Lucia Russo\textcolor{teal}{$^{2}$},}
\textbf{Ioannis Kevrekidis\textcolor{teal}{$^{3}$},}
\textbf{Constantinos Siettos\textcolor{teal}{$^{4,}$}\thanks{Corresponding author, email: \texttt{constantinos.siettos@unina.it}}\hspace{0.1cm}} \\
{}\\
\textcolor{teal}{$^{(1)}$}Modelling Engineering Risk and Complexity, \emph{Scuola Superiore Meridionale}, Naples 80138, Italy \hspace{1cm}\\
\textcolor{teal}{$^{(2)}$}Institute of Science and Technology for Energy and Sustainable Mobility, \emph{Consiglio Nazionale delle} \\ \hspace{0.39cm}\emph{Ricerche}, Naples 80125, Italy 
\hspace{1cm}\\
\textcolor{teal}{$^{(3)}$}Department of Chemical and Biomolecular Engineering, Department of Applied Mathematics and \\
\hspace{0.39cm}Statistics \& Department of Urology, \emph{Johns Hopkins University}, Baltimore, MD 21218, USA 
\hspace{1cm}\\
\textcolor{teal}{$^{(4)}$}Department of Mathematics and Applications ‘‘Renato Caccioppoli", \emph{Universit\`a degli Studi di Napoli}\\ \hspace{0.39cm}\emph{Federico II,} Naples 80126, Italy }

\usepackage{amsopn}
\DeclareMathOperator{\diag}{diag}

\begin{document}

\maketitle

\begin{abstract}
\colorbox{abstract_background}{\begin{minipage}{1\linewidth}
Bridging the microscopic/individual and macroscopic/emergent modelling scales in crowd dynamics constitutes an important, open challenge for systematic numerical analysis, optimization, and control. Here, we propose a manifold-informed machine learning approach to learn the discrete evolution operator for the emergent/collective crowd dynamics in latent spaces from high-fidelity individual/agent-based simulations. The proposed framework builds upon our previous works on the development of \textit{next-generation Equation-free algorithms} for dealing with the ``curse of dimensionality" when learning surrogate models for high-dimensional and multiscale systems. The proposed approach is a four-stage one, \textit{explicitly conserving the mass} of the reconstructed dynamics in the high-dimensional space. In the first step, we derive continuous macroscopic fields (densities) from discrete microscopic data (pedestrians' positions) using Kernel Density Estimation. In the second step, based on manifold learning, we construct a map from the density-field space into an appropriate latent space parametrized by a few coordinates based on Proper-Orthogonal Decomposition (POD) of the corresponding density distributions. The third step involves learning reduced-order surrogate models in the latent space using machine learning techniques, particularly Long Short-Term Memory networks and Multivariate Autoregressive models. Finally, we reconstruct the crowd dynamics in the high-dimensional space with POD, demonstrating that the POD reconstruction conserves the mass. Thus, with this ``embed $\rightarrow$ learn in latent space $\rightarrow$ lift back to the high-dimensional space'' pipeline, we create an effective solution operator of the unavailable (at the macroscopic scale) Partial Differential Equation for the evolution of the density distribution. For our illustrations, we used the Social Force Model to generate data in a corridor with an obstacle, imposing periodic boundary conditions in two scenarios: (i) an unidirectional flow, and (ii) a counterflow. The numerical results demonstrate high accuracy, robustness, and generalizability, thus allowing for fast and accurate modelling/simulation of crowd dynamics from agent-based simulations.
\end{minipage}
}
\end{abstract}


\section{Introduction}
\label{sec:intro}
An important open challenge in the modelling of human crowd dynamics lies in bridging across the different modelling scales. Typically, models are formulated at three distinct levels: the microscopic (individual), mesoscopic (kinetic), and macroscopic (hydrodynamic) scales. Microscopic models describe the detailed behavior of individual pedestrians using agent-based approaches, such as the Social Force Model (SFM) \cite{helbing1995social,Helbing2000} and its variants (see e.g. \cite{corbetta2018physics,yang2014guided}), cellular automata \cite{nishidate1996cellular,blue1999cellular,dietrich2014bridging}, etc. The mesoscopic scale serves as an intermediary between the microscopic and macroscopic approaches~\cite{aylaj2020unified,bellomo2008modelling,cristiani2014multiscale,bellomo2015multiscale,Bellomo2011,dogbe2012modelling,kim2019kinetic,hoogendoorn2000gas,aylaj2020crowd,bellomo2023human}. At the macroscopic scale, modelling describes the emergent/collective behavior of crowds. For instance, in the seminal work of Hughes \cite{HUGHES2002507}, the author models crowd dynamics by coupling a conservation law with an eikonal equation to represent how individuals choose paths based on a potential field. Moreover, in \cite{colombo2012nonlocal} the authors proposed nonlocal conservation laws for traffic and crowd dynamics, where the flux depends not only on local density but also on averaged information from the surrounding region. In \cite{Bellomo2011}, a framework for macroscopic crowd models was presented based on balance laws and kinetic theory, emphasizing how individual-level interactions give rise to complex collective behavior. More recently, \cite{bellomo2023human} highlights new trends such as multiscale coupling, behavioral heterogeneity, and nonlocal interaction kernels.

A major objective, when high-fidelity microsocpic/individual/agent-based simulators and/or detailed experimental data are available, is to derive mesoscopic or macroscopic models that can then be used for numerical analysis, optimization and control of the emergent/collective dynamics, for instance, in the optimal design of spaces, the reduction of evacuation times, and the control of the flows \cite{cristiani2014multiscale,Bellomo2011,cristiani2015modeling,albi2016invisible,albi2020mathematical,auletta2022herding,panagiotopoulos2022control,gong2023crowd}. In this context, kinetic models serve as an intermediate bridge. They are typically derived analytically using statistical mechanics techniques, starting from microscopic interaction rules. These kinetic formulations, in turn, facilitate the derivation of macroscopic (continuum) models through asymptotic methods \cite{bellomo2015multiscale,bellomo2023human}.

The transition from high-dimensional agent-based representations to macroscopic or hydrodynamic descriptions is commonly achieved through closure techniques. These methods establish relationships between higher-order moments of the microscopic distributions and a limited, smaller set of lower-order moments that capture the observable emergent behavior \cite{lee2023learning,Psarellis2024,fabiani2024task}. However, such closures inherently rely on assumptions that introduce biases both for modelling and analysis tasks. Typically, these assumptions include infinitely large populations, homogeneous agents, and uniform and/or local interaction networks, but also unknown ``behavioral'' characteristics, which can significantly impact the accuracy and validity of collective-level analyses.

Recently, data-driven computational approaches based on physics-free machine learning (ML) have been employed for modelling, numerical analysis, and control of pedestrian, crowd, and traffic dynamics. In this context, a wide range of ML architectures has been explored, including deep neural networks (DNNs) \cite{song2018data}, Long Short-Term Memory (LSTM) networks \cite{alahi2016social,yu2023multi}, Generative Adversarial Networks (GANs) \cite{gupta2018social,minartz2024understanding}, Convolutional Neural Networks (CNNs) \cite{nikhil2018convolutional,tripathi2019convolutional,zamboni2022pedestrian}, and reinforcement learning methods \cite{yao2019data,yoshida2025visual}. For a comprehensive review of these and related ML approaches, see \cite{huang2024review}. These approaches have proven effective in forecasting individual and group trajectories within the framework of time-series modelling and forecasting, without relying on physics-based formulations. However, at the same time, such approaches lack interpretability and, importantly, do not incorporate information of the global spatio-temporal emergent dynamics, such as density fields, that are essential in modelling crowd dynamics.

An alternative data-driven paradigm that bypasses the need to explicitly learn surrogate models is the Equation-free multiscale framework \cite{kevrekidis2003equation}. By appropriately coupling on-demand microscopic simulations with coarse-graining, or ``restriction", and subsequent fine-graining, or ``lifting", this approach builds a coarse time-stepper for the evolution at the macroscopic scale, thus bypassing the need for analytically derived continuum equations \cite{corradi2012equation,marschler2014equation,Patsatzis2023,panagiotopoulos2023continuation,Chin2024}. Although the equation-free framework enables powerful multiscale numerical analysis and control, it is fundamentally based on the on-demand construction of local discrete maps rather than on deriving a global macroscopic model of the crowd or collective dynamics (for example, a mass-conserving partial differential equation (PDE)). That reliance on local maps reduces interpretability and physical insight, and it hinders reliable generalization in long-term simulations.

To address the aforementioned challenges, a few recent studies have utilized physics-informed neural networks (PINNs) \cite{Raissi2018} including physics-informed neural operators (NOs) \cite{goswami2023physics,fabiani2025enabling} for solving the inverse problem, i.e., that of constructing hybrid models in the form of ``grey-box'' PDE models from data \cite{lee2023learning,fabiani2024task,fabiani2025enabling,Lee2020,Galaris2022}. Such an approach, combines ML with underlying knowledge of physical principles, such as conservation laws \cite{jagtap2020conservative} and closures \cite{Lee2023}, to create hybrid models that can both predict the spatio-temporal dynamics and ensure physical consistency, even under data scarcity. In the context of crowd dynamics, PINNs have been employed to construct both gray-box hydrodynamic PDEs \cite{guo2024analysis} as well as for discovering the physics/rules governing interactions and motion at the microscopic level \cite{zhang2022physics}. However, while PINNs provide an effective way to incorporate physical constraints and knowledge into ML models, they can also introduce biases due to their reliance on predefined forms of PDEs. Moreover, a key physical property in crowd modeling, mass conservation, is typically not enforced explicitly, but rather imposed as a soft constraint through the loss function \cite{jagtap2020conservative,pang2019fpinns,ahmadi2024ai}. This indirect enforcement can limit the generalizability and physical reliability of PINN-based models, particularly in heterogeneous or unforeseen scenarios where the true dynamics deviate from the assumed PDE forms.

For learning classes of PDEs that incorporate conservation laws--such as those found in crowd dynamics--the use of black-box DNNs and black-box NOs \cite{fabiani2024task, ahmadi2024ai, li2020fourier,li2023deep,azizzadenesheli2024neural,zappala2024learning} is inherently problematic. These models often lack physical consistency, particularly in preserving conservation laws, which are usually enforced only as soft constraints in the loss function. In addition, their training suffers from low approximation accuracy due to the ``curse of dimensionality", potentially leading to unphysical predictions and numerical instabilities during long-term simulations. Thus, to deal with that problem, RiemannONets were recently presented in \cite{peyvan2024riemannonets}, which are NOs based on DeepONets to solve Riemann problems encountered in compressible flows, thus enforcing the preservation of conservation laws. Towards solving the inverse problem, we have recently proposed GoRINNs \cite{patsatzis2025gorinns}, neural networks coupled with Godunov numerical schemes and Riemann solvers for learning hyperbolic PDEs that explicitly respect the conservation laws by construction. 

It is important to note, however, that despite the long-standing tradition of modeling crowd dynamics through (hydrodynamic) PDEs, such formulations inherently rely on an ``infinite-size" assumption. In realistic settings involving crowds of moderate or finite-size, this assumption often breaks down. Finite-size effects, such as spatial heterogeneity, individual interactions, and boundary influences, can dominate the dynamics, rendering continuum PDEs insufficient for accurate, quantitative predictions of real-world crowd behavior.

Here, inspired by the Equation-free multiscale framework \cite{kevrekidis2003equation}, and building on our previous recent works, on what we call \textit{next-generation Equation-free algorithms} based on both manifold learning and machine learning \cite{fabiani2024task,Patsatzis2023, fabiani2025enabling, della2024learning,papaioannou2022time,gallos2024data}, we present a computational framework, to learn the discrete evolution operator of the emergent dynamics in latent spaces from high-fidelity microscopic simulations of crowd dynamics. The proposed framework bridges the microscopic and macroscopic modelling scales, while explicitly preserving the underlying conservation law of mass in the ambient space. The methodology is deployed in four main stages. First, we map discrete microscopic distributions to continuous macroscopic representations (i.e., density fields). Second, we project the macroscopic fields into a low-dimensional latent space using Proper Orthogonal Decomposition (POD). Third, we learn surrogate reduced-order models (ROMs) in this latent space using discrete autoregressive maps: both linear Multivariate Autoregressive models (MVARs) and nonlinear LSTM models. These autoregressive models with lags implicitly implement a delay coordinate embedding, aligning with Takens'/Whitney's embedding theorems \cite{takens2006detecting,sauer1991embedology} for reconstructing the phase-space structure (here in the latent space). The concept of the approach (i.e., using POD embeddings and neural networks taking as inputs delayed time series for learning the dynamics of PDEs in latent spaces) can be traced back to the early '90s \cite{krischer1993model,shvartsman1998low} (see also results and discussion in \cite{papaioannou2022time,gallos2024data,shvartsman1998low,kemeth2022learning,axaas2023model}). In crowd dynamics, behavioral characteristics such as pedestrian target paths and decision-making, or path optimization constitute the ``unobservable" states, driving the observable positions/velocities and therefore the density field of the crowd; delay embeddings (as used implicitly in autoregressive models) have the potential, due to the Takens'/Whitney's embedding theorems \cite{takens2006detecting,sauer1991embedology}, to reconstruct the effective phase space containing these (unobservable) latent dynamics. Finally, we lift on-demand the learned latent-space dynamics back to the high-dimensional macroscopic space via the POD basis. Crucially, this configuration ensures explicit mass conservation-—a key physical constraint for any effective continuum description. For our numerical experiments, we employ the SFM \cite{helbing1995social,Helbing2000} to simulate pedestrian flow within a rectangular corridor containing an obstacle, under periodic boundary conditions. We particularly consider a unidirectional pedestrian flow and a more challenging, dense counterflow of two pedestrian groups. The performance and generalization ability of the proposed framework are assessed on unseen initial conditions.


\section{Methodology}
\subsection{Problem Statement}
In this framework, we consider the problem of bridging the scales in crowd dynamics with particular focus on deriving ROMs from detailed, high-dimensional, agent-based simulations of pedestrian flows. This derivation involves reconstructing on demand the spatio-temporal dynamics under the strict constraint of mass conservation.

We utilize a detailed simulator that models the dynamics of a human crowd, consisting of $N$ pedestrians moving within a domain $\Omega \subset \mathbb{R}^2$. The state of each pedestrian $i \in\{ 1, \ldots,N\}$ is characterized by its position $\mathbf{x}_i=\mathbf{x}_i(t)= (x_i(t), y_i(t))$ and velocity $\mathbf{v}_i= \mathbf{v}_i(t) =(v_{ix}(t), v_{iy}(t))$, and also a set of ``behavioral'' variables $\bu_i= \bu_i(t) \in \mathbb{R}^b$, which may include phenomena such as decision-making strategies, herding behavior, but also panic and irrationality  \cite{albi2016invisible,rassia2010escape,bellomo2015toward,bellomo2016human,haghani2019panic}, where $b\in \mathbb{N}$ is the number of behavioral traits modeled for each pedestrian. Additionally, we introduce an auxiliary variable $\bm{z}_i\in \mathbb{R}^b$ to denote the rate of change of the behavioral variables $\bu_i$. Then, the evolution of the position, velocity and behavior of each pedestrian is modelled by an individualistic/agent-based model, inspired by pseudo-Newtonian/molecular/Brownian dynamics formulations \cite{bellomo2023human,bellomo2020towards}, expressed in the general form:
\begin{align}
&\frac{d \bm{u}_i}{dt} =  \bm{z}_i,
\quad \frac{d \bm{z}_i}{d t} = \sum_{j\in\Omega_i}  \, \bm{\psi}_i(\mathbf{x}_i, \mathbf{v}_i, \bm{u}_i, \mathbf{x}_j,\mathbf{v}_j, \bm{u}_j), \label{swarm-crowd}\\
&\frac{d\mathbf{x}_{i}}{dt} =  \mathbf{v}_i, 
\quad \frac{d\mathbf{v}_{i}}{d t} =  \sum_{j \in\Omega_i} \boldsymbol{\phi}_i(\mathbf{x}_i, \mathbf{v}_i, \bm{u}_i, \mathbf{x}_j,\mathbf{v}_j, \bm{u}_j), \label{eq:microGen}
\end{align}
where $j$ refers to all pedestrians within the interaction domain of the $i$-th pedestrian $\Omega_i$, and $\bm{\psi}_i(\cdot), \boldsymbol{\phi}_i(\cdot)$ denote pseudo-acceleration terms incorporating interaction rules. The majority of models in this category are based on the celebrated social force model (SFM) for pedestrian dynamics~\cite{helbing1995social,bellomo2023human}, which was proposed in the mid-1990s. 

In contrast, at the macroscopic scale, crowd dynamics are typically described using phenomenological continuum models, most often formulated as PDEs. These models capture the collective behavior of large-scale pedestrian flows and are grounded in the theories of hydrodynamics and continuum mechanics. Well‑known examples include the so‑called Hughes model \cite{HUGHES2002507} and multi‑population nonlocal models \cite{colombo2012nonlocal}. However, such modelling approaches rely on several simplifying assumptions in order to derive algebraic closure relations that link individual-level motion with density‑scale dynamics. These assumptions often include, for example, that pedestrian speed depends solely on the local total density, that individuals aim to minimize travel time while avoiding congestion, or that their motion is guided by a potential field directing them toward their destination.

To address the limitations of phenomenological continuum models, several data-driven computational methods have been introduced to construct macroscopic hybrid models in the form of ``black-box'' or ``grey-box'' PDE models \cite{lee2023learning, fabiani2024task, Lee2020,Galaris2022,ahmadi2024ai,li2020fourier,li2023deep,azizzadenesheli2024neural,zappala2024learning}. These methods are predominantly based on DNNs/NOs, thus suffering from the ``curse of dimensionality'' that is inherent in their training. To circumvent this challenge, learning ROMs in well-constructed latent spaces is crucial, since it significantly lowers computational cost while capturing the essential dynamics. These ROMs must also inherently preserve fundamental principles like mass conservation, ensuring physical accuracy in the reconstructed space. Towards this goal, we aim at bridging high-fidelity microscopic/agent-based simulations and/or experimental/real-world data with macroscopic scale modelling, using ROMs that preserve mass conservation, within the framework of learning discrete operators/coarse-timesteppers \cite{kevrekidis2003equation}. 

The proposed methodology, presented in detail in the following section, is based on our previous efforts to develop what we call the \textit{next generation Equation-free (EF) framework} \cite{fabiani2024task,Patsatzis2023,della2024learning, papaioannou2022time,gallos2024data}. This framework deals with the ``curse of dimensionality'' by learning surrogate ML models in appropriate latent spaces emerging from the complex spatio-temporal dynamics of high-dimensional systems. It originates from the \emph{Equation-free} approach \cite{kevrekidis2003equation}, which enables high-dimensional simulations to perform system-level numerical analysis tasks on an appropriate latent space. This is achieved via the construction of a \emph{lifting operator} $\mathcal{L}$ and a \emph{restriction operator} $\mathcal{R}$ bridging the gap between states in a high-dimensional space, say $\mathbf{u}(t) \in \mathbb{R}^{M}$, and states in a latent space, say $\mathbf{y}_d(t) \in \mathbb{R}^d$, with $d \ll M$, as:
\begin{equation}
\mathbf{u}(t) = \mathcal{L}(\mathbf{y}_d(t)) \quad \text{and} \quad \mathbf{y}_d(t) = \mathcal{R}(\mathbf{u}(t)).
\end{equation}
These operators are used to construct a coarse-timestepper, the discrete evolution operator at the latent space, $F_{\delta t}: \mathbb{R}^d\to \mathbb{R}^d$ defined as:
\begin{equation}
    \mathbf{y}_d(t+\delta t) = \mathcal{R}\Bigl( {\Phi_{{\delta} t}}\big(\mathcal{L}(\mathbf{u}(t))\big) \Bigr):=F_{\delta t}\big(\mathbf{y}_d(t)\big),
\end{equation}
where $\Phi_{{\delta} t}: \mathbb{R}^M\to \mathbb{R}^M$ is the evolution operator at the high-dimensional space. In this work, we learn the discrete evolution operator $F_{\delta t}$ using autoregressive ROMs, thus effectively learning $\Phi_{{\delta} t}$ 
as:

\begin{equation}
    \mathbf{u}(t+\delta t) = \mathcal{L}\left(\mathbf{y}_d(t+\delta t) \right), \quad \mathbf{y}_d(t+\delta t) = F_{\delta t}\big(\mathbf{y}_d(t),\ldots,\mathbf{y}_d(t-w\delta t)\big),
    \label{eq:LER_basic}
\end{equation}
where the learned ROMs account for $w$ delayed latent points $\mathbf{y}_d(t-j \delta t)$ for $j=1,\ldots,w$. The latent points are obtained 
from previous predictions of $F_{\delta t}$ resulting in a closed-loop/recursive forecasting operator requiring only the initial $j=1,\ldots,w$ embeddings $\mathcal{R}(\mathbf{u}(j \delta t))$. This formulation enables efficient on-demand reconstruction of the high-dimensional space, as $F_{\delta t}$ evolves in latent space and lifting is performed only when a high-dimensional reconstruction is demanded.    


A key distinction between our proposed \textit{next generation EF}  and the traditional EF framework is that, in the latter, one constructs local  (numerical) maps and computes the quantities needed for bifurcation analysis on demand via short bursts of microscopic simulations to bridge detailed simulations and emergent dynamics (captured usually by a few moment distributions). In contrast, our proposed method learns a ``global dynamical model" in Eq.~\eqref{eq:LER_basic} via ML from long-term, high-fidelity simulations spanning the entire 
phase space, in low-dimensional latent spaces. This global dynamical-systems-based perspective enables the use of both linear (such as POD) and nonlinear manifold-learning techniques (such as Diffusion maps \cite{coifman2005geometric,Coifman2006}), and techniques for the solution of the pre-image problem, such as geometric harmonics \cite{COIFMAN200631,dietrich2015manifold} or autoencoders \cite{hinton2006reducing,kingma2019introduction}, to construct maps between the high-dimensional and the latent space. The methodology is described in detail in the following section.

The key advantage of our \textit{next generation EF framework} is its ability to circumvent the need for explicit derivation or closure assumptions in PDEs—a process that inherently introduces biases in modelling crowd density dynamics. Instead, it directly learns the ``actual'' evolution operator in Eq.~\eqref{eq:LER_basic}–-the map that advances observed macroscopic distributions (here the density) forward in time. This learned operator, embedded within computationally efficient linear or nonlinear autoregressive ROMs (MVARs or LSTMs, respectively), effectively encodes the complex, often analytically intractable, dynamics of the crowd. By focusing on learning how densities evolve (the evolution operator) rather than learning governing equations in the form of PDEs, the method circumvents the significant challenges of the ``curse of dimensionality'' associated with DNNs/NOs training of such PDEs. The resulting ROMs thus act as practical, data-driven surrogates that implicitly contain the action of the unknown PDE's solution operator, enabling accurate prediction of future crowd states solely from observed density snapshots. Importantly, we demonstrate that for the proposed configuration, the POD approach (used for identifying the latent space) conserves explicitly the mass of the reconstructed density field resulting from the ROMs. An explicit mass conservation of the continuum representations is an essential physical constraint for any credible surrogate model of crowd dynamics.

\subsection{The proposed ML framework}
The proposed methodology consists of four main steps. First, we map the discrete distributions of pedestrian positions into continuous density fields defined on a spatial grid using KDE \cite{Botev2010}. Second, we construct a \emph{restriction} operator to project the resulting density fields into an appropriate latent space. For this purpose, we employ the POD \cite{sirovich1987turbulence} method on normalized density profiles across the computational domain, ensuring that mass conservation is preserved in the reconstructed fields. Third, within the POD latent space, we train surrogate autoregressive ROMs using MVAR and LSTM time-series forecasting models to capture the complex dynamics present in the data. Finally, in the fourth step, we lift the latent ROM dynamics back to the macroscopic density ambient space by solving the pre-image problem. In the case of POD, this \emph{lifting} step is computationally inexpensive, as it reduces to a simple linear projection onto the ambient space using the stored orthogonal basis. A schematic overview of the methodology is shown in Fig. \ref{fig:Schematic}.
\begin{figure}[!b]
    \centering
    \includegraphics[width=0.85\linewidth]{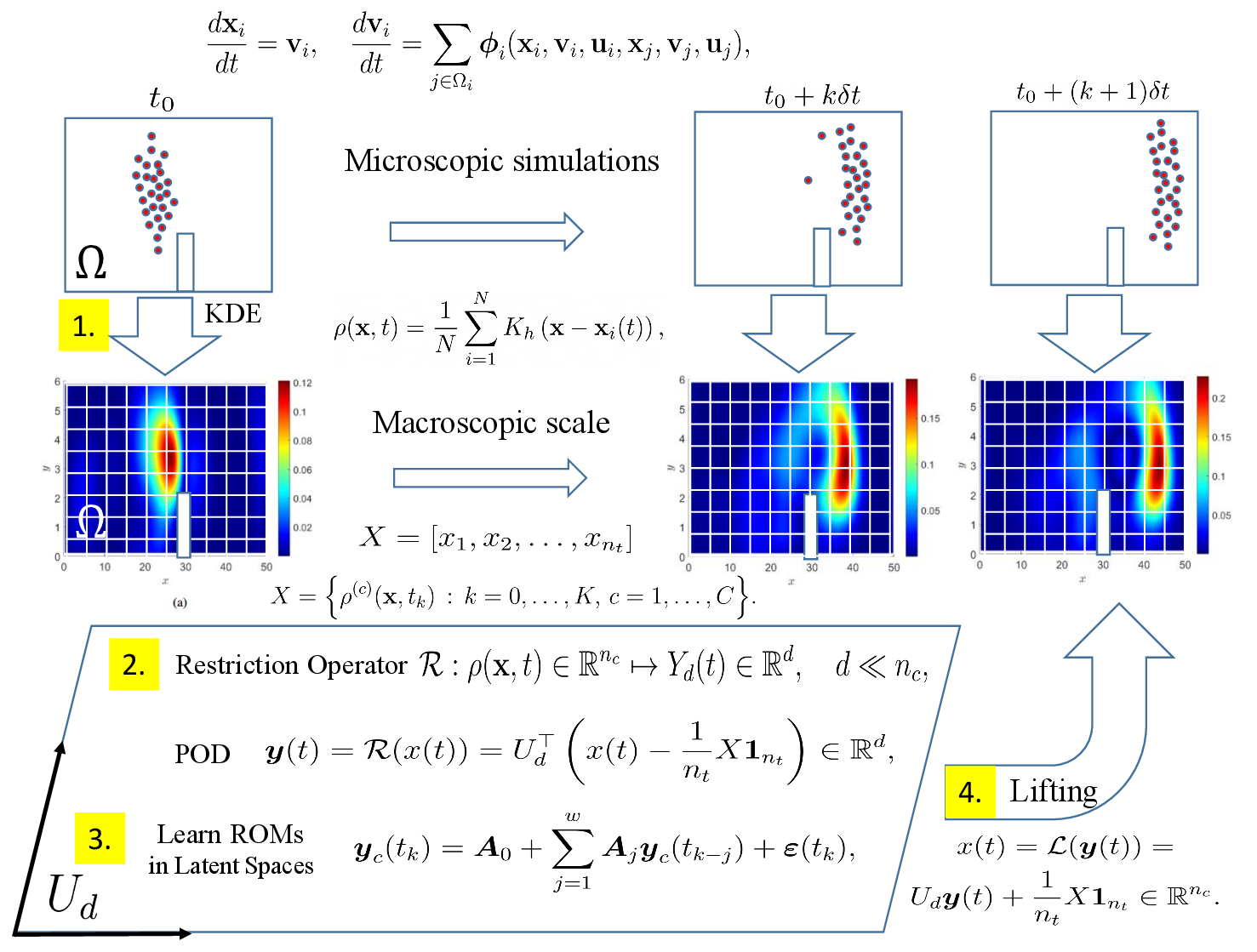}
    \caption{Schematic of the proposed methodology. First, microscopic distributions of pedestrian positions are mapped to macroscopic density fields (Step 1) and embedded into a latent space using POD (Step 2). Second, surrogate reduced-order models (MVARs and LSTMs) are trained in this latent space to capture and forecast the complex dynamics (Step 3). Finally, the learned dynamics are lifted back to the high-dimensional density fields via linear projection with the constructed POD basis (Step 4).}
    \label{fig:Schematic}
\end{figure}

\subsubsection{From discrete pedestrian positions to continuous density fields}
\label{sbsb:pos2den}
The initial step in our framework involves constructing a continuous density field from the discrete positions of pedestrians. Given microscopic observations, obtained by agent-based simulators in the form of Eq. (\ref{swarm-crowd}, \ref{eq:microGen}) or experimental data, macroscopic continuous quantities are typically derived via local averaging at each point of the spatial domain where the crowd moves. Here, we start from $N$ pairs of pedestrian positions $\{\mathbf{x}_i(t)\}_{i=1}^{N} \subset \Omega$ in a rectangular domain $\Omega = [a, b] \times [c, d] \subset \mathbb{R}^2$, as obtained by the SFM microscopic simulator, detailed in ~\ref{app:sfm}. We next employ KDE \cite{devroye1983equivalence} to estimate the continuous density fields along points on the spatial grid $\mathbf{x} \in \Omega$, as:  
\begin{equation}
    \rho(\mathbf{x},t) = \frac{1}{N} \sum_{i=1}^{N} K_h\left(\mathbf{x} - \mathbf{x}_i(t)\right), \label{eq:pos2den}
\end{equation}
where $K_h$ is a smooth non-negative kernel function satisfying $\int_{\Omega} K(\mathbf{x})\, d\mathbf{x} = 1$. Using a standard uniform spatial discretization $\delta x = (b-a)/n_x$, $\delta y = (d-c)/n_y$ with $n_x$/$n_y$ grid points along the $x$/$y$-axis of the domain $\Omega$, the resulting from Eq.~\eqref{eq:pos2den} density fields are computed as $\rho(\mathbf{x},t)=[\rho((x_1,y_1),t),\ldots,\rho((x_{n_x},y_{n_y}),t)]^\top\in\mathbb{R}^{n_x\times n_y}$, where $\rho((x_i,y_j),t)$ corresponds to the density computed at the grid point $(x_i,y_j)$, for $i = 1,\dots,n_x$,  $j = 1,\dots,n_y$. Then, the total mass in the domain is given by $ S(t) = \sum_{i=1}^{n_x} \sum_{j=1}^{n_y} \rho((x_i,y_j),t) \, \delta x\, \delta y$. The above procedure can be applied in multiple groups of pedestrians to derive (normalized) group macroscopic density fields. While KDE can be used for extracting continuous macroscopic fields, such as momentum and energy, from any microscopic discrete quantities, here we only construct density fields. Additional details for the discretization of the computational domain $\Omega$, the choice of kernel $K_h(\cdot)$, the consideration of obstacles in $\Omega$ (where a binary mask function is employed to consider zero density at the grid points of the obstacle), and the treatment of periodic boundary conditions are provided in ~\ref{app:kde}. 

\subsubsection{The Restriction and Lifting Operators}
\label{sbsb:RestLiftOPs}
Having obtained the macroscopic density fields, we now proceed to the second and fourth key elements of the approach, which, in analogy to the EF framework, constitute the construction of an ``embedding/restriction'' and ``lifting'' operators to and from an appropriate low-dimensional latent space.

The restriction operator, say, $\mathcal{R}$,  \cite{kevrekidis2003equation}, is a map from density distributions $\rho(\mathbf{x},t)$ to low-dimensional latent space coordinates, say $Y_d(t)$, i.e.,: 
\begin{equation}
\mathcal{R}:
 \rho (\mathbf{x},t)\in\mathbb{R}^{n_c}  \mapsto Y_d(t) \in \mathbb{R}^d, \quad d \ll n_c,
\label{eq:restrictionoperator} 
\end{equation}
where $n_c=n_x \times n_y$. To construct the restriction operator, we apply POD \cite{sirovich1987turbulence}  to a dataset consisting of macroscopic density ``solutions'' in time, obtained by microscopic simulations via Eq.~\eqref{eq:pos2den}, for different initial conditions; see ~\ref{app:intial_conditions} for details. Denoting the densities by $\rho^{(c)}(\mathbf{x},t_0+k\delta t)\equiv\rho^{(c)}(\mathbf{x},t_k)\in\mathbb{R}^{n_c}$, where $\delta t$ is a microscopic inner simulation step/observation sampling time of an experiment/real data, $t_0$ is the initial time, and the superscript $(c)$ denotes the case $c=1,\ldots,C$ of the different initial conditions considered, the training dataset generated by microscopic runs is represented by the matrix $X\in \mathbb{R}^{n_c\times (C\cdot K)}$:
\begin{equation}
X = \Bigl\{\rho^{(c)}(\mathbf{x}, t_k) \,:\, k=1,\ldots, K, \, c=1,\ldots, C \Bigr\}. \label{Dataset}
\end{equation}
Compactly, we represent the above data set as: 
\begin{equation}
    X=[x_1, x_2,\dots,x_{n_t}] \, \in \mathbb{R}^{n_c\times n_t}, \label{Dataset2}
\end{equation}
with columns, say $x_k=\rho^{(c)}(\mathbf{x}, t_k)/S(t_k)\in \mathbb{R}^{n_c}$, being the normalized density fields along cases and time steps $k=1,\ldots,n_t=C\cdot K$.

To complete the restriction operator at the second step, we employ POD on $X$, implemented via Singular Value Decomposition (SVD) on the centered matrix:
\begin{equation}
    \bar{X}=XH=U \Sigma V^\top, \quad H=I_{n_t} - \frac{1}{n_t}\mathbf{1}_{n_t} \mathbf{1}_{n_t}^\top,
\end{equation}
where $U\in\mathbb{R}^{n_c \times n_c}$ is the matrix collecting the left-singular vectors, $\Sigma\in\mathbb{R}^{n_c \times n_t}$ is the matrix collecting the singular values, $V\in\mathbb{R}^{n_t \times n_t}$ is the matrix collecting the right-singular vectors, $I_{n_t}$ is the $n_t$-dim. unity matrix and $\mathbf{1}_{n_t}$ is the $n_t$-dim. unit column vector. The latent space is spanned by the first $d$ left-singular vectors forming $U_d$. Thus, the projection of $X$ on the latent space provides the low-dimensional embedding:
\begin{equation}
    Y_d=U_d^\top\bar{X}, \, \,Y_d \in \mathbb{R}^{d\times n_t}.
    \label{eq:PODproj}
\end{equation}

For the lifting operator that maps $Y_d$ to the space of density profiles, we demonstrate that the POD reconstruction (aka the POD lifting operator) preserves the original total density. This aligns with the general characteristic of POD in maintaining symmetries \cite{aubry1993preserving}. 
Here, in particular, we demonstrate the following proposition.
\begin{prop}
\label{prop:1} Let us assume a matrix $X \in \mathbb{R}^{n_c\times n_t}$, with columns being normalized density fields. If the density is preserved along time steps, i.e., if: 
\begin{align}
    \begin{bmatrix}
      \sum_{i=1}^{n_c} x_{i,1}&  \sum_{i=1}^{n_c} x_{i,2}&\dots&\sum_{i=1}^{n_c} x_{i,n_t}  
    \end{bmatrix} =\mathbf{1}^\top_{n_t},
\end{align} 
then, the reconstructed density field, computed by the POD as: 
\begin{equation}
\tilde{X} = U_d U_d^\top \bar{X} + X(I_{n_t}-H), \label{eq:PODrecon}
\end{equation}
where $U_d=[u_1,u_2,\dots u_{d}]\in\mathbb{R}^{n_c\times d}$, is the orthonormal basis formed by the first $d$ left-singular vectors computed by the SVD, is also preserved, i.e.,:
 \begin{align}
    \begin{bmatrix}
      \sum_{i=1}^{n_c} \tilde{x}_{i,1}&  \sum_{i=1}^{n_c} \tilde{x}_{i,2}&\dots&\sum_{i=1}^{n_c} \tilde{x}_{i,n_t}  
    \end{bmatrix} =\mathbf{1}^\top_{n_t}. 
\end{align}
\end{prop}
\begin{proof}
By construction, since the sum of each column of $X$ is equal to one (that is $\mathbf{1}_{n_c}^\top X=\mathbf{1}_{n_t}^\top$), we have that:
\begin{equation}
   \mathbf{1}_{n_c}^\top \bar{X}= \mathbf{1}_{n_c}^\top XH=\mathbf{1}_{n_t}^\top H=\mathbf{1}_{n_t}^\top-\frac{1}{n_t}\mathbf{1}_{n_t}^\top\mathbf{1}_{n_t}\mathbf{1}_{n_t}^\top=\mathbf{1}_{n_t}^\top-\mathbf{1}_{n_t}^\top=\boldsymbol{0}^\top_{n_t}.
   \label{eq:orthog1}
\end{equation}
Therefore, the sum of each column of the centered matrix $\Bar{X}$ is equal to zero.

This implies that $\mathbf{1}_{n_c}$ is orthogonal to the column space of $\bar{X}$, and therefore is orthogonal to the left singular vectors 
of $\bar{X}$, which provide an orthogonal basis for the column space, i.e.,: 
\begin{equation} \label{eq:ortho1a}
    \mathbf{1}_{n_c}^\top U_d = \boldsymbol{0}^\top_d,
\end{equation}
where $d=1,2,\dots r$ with $r=rank(X)$.
Then, the multiplication of the reconstructed matrix in Eq.~\eqref{eq:PODrecon}  by $\mathbf{1}_{n_c}^\top$ implies:
\begin{equation}
    \mathbf{1}_{n_c}^\top \tilde{X} = \mathbf{1}_{n_c}^\top U_d U_d^\top \bar{X} + \mathbf{1}_{n_c}^\top X(I_{n_t}-H) = \mathbf{0}_{n_t}^\top + \frac{1}{n_t}\mathbf{1}_{n_t}^\top\mathbf{1}_{n_t}\mathbf{1}_{n_t}^\top = \mathbf{1}_{n_t}^\top, 
\end{equation} 
thus getting:
\begin{align}
    \mathbf{1}_{n_c}^\top \tilde{X} = \begin{bmatrix}
    \sum_{i=1}^{n_c} \tilde{x}_{i,1}&  \sum_{i=1}^{n_c} \tilde{x}_{i,2}&\dots&\sum_{i=1}^{n_c} \tilde{x}_{i,n_t}  
    \end{bmatrix} =\mathbf{1}^\top_{n_t}.
\end{align}
Hence, the POD reconstruction preserves the total original density.
\end{proof}

We next consider a more challenging case of two interacting pedestrian groups in the domain $\Omega$, a typical scenario in dense counterflow settings. As in the one-group case, we collect the density profiles $\rho^{(c)}_l(\mathbf{x}, t_k)\in \mathbb{R}^{n_{c}}$ of each group $l=1,2$ via Eq.~\eqref{eq:pos2den}, and represent them in the snapshot matrices:
\begin{equation}
X^{(1)}=\Big[x_{1}^{(1)}, \ldots, x_{n_t}^{(1)} \Big] \in \mathbb{R}^{n_c \times n_t}, \quad
X^{(2)}=\Big[x_{1}^{(2)}, \ldots, x_{n_t}^{(2)} \Big] \in \mathbb{R}^{n_c \times n_t},
\end{equation}
with columns $x^{(l)}_k=\rho^{(c)}_l(\mathbf{x}, t_k)/S_l(t_k)\in \mathbb{R}^{n_c}$ for $k=1,\ldots,n_t=C\cdot K$. Applying the single‑group POD (Proposition~\ref{prop:1}) separately to $X^{(1)}$ and $X^{(2)}$ yields reconstructions that preserve each group’s density, but ignore any statistical coupling between the groups. To retain mass‑preservation while also incorporating dependencies among groups, we construct enriched projection bases that combine the individual POD modes with cross‑covariance modes, as shown in the following proposition.

\begin{prop} \label{prop:2}
Let us assume the snapshot matrices $X^{(1)},\, X^{(2)} \in\mathbb{R}^{n_c \times n_t}$ with columns being normalized density fields of two interacting population groups. Furthermore, let $W\in \mathbb{R}^{n_c \times m}$ and $T \in \mathbb{R}^{n_c \times m}$ be the left and right singular vectors of the cross-covariance matrix $C^{(12)}=\frac{1}{n_t}\bar X^{(1)} \left(\bar X^{(2)}\right)^\top \in \mathbb{R}^{n_c\times n_c}$, and consider the augmented orthonormal projection bases: 
\begin{equation} \label{eq:Aug_Basis}
    \tilde U^{(1)} = \begin{bmatrix} \frac{1}{\sqrt{n_c}}\mathbf{1}_{n_c} & U_{d_1}^{(1)} & \hat W \end{bmatrix} \in \mathbb{R}^{n_c \times (d_1+m+1)}, \quad
\tilde U^{(2)} = \begin{bmatrix} \frac{1}{\sqrt{n_c}}\mathbf{1}_{n_c} & U_{d_2}^{(2)} & \hat T \end{bmatrix} \in \mathbb{R}^{n_c \times (d_2+m+1)},
\end{equation} 
where $U_{d_1}^{(1)} \in \mathbb{R}^{n_c\times d_1}$ and $U_{d_2}^{(2)} \in \mathbb{R}^{n_c\times d_2}$ are the first $d_1$ and $d_2$ left singular vectors from the SVD of $\bar X^{(1)}$ and $\bar X^{(2)}$, respectively, and the terms $\hat W = W_\perp (W_\perp^\top W_\perp)^{-1/2}$ and $\hat T = T_\perp (T_\perp^\top T_\perp)^{-1/2}$ are computed by the orthogonal projections $W_\perp = \left(I_{n_c} - \mathbf{1}_{n_c}\mathbf{1}^{\top}_{n_c}/n_c - U_{d_1}^{(1)} (U_{d_1}^{(1)})^\top\right)W \in \mathbb{R}^{n_c\times m}$ and $T_\perp = \left(I_{n_c} - \mathbf{1}_{n_c}\mathbf{1}^{\top}_{n_c}/n_c- U_{d_2}^{(2)} (U_{d_2}^{(2)})^\top\right)T \in \mathbb{R}^{n_c\times m}$. If the original densities are mass-preserving along time steps, i.e., if $\mathbf{1}^{\top}_{n_c}X^{(1)}=\mathbf{1}^{\top}_{n_{t}}$ and $\mathbf{1}^{\top}_{n_c}X^{(2)}=\mathbf{1}^{\top}_{n_{t}}$, then the reconstructed densities computed by projecting $X^{(1)}$ and $X^{(2)}$ onto the augmented bases as: 
\begin{equation}
    \tilde X^{(1)}=\tilde {U}^{(1)} \left(\tilde {U}^{(1)}\right)^{\top} \bar X^{(1)}+X^{(1)} (I_{n_t}-H), \quad \tilde X^{(2)}=\tilde {U}^{(2)} \left(\tilde {U}^{(2)}\right)^{\top} \bar X^{(2)}+X^{(2)} (I_{n_t}-H),
    \label{Eq:Reconstruction_agumented}
\end{equation}
are also mass-preserving, i.e., $\mathbf{1}^{\top}_{n_c}\tilde X^{(1)}=\mathbf{1}^{\top}_{n_{t}}$ and $\mathbf{1}^{\top}_{n_c} \tilde X^{(2)}=\mathbf{1}^{\top}_{n_{t}}$.
\end{prop}
\begin{proof}
Assume the original densities are mass-preserving for each group. From the properties of the single-population POD (see Eq.~\eqref{eq:orthog1} and Eq.~\eqref{eq:ortho1a}), it follows that:
\begin{equation} \label{eq:singlePop_ortho}
   \mathbf{1}_{n_c}^\top \bar{X}^{(l)}=\boldsymbol{0}^\top_{n_t}, \quad \mathbf{1}_{n_c}^\top U^{(l)}_{d_l} = \mathbf{0}_{d_l}^\top, 
\end{equation}
where $\bar{X}^{(l)}=X^{(l)}H$ is the centered data matrix and $U_{d_l}^{(l)}\in \mathbb{R}^{n_c\times d_l}$ are the first $d_l$ left singular vectors from the SVD of $\bar X^{(l)}$ for the $l=1,2$ groups. Let's consider now the SVD of the cross-covariance matrix: 
\begin{equation}
    C^{(12)} = \frac{1}{n_t}\bar X^{(1)} \left(\bar X^{(2)}\right)^\top= W \Sigma_{C^{(12)}} T^{\top}, 
    \label{Eq:Cross-cov_SVD}
\end{equation}
with $W \in \mathbb{R}^{n_c \times m}$, $\Sigma_{C^{(12)}} \in \mathbb{R}^{m \times m}$ and $T \in \mathbb{R}^{n_c \times m}$, where $m$ denotes the number of retained cross-covariance modes. Each column of $C^{(12)}$ is a linear combination of the columns of $\bar X^{(1)}$, and each row a linear combination of the rows of $\bar X^{(2)}$. Consequently:
\begin{equation}
\mathcal{R}(W)\equiv \mathcal{R}\big(C^{(12)}\big) \subseteq \mathcal{R}(\bar{X}^{(1)}), 
\quad 
\mathcal{R}(T)\equiv\mathcal{R}\big((C^{(12)})^\top\big) \subseteq \mathcal{R}\big(\bar{X}^{(2)}\big). 
\end{equation}

To obtain interaction modes orthogonal to both the unity vectors $\mathbf{1}_{n_c}$ and the individual POD bases $U^{(l)}_{d_l}$ for $l=1,2$, we project the $W$ and $T$ onto the complementary subspace: 
\begin{equation}
    W_\perp = \left(I_{n_c} - \frac{1}{n_c}\mathbf{1}_{n_c}\mathbf{1}^{\top}_{n_c} - U_{d_1}^{(1)} (U_{d_1}^{(1)})^\top\right)W \in \mathbb{R}^{n_c\times m}, \quad T_\perp = \left(I_{n_c} - \frac{1}{n_c} \mathbf{1}_{n_c}\mathbf{1}^{\top}_{n_c}- U_{d_2}^{(2)} (U_{d_2}^{(2)})^\top\right)T \in \mathbb{R}^{n_c\times m},
\end{equation}
and define their orthonormalized versions as:
\begin{equation}
    \hat W = W_\perp (W_\perp^\top W
    _\perp)^{-1/2}, \quad  \hat T = T_\perp (T_\perp^\top T_\perp)^{-1/2}.
\end{equation}
By construction of $\hat{W}$ and $\hat{T}$, we obtain:
\begin{equation} \label{eq:cross_ortho}
    \mathbf{1}_{n_c}^\top \hat{W}= \mathbf{0}^\top_m, \quad (U_{d_1}^{(1)})^\top \hat{W}=0\in \mathbb{R}^{d_1 \times n_c}, \quad
    \mathbf{1}_{n_c}^\top \hat{T}= \mathbf{0}^\top_k, \quad (U_{d_2}^{(2)})^\top \hat{T}=0\in \mathbb{R}^{d_2 \times n_c}.
\end{equation}
The augmented orthonormal bases (Eq.~\eqref{eq:Aug_Basis}) are therefore:
\begin{equation} \label{eq:Aug_Basis_proof}
    \tilde U^{(1)} = \begin{bmatrix} \frac{1}{\sqrt{n_c}}\mathbf{1}_{n_c} & U_{d_1}^{(1)} & \hat W \end{bmatrix} \in \mathbb{R}^{n_c \times (d_1+m+1)}, \quad
\tilde U^{(2)} = \begin{bmatrix} \frac{1}{\sqrt{n_c}}\mathbf{1}_{n_c} & U_{d_2}^{(2)} & \hat T \end{bmatrix} \in \mathbb{R}^{n_c \times (d_2+m+1)},
\end{equation} 
according to which, the reduced coordinates of the snapshots can be obtained as:
\begin{equation}
Y^{(1)} = (\tilde{U}^{(1)})^\top \bar X^{(1)} \in \mathbb{R}^{(d_1+m+1) \times n_t}, \qquad
Y^{(2)} = (\tilde{U}^{(2)})^\top \bar X^{(2)} \in \mathbb{R}^{(d_2+m+1) \times n_t}.
\end{equation}

Multiplying the augmented basis in Eq.~\eqref{eq:Aug_Basis_proof} by $\mathbf{1}_{n_c}^\top$ from the left, and using Eq.~\eqref{eq:cross_ortho} together with the orthogonality $\mathbf{1}_{n_c}^\top U^{(l)}_{d_l} = \mathbf{0}_{d_l}^\top$ in Eq.~\eqref{eq:singlePop_ortho} gives:
\begin{equation}
    \mathbf{1}_{n_c}^\top \tilde U^{(1)} = \sqrt{n_c} \mathbf{e}_1^\top, \quad
    \mathbf{1}_{n_c}^\top \tilde U^{(2)} = \sqrt{n_c} \mathbf{e}_1^\top,
\end{equation}
where $\mathbf{e}_1=[1,0,\ldots,0]^\top$. Hence, for the projected centered data, we get:
\begin{equation} \label{eq:cross_ortho1}
    \mathbf{1}_{n_c}^\top \tilde U^{(l)} (\tilde{U}^{(l)})^\top \bar X^{(l)} = \sqrt{n_c} \mathbf{e}_1^\top (\tilde{U}^{(l)})^\top \bar X^{(l)} = \mathbf{1}_{n_c}^\top \bar X^{(l)} = \mathbf{0}_{n_t}^\top,\quad l=1,2.
\end{equation}
The last equality uses the mass-preservation property for the original densities $ \mathbf{1}_{n_c}^\top \bar{X}^{(l)}=\boldsymbol{0}^\top_{n_t}$ in Eq.~\eqref{eq:singlePop_ortho}.

Finally, using the reconstruction formula in Eq.~\eqref{Eq:Reconstruction_agumented}, and the identities derived in Eqs.~\eqref{eq:singlePop_ortho} and \eqref{eq:cross_ortho1}, we obtain:
\begin{align}
\mathbf{1}^{\top}_{n_c}\tilde X^{(l)}=\mathbf{1}_{n_c}^\top \tilde U^{(l)} (\tilde{U}^{(l)})^\top \bar X^{(l)}+\mathbf{1}^{\top}_{n_c}  X^{(l)} (I_{n_t}-H)=\mathbf{1}^{\top}_{n_c}  X^{(l)}(I_{n_t}-H)=\mathbf{1}_{n_t},\quad l=1,2. 
\end{align}
Thus, the reconstructions $\tilde X^{(1)}$ and $\tilde X^{(2)}$, obtained via the projection of $X^{(1)}$ and $X^{(2)}$ onto the augmented basis in Eq.~\eqref{eq:Aug_Basis_proof}, retain the mass-preservation property of the original data.
\end{proof}

In summary, given a density distribution $\rho(\mathbf{x},t)$ obtained via Eq.~\eqref{eq:pos2den} for a single population, we obtain the normalized density along the $n_c$ grid points of $\Omega$ as:
\begin{equation}
    x(t) = \dfrac{1}{S(t)} \big [\rho((x_1,y_1),t),\ldots,\rho((x_{n_x},y_{n_y}),t)\big]^\top \in \mathbb{R}^{n_c},
    \label{eq:denNorm}
\end{equation}
where $S(t)$ is the total mass on time $t$, and define the \emph{restriction} operator $\mathcal{R}$ by the projection on the POD basis of Eq.~\eqref{eq:PODproj} as:
\begin{equation}
    \boldsymbol{y}(t) = \mathcal{R}(x(t)) = U_d^\top \left(x(t) -\frac{1}{n_t} X\mathbf{1}_{n_t}\right) \in \mathbb{R}^d,
    \label{eq:Rest}
\end{equation}
where $X$ is the dataset matrix in Eq.~\eqref{Dataset2} and $U_d$ is the matrix containing the first $d$ left-singular vectors. On the other hand, the \emph{lifting} operator $\mathcal{L}$ that maps latent coordinates, say $\boldsymbol{y}(t)$, back to the density distributions is defined via the POD reconstruction as:
\begin{equation}
    x(t) = \mathcal{L}(\boldsymbol{y}(t)) = U_d \boldsymbol{y}(t) + \frac{1}{n_t} X\mathbf{1}_{n_t} \in \mathbb{R}^{n_c}.
    \label{eq:Lift}
\end{equation}

In the case of two interacting populations with normalized densities $x^{(l)}(t)$ for $l=1,2$, we consider the augmented bases $\tilde U^{(l)}$ in Proposition~\ref{prop:2} and adjust the restriction operator $\mathcal{R}$ to:
\begin{equation}
     \boldsymbol{y}(t) = [\boldsymbol{y}^{(1)}(t), \boldsymbol{y}^{(2)}(t)]^\top \in\mathbb{R}^{d_1+d_2+2m+2}, \quad \boldsymbol{y}^{(l)}(t) = \mathcal{R}^{(l)}(x^{(l)}(t)) = (\tilde{U}^{(l)})^\top \left(x^{(l)}(t) -\frac{1}{n_t} X^{(l)}\mathbf{1}_{n_t}\right) \in \mathbb{R}^{d_l+m+1},
    \label{eq:Rest2}
\end{equation}
and the \emph{lifting} operator $\mathcal{L}$ to:
\begin{equation}
    x(t) = [x^{(1)}(t),x^{(2)}(t)]^\top\in\mathbb{R}^{2n_c},  \quad x^{(l)}(t)=\mathcal{L}(\boldsymbol{y}^{(l)}(t)) = \tilde{U}^{(l)} \boldsymbol{y}^{(l)}(t) + \frac{1}{n_t} X^{(l)}\mathbf{1}_{n_t} \in \mathbb{R}^{n_c}.
    \label{eq:Lift2}
\end{equation}

In both cases, the above definitions of the restriction and lifting operators guarantee that the reconstructed normalized density profiles preserve the total per-population density within the domain $\Omega$, as per Propositions~\ref{prop:1} and \ref{prop:2}. We highlight here that while the operators are constructed based on the dataset $X$ (or $X^{(l)}$ for $l=1,2$), their employment applies to any seen or unseen observation $x(t)$ obtained from Eq.~\eqref{eq:denNorm} and $\boldsymbol{y}(t)$. Hereafter, for conciseness of the presentation, we will carry the single-population dimensions of $x(t)$ and $\boldsymbol{y}(t)$, unless otherwise specified.

\begin{remark}
Mass preservation and element-wise positivity in the reconstructed field are distinct properties. Standard POD/SVD truncation does not guarantee non-negativity in reconstruction: negative pointwise values may appear even when the original snapshots are positive. Since the reconstructed field is obtained by adding back the original mean, small negative entries---when they occur---are expected to appear in the low-density tails of the distribution. In such cases, we post-process the corresponding snapshots \textit{only for observation purposes}, by setting these values to zero, and renormalizing the reconstruction by the total mass of the nonnegative part. This ensures positivity while preserving total mass, 
with normalization factors remaining near 1.

\end{remark}
\begin{remark} 
The reconstruction step is only observational in our pipeline, as density fields are reconstructed via $\mathcal{L}$ in Eq.~\eqref{eq:Lift} or \eqref{eq:Lift2}, while the ROM evolves entirely in the latent space (see Eq.~\eqref{eq:LER_basic}). Since learning is performed with the original snapshots (which are nonnegative by construction), any post-process non-negativity enforcement does not influence the latent representation learning. Constrained reconstruction approaches, such as nearest-neighbor based methods commonly used in Diffusion Maps latent spaces \cite{Patsatzis2023,Chin2024}, can guarantee non-negativity by construction but require higher computational cost; integrating such techniques into our pipeline is straightforward, although a detailed comparative study is out of the scope in the current work.
\end{remark}

\subsubsection{Learning the dynamics in latent spaces using MVAR and LSTM}
\label{sbsb:ROMs}
With the restriction operator, which maps density fields to latent coordinates, constructed, the third step of the proposed approach involves training ROMs to learn the latent dynamics from sequential data. We note that while the restriction and lifting operators $\mathcal{R}$ and $\mathcal{L}$ are, in principle, defined over continuous time inputs, we apply them at a discrete set of uniformly spaced time instances $t_k = t_0 + k \delta t$ for $k=1,2,\ldots,K$, consistent with the temporal sampling of our data; see Eq.~\eqref{Dataset}.
To capture temporal dependencies via delayed embeddings (in line with Takens'/Whitney's embedding theorems \cite{takens2006detecting,sauer1991embedology}), we employ autoregressive ROMs to learn the dynamics of the latent embeddings $\boldsymbol{y}(t_k)$. The simplest such model is MVAR, which provides an efficient, linear baseline with globally optimal, least squares solutions. For a nonlinear comparison, we employ LSTM networks \cite{hochreiter1997long}, whose gating mechanisms effectively model complex temporal patterns; though numerous alternatives exist (Transformers \cite{hao2023gnot}, Neural Ordinary Differential Equations (ODEs) \cite{rubanova2019latent}, etc.), LSTMs represent state-of-the-art in nonlinear sequential modeling. A single ROM is trained in both single- and two-population cases; in the latter, this supports coupled latent dynamics that capture temporal interactions between groups. For compactness in what follows, we denote the dimension of the concatenated two-population latent coordinates in Eq.~\eqref{eq:Rest2} by $d=d_1+d_2+2(m+1)$, to match the single-population dimensions.


Let an MVAR model of order $w$ (i.e., with $w$ time lags) without intercept be expressed as:

\begin{equation}
\boldsymbol{y}_c(t_k) = \sum_{j=1}^{w}  \boldsymbol{A}_j \boldsymbol{y}_c(t_{k-j}) + \boldsymbol{\varepsilon}(t_k),
\label{eq:mvar}
\end{equation}

where $\boldsymbol{y}_c(t_k)$ is the latent variable vector, evaluated at discrete times $t_k$ for each of the $c=1,\ldots,C$ cases of different initial condition considered, and $\boldsymbol{A}_j \in \mathbb{R}^{d \times d}$ are the regression coefficient matrices for each lag $j$. The predictor vectors $\boldsymbol{y}_c(t_{k-j}) \in \mathbb{R}^d$ include the $w$ past lag latent variables evaluated at discrete times $t_{k-1},\ldots,t_{k-w}$. Since the latent variables $\boldsymbol{y}_c(t_{k})$ are zero-mean (a consequence of the mean-centered POD projection) the intercept term is excluded from Eq.~\eqref{eq:mvar}. The error term $\boldsymbol{\varepsilon}(t_k) \in \mathbb{R}^d$ is assumed to follow a multivariate white noise process:
\begin{align}
    \mathbb{E}[\boldsymbol{\varepsilon}(t_k)] = \mathbf{0}, \qquad
\mathbb{E}[\boldsymbol{\varepsilon}(t_k) \boldsymbol{\varepsilon}(t_k)^{\top}] = \Sigma \quad \text{and} \quad
\mathbb{E}[\boldsymbol{\varepsilon}(t_k) \boldsymbol{\varepsilon}(t_r)^{\top}] = \mathbf{0} \quad \forall k \neq r,
\end{align}
for every $k,r=1,\ldots,K$. 

Given sequential (time series) data for $\{\boldsymbol{y}_c(t_k)\}$, the unknown parameters of the MVAR in Eq.~\eqref{eq:mvar}, i.e., and the coefficient matrices $\boldsymbol{A}_j$, can be computed by the solution of the following least-squares problem:
\begin{equation}
\argmin_{ \{\boldsymbol{A}_j\}_{j=1}^w} \mathcal{L}_{\text{MVAR}}\left(\{ \boldsymbol{y}_c(t_k) \}, \{\hat{\boldsymbol{y}}_c(t_k)\}; \{\boldsymbol{A}_j\}_{j=1}^w\right),
\label{eq:mvar_ls}
\end{equation}
where the MVAR loss function is:
\begin{equation}
\mathcal{L}_{\text{MVAR}}\left(\{ \boldsymbol{y}_c(t_k) \}, \{\hat{\boldsymbol{y}}_c(t_k)\}; \{\boldsymbol{A}_j\}_{j=1}^w\right) = \sum_{c=1}^C\sum_{k=w+1}^{K}
\Big \lVert
\boldsymbol{y}_c(t_k) - \hat{\boldsymbol{y}}_c(t_k)
\Big \rVert_2^2,
\label{eq:mvar_loss}
\end{equation}
with the predictions of the MVAR model given by:
\begin{equation}
    \hat{\boldsymbol{y}}_c(t_k) = \sum_{j=1}^{w} \boldsymbol{A}_j \boldsymbol{y}_c(t_{k-j}),
    \label{eq:MVARpred}
\end{equation}
for $c=1,\ldots,C$ and $k=1,\ldots,K$. The optimization problem in Eq.~\eqref{eq:mvar_ls} can be solved using regularized least squares methods \cite{hoerl1970ridge}. We note that MVARs require the trivial assumption of time series stationarity, which was the case for our illustrations. If the latent coordinate time series is non-stationary, pre-processing steps like differencing can be applied to achieve stationarity before solving the optimization problem. We highlight that MVAR model training further requires the selection of an optimal lag window size $w$, which determines the length of the time series history used to model the temporal dependencies. Here, we used two information-theoretic criteria for determining $w$, the details for which are provided in \ref{app:bic_aic}.

As a non-linear autoregressive model alternative to linear MVARs, we further consider LSTM networks \cite{hochreiter1997long}, a type of Recurrent Neural Networks (RNNs) characterized by gating mechanisms, which allow the network to retain and propagate information over long sequences selectively, but also effectively handle vanishing gradients; a common challenge in traditional RNNs \cite{graves2012long,Chandra2022,Kyongmin2019}. Let an LSTM model with past $w$ latent states $\boldsymbol{y}_c(t_{k-j})$ for $j=1,\ldots,w$ be expressed as: 
\begin{align}
(\mathbf{h}_\tau, \mathbf{c}_\tau) &= F_\mathrm{LSTM} \big(\boldsymbol{y}_c(\tau), \mathbf{h}_{\tau-1}, \mathbf{c}_{\tau-1}; \boldsymbol{\theta}\big), \quad \text{for} \ \tau=t_{k-w},\ldots, t_{k-1},  \label{eq:LSTMgates}\\
\hat{\boldsymbol{y}}_c(t_k) &= \mathbf{W}_y\,\mathbf{h}_{t_{k-1}} + \mathbf{b}_y, \label{eq:LSTMout}
\end{align}
where $\mathbf{h}_{t_k}\in \mathbb{R}^{N_h}$ and $\mathbf{c}_{t_k}\in \mathbb{R}^{N_h}$ are the evaluated at discrete times $t_k$, hidden and memory cell states, respectively, $N_h$ corresponds to the number of hidden units and $\mathbf{W}_y\in\mathbb{R}^{d\times N_h},\mathbf{b}_y\in \mathbb{R}^d$ define an output layer mapping the last hidden state $\mathbf{h}_{t_{k-1}}$ to the predicted latent variable vector $\hat{\boldsymbol{y}}_c(t_k)\in \mathbb{R}^d$. The gating mechanism of the LSTM is included in Eq.~\eqref{eq:LSTMgates} at $F_\mathrm{LSTM}(\cdot)$, which we avoid here for conciseness of the presentation; a detailed presentation of the gating mechanism in $F_\mathrm{LSTM}(\cdot)$ and the related parameters $\boldsymbol{\theta}$ can be found in \cite{hochreiter1997long,graves2012long}. Here, we use a vanilla LSTM with 4 gates, containing the parameters $\boldsymbol{\theta} \in \mathbb{R}^{4N_h(d+N_h+1)}$. 

Similarly to the sequence-to-one learning of MVARs, given sequential data for $\{\boldsymbol{y}_c(t_k)\}$ for a fixed lag window of width $w$, we train the LSTM model by solving the optimization problem:
\begin{equation}
\argmin_{\boldsymbol{\theta}, \mathbf{W}_y, \mathbf{b}_y} \mathcal{L}_{\text{LSTM}}\left(\{\boldsymbol{y}_c(t_k)\}, \{\hat{\boldsymbol{y}}_c(t_k)\}; \boldsymbol{\theta}, \mathbf{W}_y, \mathbf{b}_y \right),
\label{eq:lstm_ls}
\end{equation}
where the LSTM loss function is the mean squared error (MSE) between the true value of the latent variable vector $\boldsymbol{y}_c(t_k)$ and the predicted one $\hat{\boldsymbol{y}}_c(t_k)$, as provided by Eqs.~(\ref{eq:LSTMgates}, \ref{eq:LSTMout}), that is:
\begin{equation}
\mathcal{L}_{\text{LSTM}}\left(\{\boldsymbol{y}_c(t_k)\}, \{\hat{\boldsymbol{y}}_c(t_k)\}; \boldsymbol{\theta}, \mathbf{W}_y, \mathbf{b}_y \right) = \dfrac{1}{C (K-w)} \sum_{c=1}^C\sum_{t_k=w+1}^{K}
\Big \lVert \boldsymbol{y}_c(t_k)-\hat{\boldsymbol{y}}_c(t_k) \Big \rVert^2_2.
\label{eq:lstm_loss}
\end{equation}
To solve the optimization problem in Eq.~\eqref{eq:lstm_loss}, one typically uses gradient-based methods. Here, we employ the Adam optimizer \cite{kingma2015adam}, which is a stochastic variant of gradient descent, known for its adaptive learning-rate capabilities. 

We highlight that while MVAR models are theoretically simple and efficient, their strict reliance on stationary data poses a significant limitation, requiring explicit pre‑processing that may distort dynamics or lose information \cite{Brooks_2019}. In contrast, LSTMs as dynamical models can inherently handle non‑stationarity through their gated memory cells and adaptive hidden states. This allows LSTM to learn complex, long‑range temporal dependencies and nonlinear relationships directly from raw data. However, this flexibility comes at the cost of needing much larger datasets, greater computational resources, and often lacking interpretability compared to MVAR models. 

Finally, the predicted latent coordinates $\hat{\boldsymbol{y}}_c(t_k)$ obtained from the trained MVARs or LSTMs models are subsequently mapped back to the high-dimensional density space using the lifting operator defined in Eqs.~\eqref{eq:Lift} or \eqref{eq:Lift2} for the single or two-population case, respectively.


\section{Case studies and methodology implementation}
We considered two representative benchmark problems in crowd dynamics: a unidirectional crowd flow moving through a corridor containing an obstacle, and a counterflow of two interacting populations in the same domain. 
The configurations of both case studies are described below, along with the implementation of the proposed methodology.

\subsection{Configurations}
\label{sb:conf}
The two crowd dynamics configurations simulate  pedestrians moving in a corridor of dimensions $48\,\text{m} \times 12\,\text{m}$ ($x$: length, $y$: width) in the presence of an obstacle, as shown in Fig.~\ref{fig:Pedestrians}. 
\begin{figure}[hbt!]
    \centering
    \begin{subfigure}[b]{0.47\textwidth}
        \includegraphics[width=\textwidth]{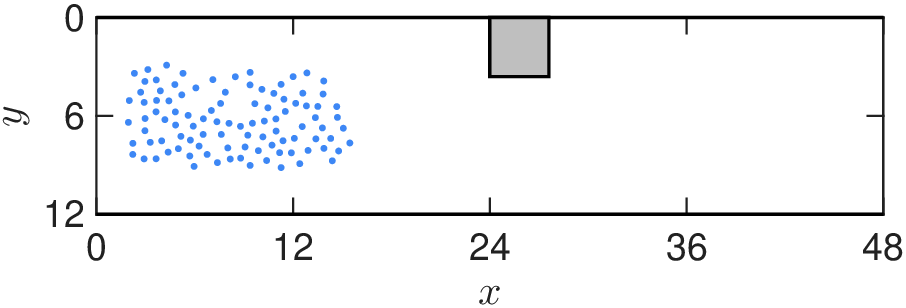}
        \caption{Unidirectional flow}
    \end{subfigure}
    \hfill
    \begin{subfigure}[b]{0.47\textwidth}
        \includegraphics[width=\textwidth]{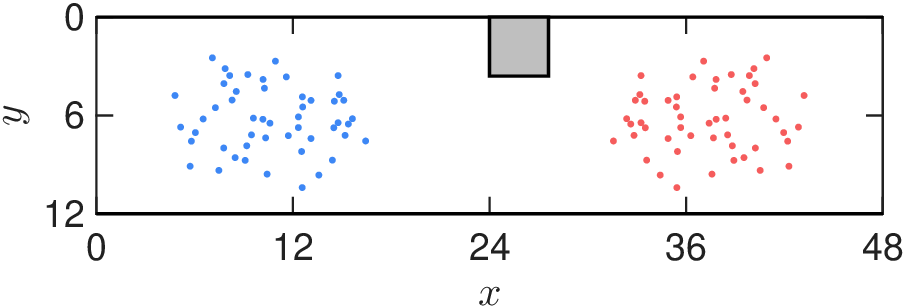}
        \caption{Counterflow}
    \end{subfigure}
    \caption{Crowd dynamics configurations of SFM simulating pedestrians moving in a corridor past an obstacle. In the unidirectional flow configuration in Fig.~\ref{fig:Pedestrians}(a), pedestrians move from the left to the right of the corridor. In the counterflow configuration in Fig.~\ref{fig:Pedestrians}(b), two pedestrian populations move in opposite directions; group~1 (blue) moves from the left to the right, while group~2 (red) follows the opposite direction.}
    \label{fig:Pedestrians}
\end{figure}
The square obstacle of area $12.96\,\text{m}^2$ is centered at $(25.8\, \text{m}, 1.8\, \text{m})$, with the pedestrians navigating around it towards reaching their destination (right or left corridor's boundary). Hence, the computational domain is $\Omega\backslash \Gamma$, with the corridor occupying $\Omega=[0, 48]\times [0,12]$ and the obstacle $\Gamma=[24,27.6] \times [0,3.6]$.

For modelling the pedestrians' movement, we used the SFM \cite{helbing1995social,Helbing2000} as a microscopic agent-based simulator. SFM falls within the general form of Eqs.~(\ref{swarm-crowd}, \ref{eq:microGen}) and the idea behind it is that the motion of individuals is governed by ``virtual forces''. These forces include acceleration toward the desired destination, as well as repulsive forces ``exerted" by obstacles, borders, and nearby individuals to maintain a safe distance. Additionally, one may also consider an attractive force due to mimetic behavior, which can be observed in scenarios such as emergency evacuations or collective movement toward objects of common interest \cite{helbing1995social}. The derivation of SFM is based on the Langevin-type equation that contains a frictional term, a deterministic force term, and a random force term. In the SFM, these terms translate into a goal-directed force (playing the role of friction), social and repulsive (from obstacles) deterministic forces, and a random term, reading:
\begin{equation}
m_i \frac{d {\bm{\dot{x}}}_i}{dt}= \underbrace{-\frac{1}{\tau_i}({v_0}_i\bm{e}_i-\bm{v}_i)}_{\text{goal-directed force}} + \underbrace{\sum_{j\neq i}\mathbf{F}_{ij}(\bm{x}_i,\bm{x}_j)}_{\text{social force}} + \underbrace{\sum_{\text{obstacles}}\mathbf{F}_{iB}(\bm{x}_i)}_{\text{repulsive force}} + \mathbf{F}_{rand},
\label{eq:SFMshort}
\end{equation}
where $m_i$ is the mass of the $i$-th pedestrian, ${v_0}_i$ is their maximum walking speed, $\bm{e}_i$ represents the desired direction (serving as a simplistic auxiliary/behavioral variable), and $\tau_i$ is a characteristic relaxation time. Here, we avoid the inclusion of random forces, since we are interested in the macroscopic behavior that often emerges from the deterministic pedestrian interactions \cite{helbing1995social}. Details on the SFM forces included in Eq.~\eqref{eq:SFMshort} and the related parameters that were considered in this work are provided in ~\ref{app:sfm}. 

\begin{remark}
Here, we employed the Social Force model purely as a convenient baseline to demonstrate and validate the proposed methodology. Naturally, the same workflow is not tied to this specific choice: it can be coupled just as well with more sophisticated agent-based approaches—e.g., models that include explicit route planning, richer interaction rules, heterogeneous behaviors, or learned decision-making—depending on the fidelity required by the application.
\end{remark}

In the unidirectional flow configuration (see Fig.~\ref{fig:Pedestrians}(a), we simulated a crowd of $N = 100$ pedestrians moving through a corridor, with initial spatial distributions ranging from homogeneous dispersions to clustered configurations (for more details, see~\ref{app:intial_conditions}). Pedestrians are directed to move from the left to the right of the corridor. Since the SFM governs pedestrians' motion through target-directed forces without intrinsic path planning, we employ a sequential two-target scheme to induce realistic obstacle-avoidance behavior. Pedestrians are initially guided towards an intermediate target zone below the obstacle at coordinates $x=25\,\text{m}$ and $y\in [6, 7.8]\,\text{m}$. Upon reaching this zone, their final destination is updated to $x=48\,\text{m}$ with $y$-coordinates remaining at the values each agent had upon reaching the intermediate zone. To maintain a constant population density, we implemented periodic boundary conditions by re-spawning agents at the left boundary at $x = 0\,\text{m}$, when agents reach the right boundary at $x = 48\,\text{m}$.

In the counterflow configuration (see Fig.~\ref{fig:Pedestrians}(b), we simulated a crowd of $N=100$ pedestrians divided equally into two groups with opposing flow directions: group~1 moves from left to right, and group~2 from right to left. We used the SFM in Eq.~\eqref{eq:SFMshort} with the same parameters as in the unidirectional case, and initial spatial distributions that are symmetric for each group along $x = 24\,\text{m}$; group~1 follows the same initial distribution as in the unidirectional case. Employing a two-target scheme to induce realistic obstacle-avoidance behavior, we prescribed distinct preferred zones for each group. Group~1 is initially guided to the zone $x=25\,\text{m}$ and $y\in [5.4, 7.8]\,\text{m}$, while group~2 to $x=25\,\text{m}$ and $y\in [9.6, 11.4]\,\text{m}$. Upon reaching these zones, their final destination is updated to $x=48\,\mathrm{m}$ for group~1 and $x=0\,\mathrm{m}$ for group~2, while the $y$-coordinates remain fixed at the values each agent had upon reaching the intermediate zone. This separation promotes smooth obstacle avoidance and organized counterflow without introducing explicit route planning. As in the unidirectional case, periodic boundary conditions are applied at both ends of the corridor by re-spawning agents at the left/right boundary upon exiting at the right/left for groups 1/2, respectively.

In both configurations pedestrians are prohibited from passing the corridor's walls or entering the obstacle area by introducing strong wall forces in the SFM microscopic simulator. Such a configuration ensures a controlled pedestrian flow through the simulated environment and  a realistic obstacle-avoidance behavior. 

\subsection{Data generation of microscopic distributions}
\label{sb:data_gen}
We generated synthetic data of pedestrian trajectories by numerically integrating the SFM with a first-order explicit Euler scheme using a small fixed time step of $\delta t_{micro} = 0.025 \, \text{s}$ to ensure stability and accuracy. Each simulation has a time span of $250\,\text{s}$. To reduce computational storage costs while retaining essential temporal features, we subsample the resulting pedestrian trajectories at intervals of $\delta t = 0.25 \, \text{s}$. As already discussed, we generate pedestrian trajectories from various initial spatial distributions modelling diverse crowd scenarios. In particular, we consider $C=20$
different initializations with zero velocities and positions distributed via: (i) uniform distributions to provide a homogeneous baseline scenario, (ii) Gaussian clusters 
to model localized groups of pedestrians, (iii) ``double bell-shaped" Gaussian distributions to capture the formation of two distinct crowd groups and (iv) cosinusoidal distributions to generate periodic macroscopic behavior
. Details on the form and parameters of the above initializations are provided in \ref{app:intial_conditions}, distinguishing between those used for training and testing the proposed framework in Tables~\ref{tab:Microscopic_distributions} and \ref{tab:Microscopic_distributions_test}, respectively. 

For the unidirectional flow, the resulting datasets are comprised of 20 cases with different initializations, each containing a matrix $\mathcal{X}^{(c)}=\{\bx_i(t_k), i=1,\ldots,N, k=1,\ldots,K \}\in \mathbb{R}^{200\times 1000}$, where $2N=200$ corresponds to the pedestrians' positions and $K=1000$ time steps correspond to the samples recorded from each trajectory. Thus, the total datasets over all cases are $\mathcal{X}_{tr}, \mathcal{X}_{ts} \in \mathbb{R}^{200\times (20\cdot 1000)}$ used for training and testing, respectively. 

For the counterflow, we generated data from the same set of $C=20$ initializations; see \ref{app:intial_conditions}. Here, the initialization for group~1 is identical to that of the unidirectional case (with half the agents), while group~2 is initialized symmetrically to group~1 along $x=24\, \text{m}$, preserving the $y$-positions. Each initialization case now produces two separate data matrices $\mathcal{X}^{(1, c)}$ for group~1 and $\mathcal{X}^{(2, c)}$ for group~2, both of dimension $\mathbb{R}^{100\times 1000}$. Consequently, the complete training and testing datasets for the counterflow are $\{\mathcal{X}_{tr}^{(1)}, \mathcal{X}_{tr}^{(2)}\}$ and $\{\mathcal{X}_{ts}^{(1)}, \mathcal{X}_{ts}^{(2)}\}$, respectively.

\subsection{Extraction of continuous density fields}
\label{sb:den_extr}
For the first step of our framework, we derive macroscopic density fields from pedestrian positions via KDE, as described in Section~\ref{sbsb:pos2den}. To achieve this, we discretize the corridor computational domain $\Omega$ into $n_x \times n_y = 80 \times 20$ control volumes with uniform spacing $\delta x = \delta y = 0.6\,\mathrm{m}$; this resolution reflects a typical personal space for each pedestrian, allowing for capturing emergent collective behavior, while minimizing artificial grid effects. For the KDE implementation on the cell centroids, we choose Gaussian kernels in Eq.~\eqref{eq:pos2den} tuned with bandwidths $\mathbf{H}=diag(h_x, h_y)=(3,2)$ for suppressing noise and preserving features of the microscopic data, while ensuring that local structures remain clear without introducing spurious artifacts. To account for the physical obstacle in $\Gamma\subset\Omega$, we applied a binary mask to the KDE-estimated density field, enforcing zero density in $\Gamma$. This is necessary because KDE can produce non-zero density estimates near obstacle boundaries when pedestrians are in close proximity. To account for the periodic boundary conditions, we further perform domain augmentation for the KDE support. All the details regarding KDE implementation are provided in \ref{app:kde}.

Using the above procedure, in the unidirectional flow case, we derived the density fields $\rho^{(c)}\big(x_i,y_j,t_k\big)$ ($c=1,\ldots,C$, $i=1,\ldots,n_x$, $j=1,\ldots,n_y$ and $k=1,\ldots,K$) from the pedestrian positions included in the training $\mathcal{X}_{tr}$ and testing $\mathcal{X}_{ts}$ data sets. For conforming to the assumptions of Proposition~\ref{prop:1}, we normalize each density field following Eq.~\eqref{eq:denNorm}, thus constructing the normalized density matrices $X_{tr},X_{ts} \in \mathbb{R}^{n_c \times n_t}$ in the form of Eq.~\eqref{Dataset2}, where $n_c=n_x\times n_y=1600$ and $n_t=C\cdot K = 20\cdot 1000$. In the counterflow case, we followed the same procedure to extract and normalize the density fields of each group separately, as it is required by Proposition~\ref{prop:2}, yielding the matrices $X_{tr}^{(1)},X_{ts}^{(1)}\in \mathbb{R}^{n_c \times n_t}$ for group 1 and $X_{tr}^{(2)},X_{ts}^{(2)}\in \mathbb{R}^{n_c \times n_t}$ for group 2. 

\subsection{POD for the restriction and lifting operators}
\label{sb:PODrest_lift}

We construct the restriction and lifting operators by employing POD on the training data, determining the latent dimension $d$ and the POD basis $U_d$ via the economy-size SVD. In all cases, we determine $d$ by the best low-rank approximation retaining $99\%$ of the energy; thus, computing the minimum number satisfying the criterion:
\begin{equation}
E_d = \dfrac{\sum_{i=1}^d \sigma_i^2}{\sum_{i=1}^r \sigma_i^2} \geq 0.99 \label{eq:energycriterion}
\end{equation}
where $r = \min(n_c,n_t) = 1600$ is the rank of the data matrix, and $\sigma_i>0$ for $i=1,\ldots,r$ are the singular values in descending order. 
The basis $U_d$ is then formed by the first $d$ left singular vectors. 

For the unidirectional flow case, we applied this procedure to the training dataset $X_{tr}$ and the resulting basis $U_d$ is then used to construct the restriction and lifting operators in Eqs.~\eqref{eq:Rest} and \eqref{eq:Lift}, respectively. For the counterflow case, the above procedure is done per group, to compute the separate POD bases $U_{d_1}^{(1)}$ and $U_{d_2}^{(2)}$ from the group-specific training datasets $X_{tr}^{(1)}$ and $X_{tr}^{(2)}$, respectively. Following Proposition~\ref{prop:2}, we then computed the SVD of the cross-covariance matrix $C^{(12)}=\frac{1}{n_t}\bar X^{(1)}_{tr} \left(\bar X_{tr}^{(2)}\right)^\top$ of the training sets, which is required to form the augmented bases $\tilde{U}_{d_l}^{(l)}$ (for $l=1,2$) defined in Eq.~\eqref{eq:Aug_Basis}. These augmented bases were subsequently used to construct the group-specific restriction and lifting operators in Eqs.~\eqref{eq:Rest2} and \eqref{eq:Lift2}.

For both configurations, to assess the accuracy of restriction and lifting operators, we evaluate the POD reconstruction error on the training and testing data sets by the relative $L_2$ error: 
\begin{equation}
    e^{2,rec}_k = \dfrac{\lVert x(t_k) - \mathcal{L}(\mathcal{R}(x(t_k)))\rVert_2}{\lVert x(t_k) \rVert_2},
    \label{eq:PODreconErr}
\end{equation}
where $x(t_k)\in X_{tr}, X_{ts}$ denotes the normalized density in Eq.~\eqref{eq:denNorm} at time $t_k$ for $k=1,\ldots,K$. For the counterflow, this error is computed separately for each group using the corresponding operators and data sets $X^{(l)}_{tr}, X^{(l)}_{ts}$ for $l=1,2$.

\subsection{ROMs for latent space dynamics and their training}
\label{sb:ROMs_impl}
Having obtained the latent representations $\boldsymbol{y}(t)\in\R^d$ from the restriction operator, the third step of the proposed approach is to construct ROMs for learning the dynamics in the latent space. For this purpose, we learn linear MVAR and non-linear LSTM models, as described in Section~\ref{sbsb:ROMs}, for predicting the latent variable $\hat{\boldsymbol{y}}_c(t_k)$ of the $c$ case at timestep $t_k$ given the past $w$ lags $\{\boldsymbol{y}_c(t_{k-w}),\ldots,\boldsymbol{y}_{c}(t_{k-1})\}$.Note that in the counterflow, a single ROM is trained on the concatenated latent vectors from both groups $\boldsymbol{y}_c(t_k) = [\boldsymbol{y}^{(1)}_c(t_k), \boldsymbol{y}^{(2)}_c(t_k)]^\top$, as detailed in Section~\ref{sbsb:ROMs}. To conform to sequence-to-one learning, the features consist of the dataset $Y_{f}=\{ \boldsymbol{y}_c(t_{k-w}),\ldots,\boldsymbol{y}_c(t_{k-1})\}\in\mathbb{R}^{d\times w\times C\cdot(K-w)}$ and the targets consist of the data set $Y_{t}=\{ \boldsymbol{y}_c(t_k)\}\in\mathbb{R}^{d\times C\cdot(K-w)}$, where each $\boldsymbol{y}_c(t_k)$ is computed via the restriction operator in Eq.~\eqref{eq:Rest} on the training data $X_{tr}$ (for the unidirectional case) and via Eq.~\eqref{eq:Rest2} on $X_{tr}^{(1)},X_{tr}^{(2)}$ (for the counterflow case).

As a preliminary step, we first tune the lag window $w$, within a bound of $w_{\text{max}} = 20$, using two information-theoretic criteria: the Bayesian Information Criterion (BIC) and the Akaike Information Criterion (AIC) \cite{lutkepohl2005new, akaike1974new, schwarz1978estimating}; details are provided in ~\ref{app:bic_aic}. Both criteria are likelihood-based, with AIC tending to select larger windows than BIC. We note that BIC and AIC are well-suited for MVARs since they assume stationarity of the data and white noise residuals. For comparison, we adopt the same $w$ for learning LSTMs, as the one tuned for MVARs by these criteria.  

For the MVAR model training in both configurations, we first verified that all input time series are stationary using the augmented Dickey--Fuller (ADF) test \cite{Dickey1979} with a threshold of significance level $p < 0.01$. Then, we computed the coefficients $\boldsymbol{A}_j$ in Eq.~\eqref{eq:mvar} using ordinary least squares to solve the optimization problem in Eq.~\eqref{eq:mvar_ls}. For the counterflow case, where the related data-matrix was ill-conditioned, we employed ridge regularization with $\lambda=10^{-6}$.

For the LSTM model in both configurations, we considered a single LSTM layer with $16$ units and a hyperbolic tangent activation function in the hidden layers; the output layer is linear, as denoted in Eq.~\eqref{eq:LSTMout}. To train the LSTM model, we solve the optimization problem in Eq.~\eqref{eq:lstm_ls} using the Adam optimizer \cite{kingma2015adam} with automatic differentiation for the gradients of the loss function, enabled through MATLAB's Deep Learning Toolbox \cite{MATLAB}. The model's parameters are initialized with a Glorot uniform initialization \cite{glorot2010}, the initial learning rate is set to $l_r = 10^{-3}$, and the mini-batch size is set to $32$; the latter hyperparameters were tuned on a trial-and-error basis, providing a good compromise between computational efficiency and generalization without exceeding memory constraints. The maximum number of training epochs was set to $40$ for the unidirectional flow case and to $70$ for the counterflow case, based on the saturation of the training loss, which yielded no additional improvement beyond these epochs. 

To assess the convergence of both MVARs and LSTMs training, we report the values of the loss functions in Eqs.~\eqref{eq:mvar_loss} and \eqref{eq:lstm_loss}, respectively, at the end of training. To further quantify the uncertainty associated with the stochastic training of the LSTM models, we repeated their training procedure 50 times, each with different randomly initialized learnable parameters. We report the mean error of the final losses and the 10–90\% error percentiles across the 50 runs. Since the MVARs loss function is based on $L_2$ error, we further report the MSE at the latent space:
\begin{equation}
MSE_{\text{MVAR}}\left(\{ \boldsymbol{y}_c(t_k) \}, \{\hat{\boldsymbol{y}}_c(t_k)\}; \{\boldsymbol{A}_j\}_{j=1}^w \right) = \dfrac{1}{C(K-w)} \sum_{c=1}^C\sum_{t_k=w+1}^{K} \Big \lVert \boldsymbol{y}_c(t_k)-\hat{\boldsymbol{y}}_c(t_k) \Big \rVert^2_2, 
\label{eq:mvar_loss_MSE} 
\end{equation} 
where $\hat{\boldsymbol{y}}_c(t_k)$ is the prediction in Eq.~\eqref{eq:MVARpred}, to enable direct comparison with the MSE-based loss function of LSTMs. Additionally, we report the computational times required for training both ROMs (for LSTMs, we report the mean error and 10--90\% error percentiles statistics across the 50 runs); we expect MVARs to be much faster than LSTMs due to the lower number of parameters and the method employed for solving the related optimization problem. 

\subsection{Accuracy assessment of the proposed framework}
Having constructed/trained each step of the proposed framework (construction of the restriction/lifting operators and training of the ROMs), we finally assess the forecasting performance of the whole framework on the $C=20$ different initializations of the testing set $\mathcal{X}_{ts}$; see ~\ref{app:intial_conditions} for details. In particular, we first obtain the ``ground truth'' normalized densities for each initialization, say $x^{(c)}(t_0+k\ \delta t)\equiv x^{(c)}(t_k)\in X_{ts}$ for all time steps $k=1,\ldots, K$, using Eq.~\eqref{eq:denNorm}. While ROMs are trained in an open-loop setting using ground truth embeddings, we evaluate the proposed framework in a closed-loop setting (i.e., recursively feeding subsequent predictions to the autoregressive ROMs), as this reflects the intended use in practice. In particular, following the restrict-evolve with ROM-lift framework in Eq.~\eqref{eq:LER_basic}, we compute the closed-loop prediction:

%

\begin{equation}
    \hat{x}^{(c)}(t_k) = \mathcal{L}\left(\hat{\boldsymbol{y}}^{(c)}(t_k)\right), \qquad
    \hat{\boldsymbol{y}}^{(c)}(t_k)=\Phi_{\text{ROM}}\left(
    \hat{\boldsymbol{y}}^{(c)}(t_{k-1}),\ldots, \hat{\boldsymbol{y}}^{(c)}(t_{k-w});\mathbf{p}\right),
    \label{eq:lift_ROMev_rest}
\end{equation}
where $\hat{\boldsymbol{y}}^{(c)}(t_{k-j})$ are latent points predicted recursively by $\Phi_{\text{ROM}}$ in the previous $j=1,\ldots,w$ steps. Closed-loop forecasting requires the initial $j=1,\ldots,w$ embeddings of the observed densities $\mathcal{R}(x^{(c)}(t_{j}))$. Here, 

$\mathcal{R}$ and $\mathcal{L}$ are the restriction and lifting operators (defined in Eqs.~(\ref{eq:Rest},\ref{eq:Lift}) for the unidirectional flow or in Eqs.~(\ref{eq:Rest2}, \ref{eq:Lift2}) for the counterflow), and $\Phi_{\text{ROM}}$ is the latent dynamics ROM with parameters $\mathbf{p}$ that can be either the trained MVAR or LSTM model; for LSTMs, we select the best-performing model among the 50 training runs. For completeness, we provide the open-loop setting and corresponding one-step prediction results in \ref{app:UniFlow} and~\ref{app:CounterFlow} for the unidirectional flow and counterflow cases, respectively. In open-loop, the latent points are obtained from ground truth embeddings $\boldsymbol{y}^{(c)}(t_{k-j})=\mathcal{R}(x^{(c)}(t_{k-j}))$ rather than previous predictions $\hat{\boldsymbol{y}}^{(c)}(t_{k-j})$ in Eq.~\eqref{eq:lift_ROMev_rest}.

For evaluation purposes, the prediction is obtained at every time step $k$. However, in practice, the lifting is applied on-demand only when a high-dimensional reconstruction is required. This renders the scheme particularly efficient, as the evolution is governed by the ROM in the latent space.

To quantify the end-to-end prediction accuracy, we compute the relative, reconstructed in the ambient-space, $L_1$, $L_2$ and $L_\infty$ errors between the ``ground truth" densities $x^{(c)}(t_k)$ and the predicted ones $\hat{x}^{(c)}(t_k)$ 
per case $c$ in the testing set and time step $k$ as:
\begin{equation}
    e^{1,(c)}_k = \dfrac{\lVert x^{(c)}(t_k) - \hat{x}^{(c)}(t_k)\rVert_1}{\lVert x^{(c)}(t_k) \rVert_1}, \quad e^{2,(c)}_k = \dfrac{\lVert x^{(c)}(t_k) - \hat{x}^{(c)}(t_k)\rVert_2}{\lVert x^{(c)}(t_k) \rVert_2}, \quad e^{\infty,(c)}_k = \dfrac{\lVert x^{(c)}(t_k) - \hat{x}^{(c)}(t_k)\rVert_\infty}{\lVert x^{(c)}(t_k) \rVert_\infty}.
    \label{eq:ForcastErr}
\end{equation}
For a summary of statistics, we report the mean and the 10--90\% percentiles of the above errors \emph{over all cases and time steps}. To track down the prediction accuracy in time, we further depict the mean and 10--90\% percentiles of the error evolution in time \emph{over all cases}. For the counterflow, these errors are computed separately for each group $l=1,2$.

All numerical computations were executed on a PC equipped with an 11th Gen Intel(R) Core(TM) i7-1165G7 2.80 GHz CPU and 16 GB of system memory. The source code is publicly available at \url{github.com/patsatzisdim/NGEF_CD.}

\section{Numerical results} 
\label{sec:NumRes}

Here, we present the results of the proposed framework on predicting pedestrian flow in a rectangular corridor with an obstacle for the unidirectional flow and counterflow configurations described in Section~\ref{sb:conf}. We end the section with a comparative discussion about the use of MVAR and LSTM models.

\subsection{Unidirectional flow case study} \label{sb:UniFlow_res}
We first considered the unidirectional flow of pedestrians moving from the left to the right of the corridor as shown in Fig.~\ref{fig:Pedestrians}(a). As already discussed in Section~\ref{sb:data_gen}, we first collected pedestrian trajectory data using the SFM microscopic agent-based simulator under a wide range of initial conditions. The resulting training dataset $\mathcal{X}_{tr}$ is used to train our framework, while the testing dataset $\mathcal{X}_{ts}$ is used to assess its efficiency and prediction accuracy.

In the first step, we extracted the normalized density fields data $X_{tr},X_{ts}$ from the discrete pedestrian positions in $\mathcal{X}_{tr},\mathcal{X}_{ts}$ using KDE, as described in Section~\ref{sb:den_extr}. Next, we employed POD on $\mathcal{X}_{tr}$ to discover the latent space dimension; see Section~\ref{sb:PODrest_lift}. Following the criterion in Eq.~\eqref{eq:energycriterion}, the minimum latent dimension required for retaining more than 99\% of the energy in the training dataset is $d=6$ modes, as shown in Fig.~\ref{fig:PODrecon}(a). 
\begin{figure}[!ht]
    \centering
    \begin{subfigure}[b]{0.47\textwidth}
        \includegraphics[width=\textwidth]{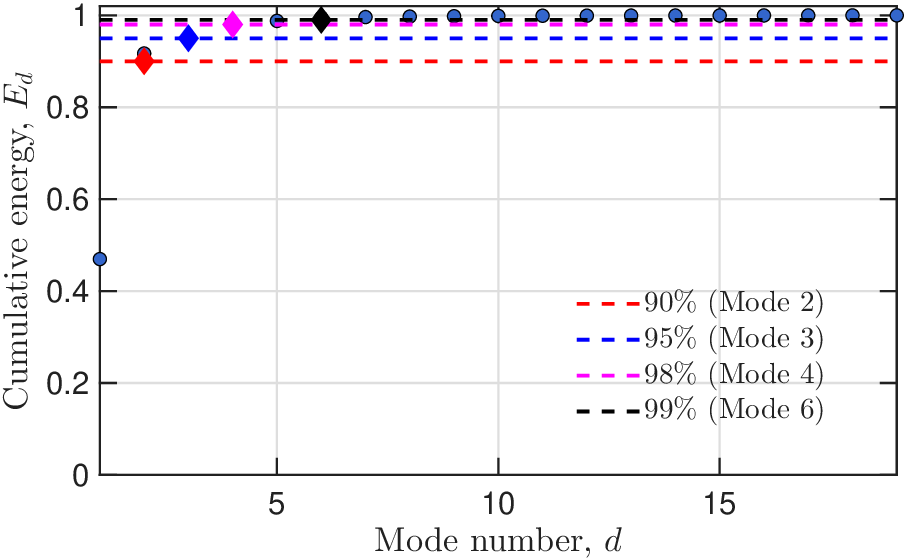}
        \caption{POD cumulative energy}
    \end{subfigure}
    \hfill
    \begin{subfigure}[b]{0.47\textwidth}
        \includegraphics[width=\textwidth]{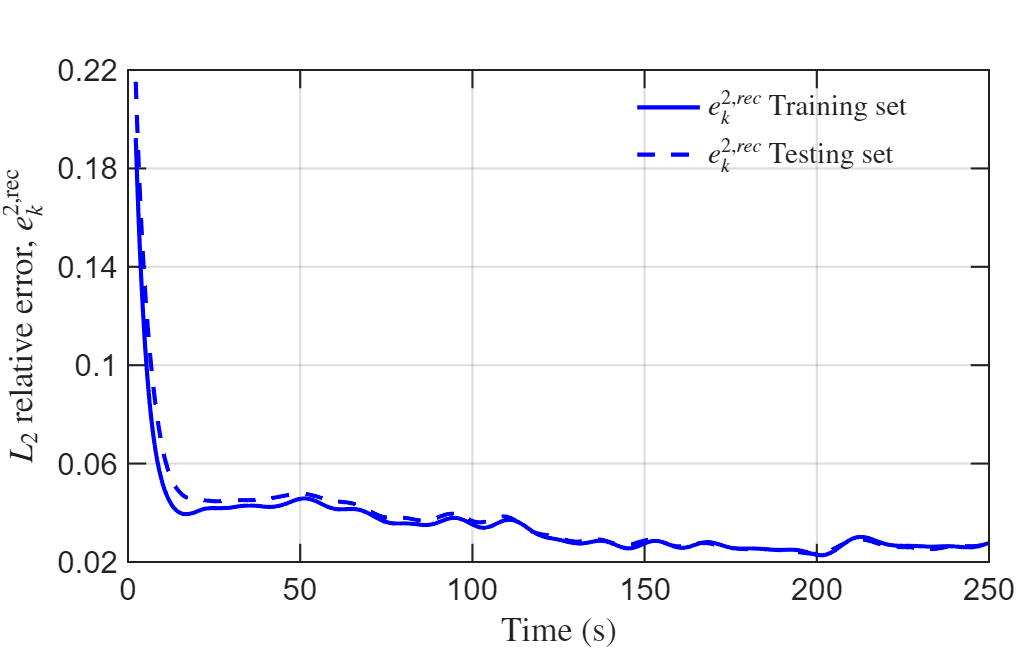}
        \caption{POD reconstruction error}
    \end{subfigure}
    \caption{Accuracy of the restriction and lifting operators in Eqs.~\eqref{eq:Rest} and \eqref{eq:Lift} for the unidirectional flow case. Panel (a) shows the minimum number of POD modes (colored diamonds) required for retaining the desired percentage of energy in the training data set. Panel (b) depicts the average (over the $C=20$ cases with different initial conditions) reconstruction error $e^{2,rec}_k$ in Eq.~\eqref{eq:PODreconErr} with $d=6$ POD modes, in time for the training and testing datasets.}
    \label{fig:PODrecon}
\end{figure}
Using the discovered POD basis $U_d$, we next constructed the restriction and lifting operators in Eqs.~\eqref{eq:Rest} and \eqref{eq:Lift}. To assess their accuracy, we computed the relative $L_2$ reconstruction error $e^{2,rec}_k$ in Eq.~\eqref{eq:PODreconErr} across the $k=1,\ldots,K$ time steps of each trajectory. Fig.~\ref{fig:PODrecon}(b) displays the average relative error \emph{over all the} $C=20$ cases with different initial conditions included in either the training or testing datasets. While initially the relative reconstruction error is $\sim$21\%, it decreases to 4\% at the first $20\ \text{s}$ and further reduces to $\sim$3\% after $50\ \text{s}$ for both seen (training set) and unseen (testing set) data, indicating high accuracy of the restriction and lifting operators with a significantly reduced latent space dimension $d=6$. We note that this level of reconstruction error serves as an \textit{infimum approximation error baseline}, i.e., the best possible error for the proposed framework at the ambient--density profiles--space. ROM forecasts in the long term are therefore expected to surpass this baseline error due to the additional error contributions from model-form error and training/generalization errors.

We next trained MVAR and LSTM models to learn ROMs for the latent dynamics (see Section~\ref{sbsb:ROMs}), where the latent variables $\boldsymbol{y}_c(t_k)\in \mathbb{R}^{6}$ of the $c$ case at timestep $t_k$ 
are obtained from the density profiles of the training set via the restriction operator in Eq.~\eqref{eq:Rest}. To determine the lag window, we employed the BIC and AIC criteria described in \ref{app:bic_aic}; BIC indicated a lag window of $w=4$, whereas AIC favored a wider (as expected) $w = 9$. Both windows were used for training both MVAR and LSTM models, denoted from now on as MVAR(4), MVAR(9), LSTM(4), and LSTM(9) models. The training data sets were divided into features $Y_{f}=\{ \boldsymbol{y}_c(t_{k-w}),\ldots,\boldsymbol{y}_c(t_{k-1})\}$ and targets $Y_{t}=\{ \boldsymbol{y}_c(t_k)\}$.  

For the MVARs training, we first verified that the latent dynamics datasets $Y_{f}, Y_{t}$ constitute stationary time series; the ADF test confirmed stationarity with a significance level of $p < 0.001$ for all $d$ variables. We next trained all MVAR and LSTM models using algorithms, hyperparameters and stopping criteria discussed in detail in Section~\ref{sb:ROMs_impl}. The training results are shown in Table~\ref{tab:performance} for the four ROMs, including the loss function values attained at the end of training and the computational time (in seconds) required. For the LSTMs, as each model was trained over 50 runs with different parameter initializations, we report in Table~\ref{tab:performance} the mean and the 10--90\% error percentiles across the 50 runs. 
\begin{table}[htbp]
\centering
\caption{Training results of MVAR(4), MVAR(9), LSTM(4) and LSTM(9) models at the latent space for the unidirectional flow case.~Loss functions $\mathcal{L}_{\text{MVAR}}\left(\cdot\right)$ in Eq.~\eqref{eq:mvar_loss} for MVARs and $\mathcal{L}_{\text{LSTM}}\left(\cdot \right)$ in Eq.~\eqref{eq:lstm_loss} for LSTMs are reported at the end of training, along with computational times (in seconds) required for training. For the LSTMs, the mean and 10--90\% error percentiles (in parentheses) are reported across the 50 training processes with different parameter initializations. Additionally, the MSE loss function of MVARs in Eq.~\eqref{eq:mvar_loss_MSE} is reported for comparison to the MSE-based loss $\mathcal{L}_{\text{LSTM}}(\cdot)$.}
\label{tab:performance}
\begin{tabular}{l c c c c}
\toprule
\textbf{Model} & $\mathcal{L}_{\text{MVAR}}(\cdot)$ & $MSE_{\text{MVAR}}(\cdot)$ & $\mathcal{L}_{\text{LSTM}}(\cdot)$ & \boldmath{Comput. Time (s)} \\
\midrule
MVAR(4) & \shortstack{$4.65 \times 10^{-8}$} & \shortstack{$3.89 \times 10^{-13}$} & \shortstack{-} & \shortstack{$0.0013$} \\ 
MVAR(9) & \shortstack{$2.04 \times 10^{-8}$} & \shortstack{$1.72 \times 10^{-13}$} & \shortstack{-} & \shortstack{$0.0017$} \\
LSTM(4) & \shortstack{-} & \shortstack{-} & 
\shortstack{$7.12\ (6.27, 8.16) \times 10^{-9}$} 
& \shortstack{$62\ (55,66)$} \\
LSTM(9) & \shortstack{-} & \shortstack{-} & 
\shortstack{$1.51\ (1.30, 1.81) \times 10^{-8}$} 
& \shortstack{$86\ (77,89)$}
\\
\bottomrule
\end{tabular}
\end{table}
Since the loss function of MVARs is $L_2$-based, we additionally compute the MSE in Eq.~\eqref{eq:mvar_loss_MSE} to enable direct comparison with the MSE-based LSTM loss. MVARs converge to much smaller loss function values in comparison to LSTMs. As expected, LSTMs require significantly higher computational time for training in comparison to MVARs. Table~\ref{tab:performance} further shows that the training of both types of ROMs with increased lag window requires more computational time. 

\paragraph{Recursive/Closed-loop} Predictions
Having constructed and trained all components of our framework, we evaluated its end-to-end performance in the ambient--density profile--space. While the ROMs were trained in an open-loop setting, we assess them in the practical closed-loop setting (see Eq.~\eqref{eq:lift_ROMev_rest}), where only the first $w$ observations are supplied and all subsequent predictions are generated autoregressively. For LSTM models, we consider the best-performing model from the 50 training runs. We quantify the prediction accuracy via the relative $L_1$, $L_2$, and $L_\infty$ errors, $e^{1,(c)}_k,\ e^{2,(c)}_k,\ e^{\infty,(c)}_k$ in Eq.~\eqref{eq:ForcastErr}, computed between the ``ground-truth" density profiles $x^{(c)}(t_k)$ and the predicted/reconstructed ones $\hat{x}^{(c)}(t_k)$ for the $C$, unseen in training, cases (of different initial conditions) included in the testing set $X_{ts}$ across time steps $k$.

Table \ref{tab:errors_CL} provides a summary of the closed-loop reconstructed errors, including the mean and the 10--90\% error percentiles, aggregated \emph{over all cases and time steps}.
\begin{table}[htbp]
\centering
\caption{Closed-loop/recursive prediction relative errors in the ambient--density profile--space, for the testing set using the trained MVAR and LSTM (best out of 50 training runs) models for the unidirectional flow case. The relative reconstructed $L_1$, $L_2$ and $L_\infty$ errors $e^{1,(c)}_k,\ e^{2,(c)}_k,\ e^{\infty,(c)}_k$ in Eq.~\eqref{eq:ForcastErr} are reported for the MVAR(4), MVAR(9), LSTM(4), and LSTM(9) latent dynamics models. Mean and 10--90\% error percentiles are shown \emph{over all the} $C=20$ cases and $K=\{991,996\}$ time steps (for width $w=\{9,4\}$, respectively).}
\label{tab:errors_CL}
\begin{tabular}{lccc}
\toprule
\textbf{Model} &  $L_1$ error, $e^{1,(c)}_k$ & $L_2$ error, $e^{2,(c)}_k$ & $L_\infty$ error, $e^{\infty,(c)}_k$ \\
\midrule
MVAR(4) & \shortstack{$0.180\ (0.085,0.246)$} & \shortstack{$0.153\ (0.077,0.211)$} & \shortstack{$0.172\ (0.095,0.242)$} \\
MVAR(9) & \shortstack{$0.163\ (0.087,0.228)$} & \shortstack{$0.140\ (0.079,0.194)$} &  \shortstack{$0.160\ (0.092,0.230)$} \\
LSTM(4) & \shortstack{$0.177\ (0.128,0.235)$} & \shortstack{$0.156\ (0.115,0.209)$} & \shortstack{$0.185\ (0.133,0.243)$} \\
LSTM(9) & \shortstack{$0.165\ (0.112,0.223)$} & \shortstack{$0.142\ (0.110, 0.190)$} & \shortstack{$0.166\ (0.116,0.227)$} \\
\bottomrule
\end{tabular}
\end{table}
As shown in Table~\ref{tab:errors_CL}, the proposed framework with the use of MVAR models slightly outperforms that of LSTMs in the closed-loop setting across error metrics. The MVAR(9) achieves the lowest errors, indicating accuracy in long-term predictions. The MVAR(4) also performs well, with broader percentiles across test cases. Increasing the lag window from $w=4$ to $w=9$ slightly improves MVAR accuracy, demonstrating the benefit of incorporating longer temporal dependencies within the linear autoregressive structure. The LSTM models exhibit slightly higher errors than MVARs, with tighter 10--90\% percentiles, while the increase of lag window shows a similar to MVARs improvement.

We next assessed the evolution/accumulation of the relative reconstructed errors $L_2$ and $L_\infty$ over time. Figure~\ref{fig:All_Errors_CL} depicts the average reconstructed relative errors in the ambient--density profile--space, and their 10--90\% error percentile ranges across the multiple cases of different, unseen initial conditions in the testing set for all four ROMs.
\begin{figure}[htbp]
    \centering
    \begin{subfigure}[b]{0.47\textwidth}
        \includegraphics[width=\textwidth]{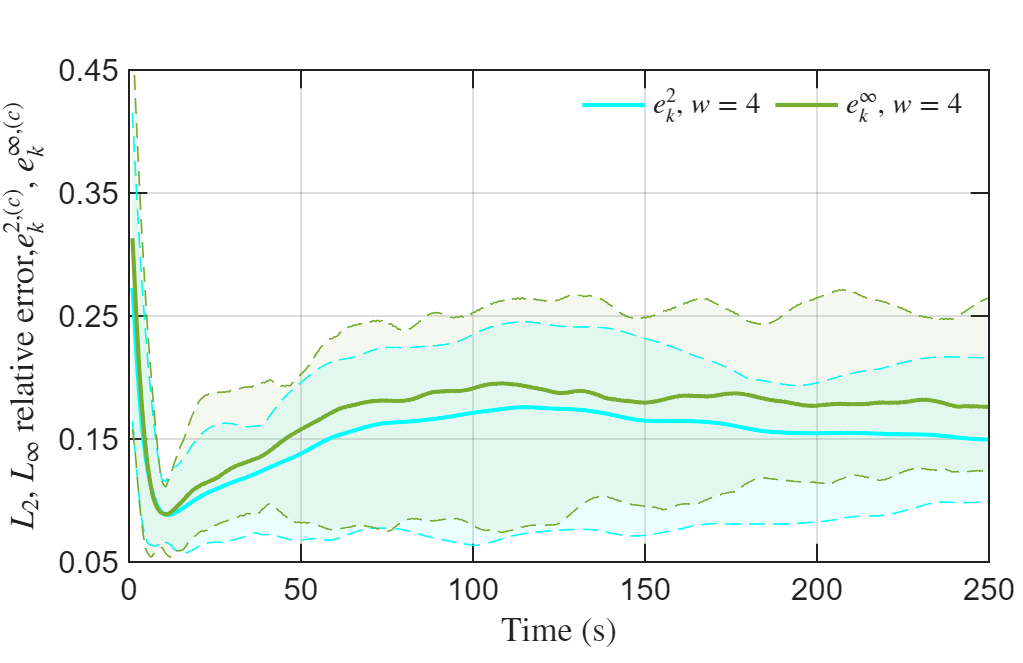}
        \caption{Relative Errors with MVAR(4)}
    \end{subfigure}
    \hfill
    \begin{subfigure}[b]{0.47\textwidth}
        \includegraphics[width=\textwidth]{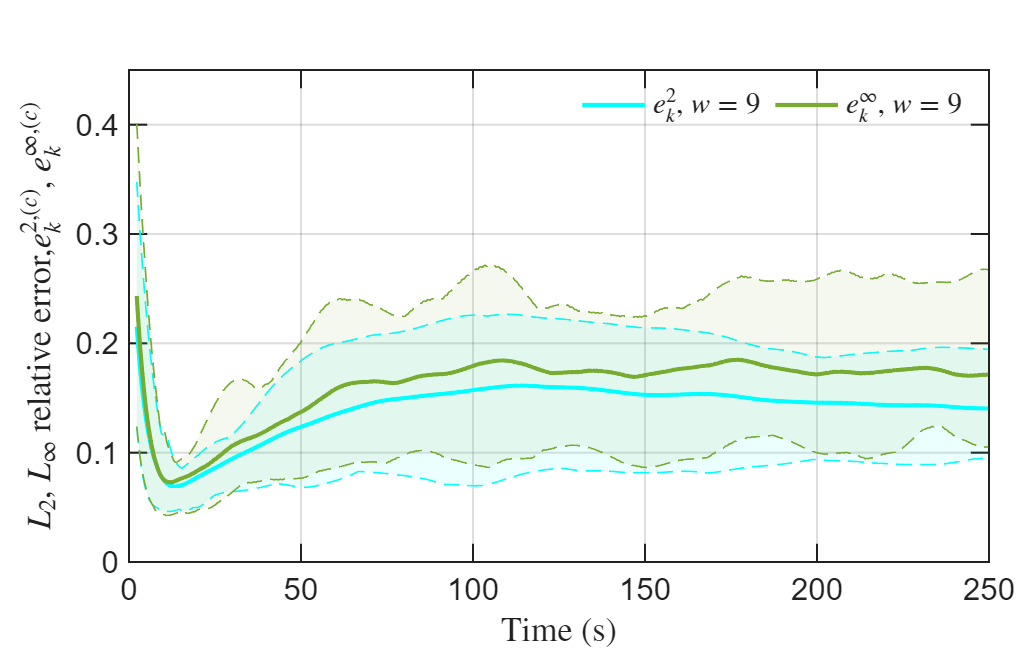}
        \caption{Relative Errors with MVAR(9)}
    \end{subfigure}
    
    \begin{subfigure}[b]{0.47\textwidth}
        \includegraphics[width=\textwidth]{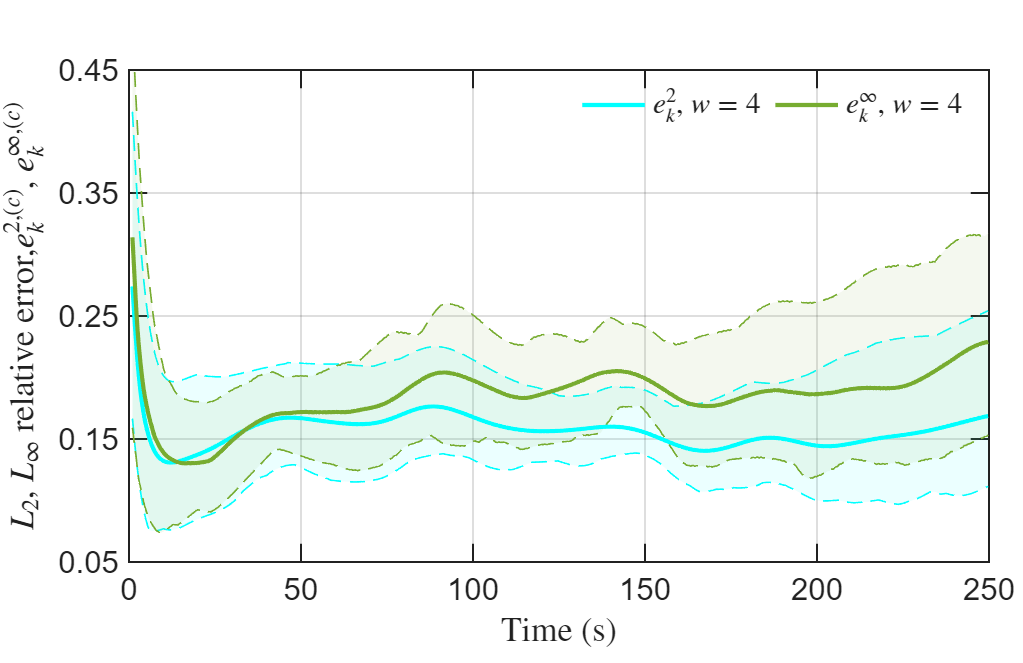}
        \caption{Relative Errors with LSTM(4)}
    \end{subfigure}
    \hfill
    \begin{subfigure}[b]{0.47\textwidth}
        \includegraphics[width=\textwidth]{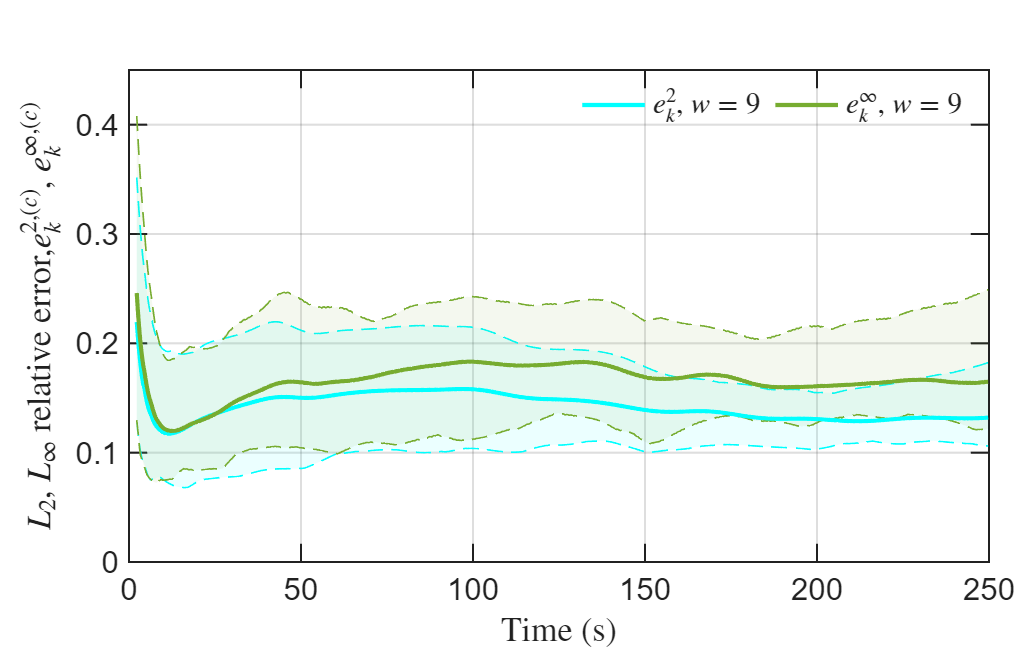}
        \caption{Relative Errors with LSTM(9)}
    \end{subfigure}
    
    \begin{subfigure}[b]{0.47\textwidth}
        \includegraphics[width=\textwidth]{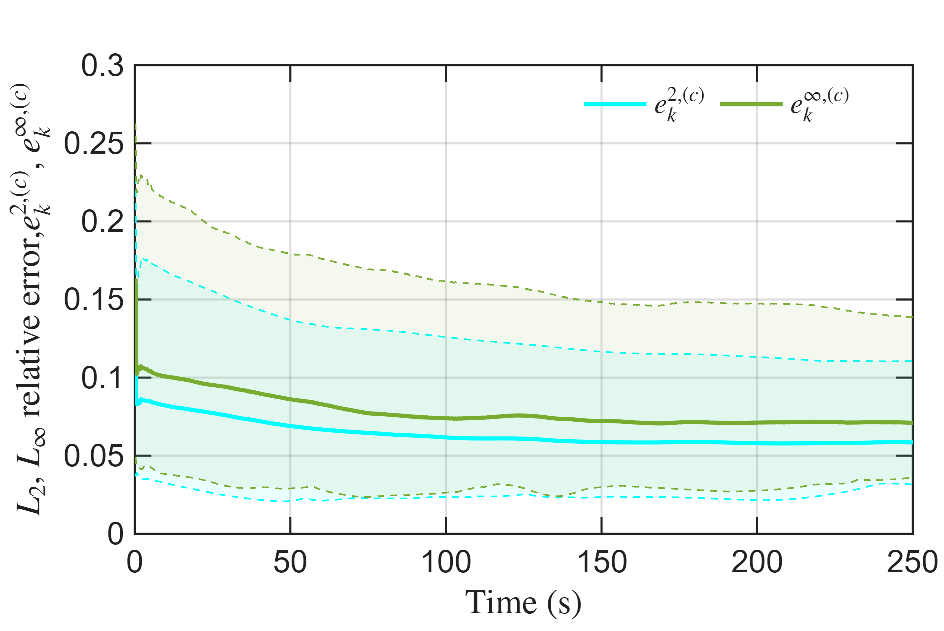}
        \caption{Relative Errors with SFM, $+1\%$ perturbation}
    \end{subfigure}
    \hfill
    \begin{subfigure}[b]{0.47\textwidth}
        \includegraphics[width=\textwidth]{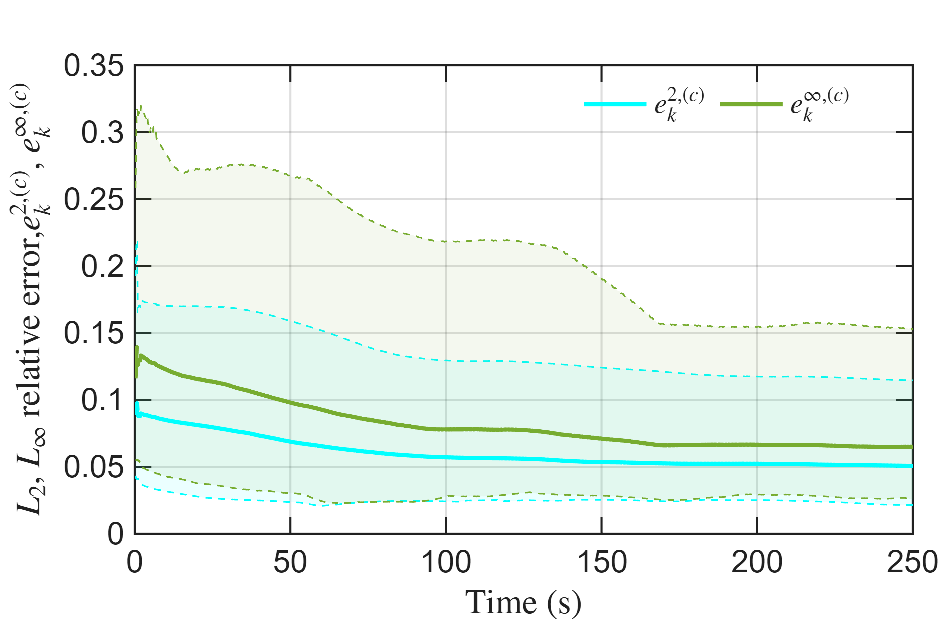}
        \caption{Relative Errors with SFM, $-1\%$ perturbation}
    \end{subfigure}
    \caption{Closed-loop/recursive simulation errors in the ambient--density profile--space, across the testing set over time, using the trained MVAR and LSTM models for the unidirectional flow case. Panels (a)--(d) display the relative $L_2$ (cyan) and $L_\infty$ (green) errors $e^{2,(c)}_k$ and $e^{\infty,(c)}_k$ in Eq.~\eqref{eq:ForcastErr} for MVAR(4), MVAR(9), LSTM(4) and LSTM(9) models, respectively. Panels (e) and (f) provide a system's ``baseline'' relative error estimation, obtained from the SFM simulations using $\pm1\%$ perturbations in the initial conditions for the density profiles (see in Section~\ref{sb:conf}). Mean relative error (solid) and 10--90\% percentiles (dashed) are shown over $C=20$ cases per time step. For the LSTM models, results correspond to the model achieving the best training performance out of the 50 independent runs.
    \label{fig:All_Errors_CL}}
\end{figure}
Indicatively, for our illustrations, to provide a comparative ``baseline'' of the reconstructed error for tiny perturbations of the density profiles, we quantified the sensitivity of the SFM simulations per se with respect to small perturbations of the initial conditions (see in Section~\ref{sb:conf}). Specifically, we introduced $\pm1\%$ perturbations in the initial conditions of the density profiles of the reference test cases and computed the relative reconstructed errors between these perturbed initial conditions and their unperturbed counterparts, as shown in Fig.~\ref{fig:All_Errors_CL}(e,f).

The time evolution in Fig.~\ref{fig:All_Errors_CL} reveals a characteristic pattern across all models: errors are higher initially, decrease drastically during the first $20\, \text{s}$ as the POD reconstruction error baseline diminishes, and then slowly increase towards the end of the simulation due to error accumulation in the recursive forecasting. The MVAR models in Fig.~\ref{fig:All_Errors_CL}(a,b) demonstrate remarkable approximation accuracy with the mean relative $L_2$ error remaining bounded below $\sim$14\%(15\%) for the MVAR(9)(MVAR(4)) models by the end of the simulation horizon. The respective 10--90\% error percentiles are bounded in (7, 25)\% errors for the MVAR(4) model, while for the MVAR(9) model they tighten further to (10, 20)\%. Both MVAR models achieve errors that compare favorably with the intrinsic system baseline obtained from $\pm1\%$ perturbed SFM simulations, which shows mean $L_2$ error of $\sim 8\%$ and 10--90\% percentiles $(3,32)\%$ (Fig.~\ref{fig:All_Errors_CL}(e,f)). The LSTM models in Fig.~\ref{fig:All_Errors_CL}(c,d) provide slightly less accurate predictions than the MVAR models. By the end of the simulation, the LSTM(4) mean relative $L_2$ error reaches approximately $16\%$, with 10--90\% percentiles widening to $(11,25)\%$. The LSTM(9) improves marginally, with mean error around $13\%$ and tighter percentiles $(10,18)\%$. The MVAR and LSTM models demonstrate similar trends, with MVARs providing slightly higher accuracy; a result confirming the averaged in time results reported in Table~\ref{tab:errors_CL}.

For completeness, we directly compared the ground truth normalized density fields $x^{(c)}(t_k)$ and the predicted fields $\hat{x}^{(c)}(t_k)$ generated by the closed-loop/recursive simulations. The predictions of the proposed framework for all ROMs (see Figs~\ref{fig:MVARw4den}--\ref{fig:LSTMw9den}) successfully preserve the global structure and propagation direction of the ground truth density, demonstrating the framework's capacity to capture dominant crowd dynamics and produce meaningful long-horizon predictions. The spatial distribution of the absolute errors between ground truth and closed-loop predictions (see Fig.~\ref{fig:Unidirectional_flow_error}) further shows that the all models correctly respect the obstacle regions. The primary observed discrepancy is a minor underestimation of the peak density values, and of density in regions immediately downstream and upstream of these peaks coupled with the introduction of artificial density in low-density regions; the latter is a likely consequence of enforcing mass conservation via the POD restriction operator. These artifacts are fainter in the MVAR(9) and LSTM(9) models, and become slightly stronger in the less accurate MVAR(4) and LSTM(4) models, showing a lessening effect as time passes (see Fig.~\ref{fig:Unidirectional_flow_error}). 

Finally, the open-loop (one-step) prediction results are provided in \ref{app:UniFlow}; see Table~\ref{tab:errors_OL} and Fig.~\ref{fig:Errors_openloop}. Similarly to the closed-loop findings above, all four ROMs achieve comparable accuracy in the open-loop setting, with mean $L_2$ errors of approximately $0.038$ across models and narrow 10--90\% percentiles. This indicates that while both models capture the local dynamics very accurately in one-step predictions ($L_2$ errors just above the POD baseline), their long-term forecasting is, in practice, affected by error accumulation that can degrade their performance.

\subsection{Counterflow case study} \label{sb:Counterflow_res}

For the second case study, we applied the proposed framework in a counterflow configuration in which two interacting populations are divided equally into two groups with opposing flow directions; see Fig.~\ref{fig:Pedestrians}(b). Following the same procedure as in the unidirectional flow case (for implementation details see Sections~\ref{sb:data_gen}--\ref{sb:PODrest_lift}), we constructed the restriction and lifting operators in Eqs.~\eqref{eq:Rest2} and \eqref{eq:Lift2} by retaining $d_1=6$ and $d_2=8$ POD modes per group $l=1,2$ and $m=4$ cross-covariance modes. The accuracy of the resulting operators is very high, with relative $L_2$ errors reducing to $3-4\%$ after the first $50\, \text{s}$ for both groups (see Fig.~\ref{fig:PODrecon_co} in \ref{app:CounterFlow}).

With latent variables $\boldsymbol{y}_c(t_k)=[\boldsymbol{y}^{(1)}_c(t_k),\boldsymbol{y}^{(2)}_c(t_k)]\in \mathbb{R}^{24}$ determined, we trained MVAR and LSTM models following Section~\ref{sb:ROMs_impl}. The AIC-identified lag was $w=10$, and stationarity was verified (ADF test). The MVAR(10) again converged to lower loss function values with substantially less computational time than the LSTM(10) one (see Table~\ref{tab:performance_counter} in \ref{app:CounterFlow}), though both models show deteriorated performance compared to the unidirectional flow case.

To evaluate the end-to-end performance of the proposed framework in the ambient--density profile--space in long-term (closed-loop) predictions, we followed the same quantification procedure as in Section~\ref{sb:UniFlow_res}. We again observe that the closed-loop prediction accuracy with the use of MVAR is superior to that of LSTM (see Table~\ref{tab:errors_CL_counterflow} in \ref{app:CounterFlow}). In particular, the mean, overall testing cases and time steps, reconstructed $L_2$ error of the MVAR(10) model is $8\%$ with tight 10--90\% percentiles, while that of the LSTM(10) model is $\sim10\%$ with wider percentiles for both population groups. The detailed error evolution over time is displayed in Fig.~\ref{fig:Errors_closedloop_co}. The MVAR(10) model shows mean $L_2$ errors dropping to $\sim 9-10\%$ within the first $20\, \text{s}$ for both population groups. These errors remain bounded over the long-time horizon at values slightly over $10\%$, exhibiting high accuracy in comparison to the baseline accuracy of pure restriction and lifting operators ($3-4\%$ after the first $50\, \text{s}$ for both groups; see Fig.~\ref{fig:PODrecon_co}). In contrast, LSTM(10) exhibits slightly larger errors that stabilize at $13-14\%$ by the simulation end. Across both models, group 2 shows slightly higher errors and wider uncertainty bands.

\begin{figure}[!h]
    \centering
    \begin{subfigure}[b]{0.47\textwidth}
        \includegraphics[width=\textwidth]{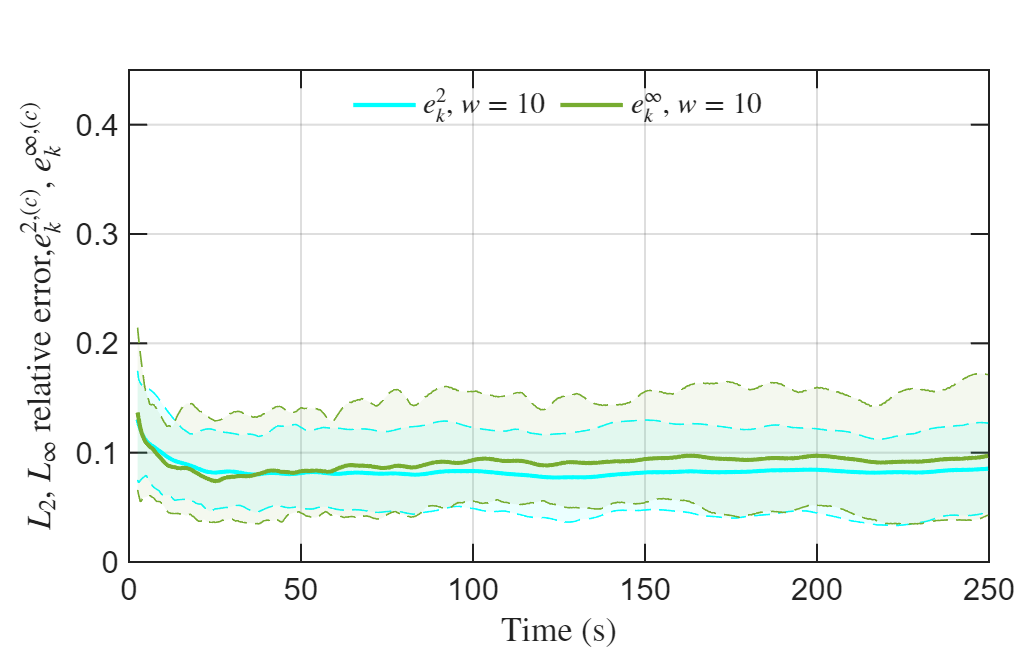}
        \caption{Relative Errors with MVAR(10) for group 1}
    \end{subfigure}
    \hfill
    \begin{subfigure}[b]{0.47\textwidth}
        \includegraphics[width=\textwidth]{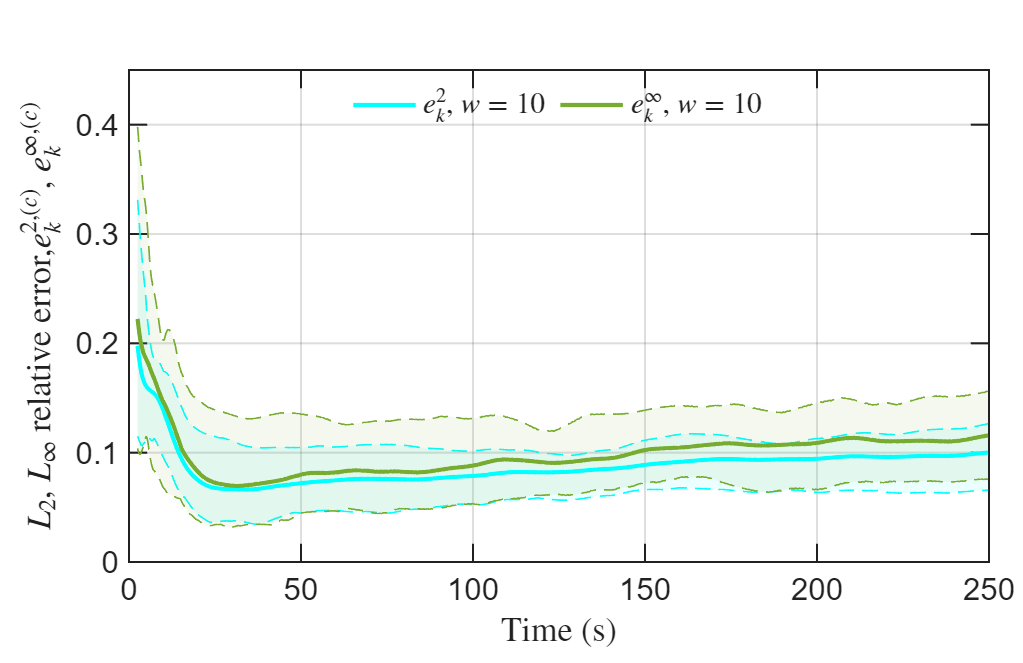}
        \caption{Relative Errors with MVAR(10) for group 2}
    \end{subfigure} 
    \begin{subfigure}[b]{0.47\textwidth}
        \includegraphics[width=\textwidth]{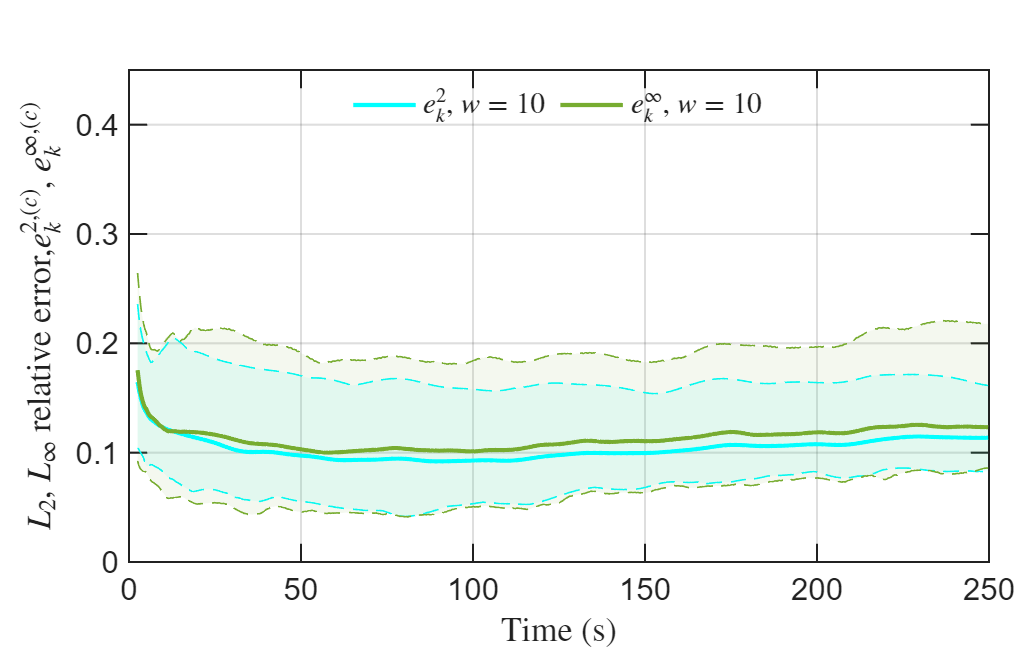}
        \caption{Relative Errors with LSTM(10) for group 1}
    \end{subfigure}
    \hfill
    \begin{subfigure}[b]{0.47\textwidth}
        \includegraphics[width=\textwidth]{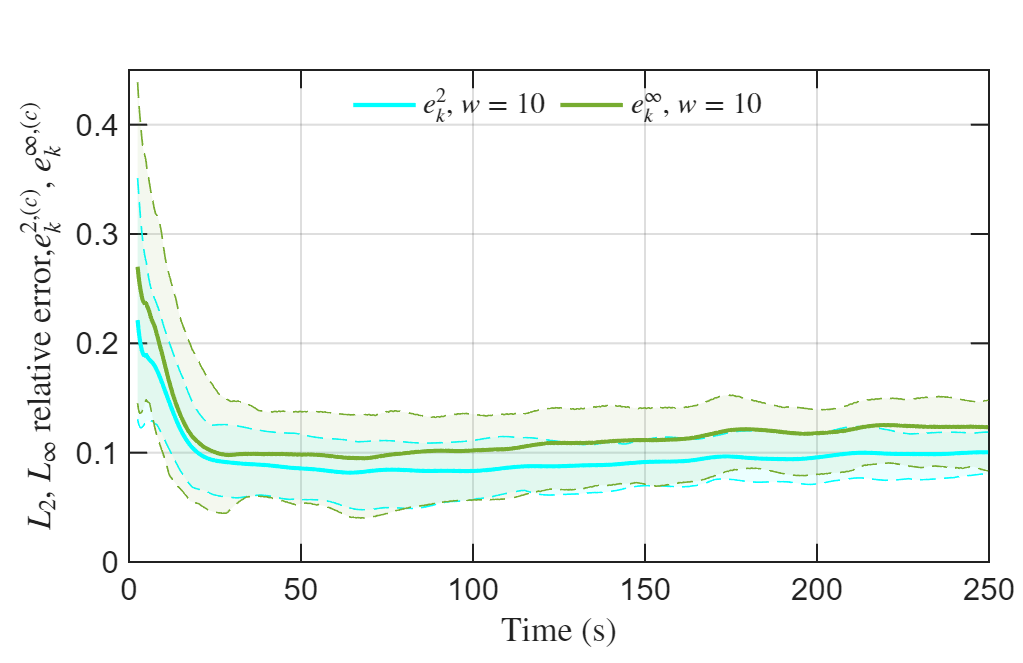}
        \caption{Relative Errors with LSTM(10) for group 2}
    \end{subfigure}
    \caption{Closed-loop/recursive simulation errors in the ambient--density profile--space, across the testing set over time, using the MVAR and LSTM models for the counterflow case. Panels (a,b) and (c,d) display the relative reconstructed $L_2$ (cyan) and $L_\infty$ (green) errors $e^{2,(c)}_k$ and $e^{\infty,(c)}_k$ in Eq.~\eqref{eq:ForcastErr} for MVAR(10) and LSTM(10) models, respectively. The left and right column panels show errors for group 1 and 2, respectively. Mean relative errors (solid) and 10--90\% error percentiles (dashed) are reported over $C=20$ cases per time step. For the LSTM model, results correspond to the run achieving the best training performance out of the 50 independent initializations.}
    \label{fig:Errors_closedloop_co}
\end{figure}

Additionally, we provide a direct comparison of the ground truth densities with their closed-loop predictions for both groups (see Figs.~\ref{fig:Counterflow_densities_MVAR} and \ref{fig:Counterflow_densities_LSTM} in \ref{app:CounterFlow}). Both MVAR and LSTM models accurately reproduce the two dominant density lobes and capture migration and alignment patterns of the two interacting groups. These results, together with the spatial distribution of the absolute errors between ground truth and closed-loop predictions (see Fig.~\ref{fig:Counterflow_Error}), further show that both models respect obstacle regions, without significant errors accumulating in the immediate obstacle proximity, while accurately reproducing localized congestion zones formed upstream and downstream of the obstacle. Minor discrepancies include peak attenuation and spatial dispersion of the Gaussian-like lobe structure of each group, which are slightly more pronounced with the LSTM(10) model, especially as time progresses; a finding that is consistent with LSTM's higher error metrics compared to MVAR in Fig.~\ref{fig:Errors_closedloop_co}. 

For completeness, we also provide the open-loop/one-step prediction results in Table~\ref{tab:errors_CL_counterflow} and Fig.~\ref{fig:Errors_openloop_co} of \ref{app:CounterFlow}. As in the unidirectional flow case, both MVAR and LSTM models achieve comparable accuracy with mean $L_2$ errors $\sim4-6\%$ and very narrow 10--90\% percentiles. This very accurate performance in one-step predictions demonstrates once again the effect of error accumulation introduced in the recursive predictions for long-term horizons.

Finally, the proposed framework achieves higher accuracy in the counterflow case than in the unidirectional flow case, and this difference is consistent across ROM model selections. This may be attributed to the larger latent dimension ($24$ vs $6$), which provides a richer representation space. However, this comes at increased computational cost, highlighting the accuracy-efficiency trade-off inherent in dimensionality selection.

\subsection{MVAR vs. LSTM performance}
The comparison between model classes in both unidirectional and counterflow case studies reveals that the use of MVARs provides a slightly more accurate and reliable framework for long-term predictions than that of LSTMs. This is clearly demonstrated by their lower errors and narrower error percentiles in the closed-loop setting (Table~\ref{tab:errors_CL}, Fig.~\ref{fig:All_Errors_CL} and Table~\ref{tab:errors_CL_counterflow}, Fig.~\ref{fig:Errors_closedloop_co}). While all models perform remarkably well in an open-loop setting (see Table~\ref{tab:errors_OL}, Fig.~\ref{fig:Errors_openloop} and Table~\ref{tab:errors_CL_counterflow}, Fig.~\ref{fig:Errors_openloop_co}), the effect of error accumulation in recursive predictions degrades their predictive accuracy. For instance, in the unidirectional flow case the MVAR(4) model's mean $L_2$ error across time increased from $\sim$4\% (open-loop) to $\sim$15\% (closed-loop), and the LSTM(4) model's error increased from $\sim$4\% to $\sim$16\%. This consistent pattern in both case studies with different latent dimensions ($6$ vs $24$) and interaction complexities is attributed to the challenge of error accumulation in closed-loop forecasting; the linear structure of the MVAR model proves to be less susceptible to this instability than complex, nonlinear LSTMs. While LSTMs can theoretically capture richer nonlinear dynamics, their performance degrades due to the ``curse of dimensionality'' in their training, i.e., finding a global optimum, accompanied by a distribution shift \cite{Bengio2015Scheduled,Goodfellow-et-al-2016} (i.e., the error accumulation of the recursive scheme during inference, that shifts the distribution upon which the models were trained), despite their ability to learn the one-step dynamics effectively. In addition, there are several challenges regarding the ``correct'' initialization of LSTMs \cite{kemeth2021initializing}, especially in dynamical systems and multiscale modeling tasks. In contrast, the training of the linear MVAR models is based on least squares regression, thus giving a global optimum.  For the particular applications, the results suggest that MVARs are a better choice for long-term horizon predictions in the latent space. This is a result in line also with other studies (see e.g., \cite{papaioannou2022time}).

The proposed framework offers a balanced trade-off between prediction accuracy and computational efficiency for coarse-grained crowd dynamics.

For the unidirectional flow case ($250\, \text{s}$ horizon, $K=1000$ steps), one conventional simulation requires $80\, \text{s}$ for SFM integration plus approximately $18\, \text{s}$ for density field extraction via KDE, totaling $98\, \text{s}$ per complete prediction. Our framework’s online closed-loop prediction requires $w \times (t_{ID}+t_R)\, \text{s}$ for density extraction and restriction of the first $w$ lagged snapshots, followed by $(1000-w)\times(t_F+t_L)$ for per-step ROM forecasting and lifting operations; see Eq.~\eqref{eq:lift_ROMev_rest}. This results in total online execution times of $0.397$, $0.887$, $1.413$, and $ 1.871\,\text{s}$ for the MVAR(4), MVAR(9), LSTM(4), and LSTM(9) models, respectively, corresponding to speed-ups of $247\times$, $110\times$, $69\times$, and $52\times$ relative to the conventional approach. For the counterflow case, one conventional simulation requires a total of $77.35\,\text{s}$, while our framework yields total online execution times of $0.781$ and $2.141\,\text{s}$ for the MVAR(10) and LSTM(10) models, respectively, corresponding to speed-ups of $99\times$ and $36\times$. These online times include lifting at every time step for a fair comparison with the conventional simulation output. In practice, with on-demand lifting (i.e., applying $\mathcal{L}$ only when density fields are needed), the speed-ups would be even greater, as the framework evolves in the latent space (Eq.~\eqref{eq:lift_ROMev_rest}).

Offline costs (one-time) for the unidirectional flow case include $(80+18)\times 20 = 1960\,\text{s}$ for generating $20$ training trajectories via SFM with KDE extraction, $4.24\,\text{s}$ for POD basis computation, and model training times of $0.0013$, $0.0017$, $62$, and $86\,\text{s}$ for MVAR(4), MVAR(9), LSTM(4), and LSTM(9) models respectively (the LSTMs times are averaged over 50 runs; see Table~\ref{tab:performance}). For the counterflow case, the offline costs are $1547\,\text{s}$ for generating $20$ training trajectories, $15.51\,\text{s}$ for the augmented POD basis, and model training times of $0.031$ and $104\,\text{s}$ for MVAR(10) and LSTM(10) models respectively (see Table~\ref{tab:performance_counter}). While offline costs appear significant, they become negligible when amortized over multiple predictions. For example, the total offline cost for LSTM(9) in the unidirectional flow case ($2050.24\,\text{s}$) equals just $21$ conventional simulations; a cost quickly recouped when dozens or hundreds of predictions are required, as in parameter studies or real-time applications. 

Overall, MVAR variants prove particularly efficient for real-time applications, while maintaining better accuracy than their LSTM counterparts. These gains are especially valuable for real-world scenarios, where the number of pedestrians renders the conventional approaches infeasible.


\section{Conclusions}
In this work, we present a four-step data-driven framework that combines manifold and machine learning to construct a discrete evolution operator for emergent crowd dynamics. Our framework provides a computationally efficient and systematic approach for learning ROMs for spatio-temporal emergent crowd dynamics. By first encoding the full high‐dimensional state into a compact latent representation, we learn the system’s evolution operator using surrogate ROMs, bypassing the prohibitive computational cost of learning PDEs via, for example, CNNs or other high‐dimensional ML models such as DNNs. Working in latent spaces with delayed POD coordinates simplifies the enforcement of physical constraints (such as mass conservation) over long time horizons, since, as we demonstrate, the POD algorithm used to find an appropriate set of coordinates of the latent space explicitly conserves the mass of the reconstructed dynamics. In fact, under the usual generic assumptions of Takens’/ Whitneys' embeddings theorems, applying time-delay embeddings to POD modes produces a delay-coordinate latent model that is diffeomorphic to the original system dynamics (see also results and discussion in \cite{papaioannou2022time,shvartsman1998low,kemeth2022learning,axaas2023model,dylewsky2022principal}). For our illustrations, we assessed our approach using data generated from the SFM in a corridor with an obstacle under both unidirectional and counterflow scenarios, though more advanced crowd models incorporating behavioral aspects or optimal path planning \cite{bellomo2008modelling,cristiani2023all} can be readily employed.

Our approach does not attempt to learn any closed‑form PDE, but instead learns the effective solution operator of such an unavailable PDE, mapping initial crowd density to its future evolution. The resulting ROMs trained in the latent space, in the form of linear (MVARs) or nonlinear multivariate autoregressive models (such as LSTMs), encode and reconstruct this operator numerically via the POD reconstruction. As a result, our framework allows for rapid and accurate ``what‐if" scenarios exploration and real‐time forecasting of crowd dynamics. Critically, we demonstrate that linear MVAR models within our framework not only outperform more complex LSTMs in long-term predictive accuracy and reliability, but also achieve a computational speed-up of over two orders of magnitude compared to conventional simulation. This combination of interpretability, speed, and accuracy makes our approach, particularly with MVARs, an efficient approach for large‐scale, real-world crowd‐dynamics analysis and distributed control \cite{cristiani2015modeling,albi2020mathematical,auletta2022herding,gong2023crowd,christofides2002nonlinear, armaou2004time,maffettone2022continuification}. The predictive performance of linear MVAR models in our framework (consistently matching or exceeding that of nonlinear LSTMs) aligns with a growing body of literature \cite{papaioannou2022time,gallos2024data,axaas2023model, dylewsky2022principal} demonstrating that linear ROMs constructed in judiciously chosen latent spaces (e.g., POD-based delay-coordinate embeddings, Koopman operator) can effectively approximate complex nonlinear or chaotic dynamics. For example, nonlinear manifold learning techniques such as Diffusion Maps have been shown to provide more accurate embeddings of complex dynamics compared to POD, particularly during transients \cite{della2024learning}.

As formulated in its vanilla form, the proposed method inherits the well-known limitations of POD-based reduced-order models, and in particular their strong dependence on the geometries and boundary conditions represented in the training data, and the present algorithm is no exception: if the test configuration departs substantially away from the snapshot ensemble, the linear subspace may fail to represent the solution manifold accurately.  Handling boundary-aware cases would require a different modeling approach, such as NOs. Incorporating NOs in latent spaces within the proposed framework is a promising direction which will be investigated in future studies, including a broader comparison with alternative surrogate models as for example via Transformer-based architectures \cite{hao2023gnot,shih2025transformers} that have shown a good performance for long-horizon modeling, and latent Neural ODEs \cite{rubanova2019latent}, which offer a continuous-time perspective for learning reduced dynamics. Moreover, approaches based on Learning Effective Dynamics (LED) \cite{kemeth2022learning,champion2019data,vlachas2022multiscale,Dietrich2023} and its extensions, such as G-LED \cite{gao2024generative} and interpretable LED \cite{menier2025interpretable}, provide complementary frameworks that emphasize structure preservation and interpretability of the learned operators. A vital extension of this work for real-world deployment is the transition to probabilistic forecasting to quantify predictive uncertainty. While techniques such as Monte Carlo dropout or quantile regression could be applied, the linear Gaussian structure of the MVAR model offers a particularly direct and interpretable pathway to achieve this by propagating initial condition uncertainty through its stochastic latent dynamics. Finally, the proposed mass-conserving framework provides a foundation for applications to real-world trajectory data, systematic benchmarking against alternative surrogates, and extension to other collective motion systems such as vehicular traffic.

\section*{Acknowledgments}
C.S. acknowledges partial support by the PNRR MUR projects PE0000013-Future Artificial Intelligence Research-FAIR \& CN0000013 CN HPC - National Centre for HPC, Big Data and Quantum Computing, Gruppo Nazionale Calcolo Scientifico-Istituto Nazionale di Alta Matematica (GNCS-INdAM). I.K. acknowledges partially support by the US National Science Foundation.

\bibliographystyle{unsrt}
\bibliography{Biblography}
\clearpage
\appendix

\appendix
\renewcommand{\theequation}{A.\arabic{equation}}
\renewcommand{\thefigure}{A.\arabic{figure}}
\renewcommand{\thetable}{A.\arabic{table}}
\setcounter{equation}{0}
\setcounter{figure}{0}
\setcounter{table}{0}
\section{Social Force Model}
\label{app:sfm}
Here, we provide the detailed expressions of the virtual forces included in the SFM Eq.~\eqref{eq:SFMshort} and the set of parameters selected for the microscopic simulations. 

As a reminder, we consider the SFM including a goal-directed, a social, and a repulsive force term (the random force is neglected), reading:
\begin{equation}
m_i \frac{d {\bm{\dot{x}}}_i}{dt}= \underbrace{-\frac{1}{\tau_i}({v_0}_i\bm{e}_i-\bm{v}_i)}_{\text{goal-directed force}} + \underbrace{\sum_{j\neq i}\mathbf{F}_{ij}(\bm{x}_i,\bm{x}_j)}_{\text{social force}} + \underbrace{\sum_{\text{obstacles}}\mathbf{F}_{iB}(\bm{x}_i)}_{\text{repulsive force}}, 
\label{eq:SFMapp}
\end{equation}
where $m_i$ is the mass of the $i$-th pedestrian, ${v_0}_i$ is the maximum walking speed, $\bm{e}_i$ represents the desired direction (serving as a simplistic auxiliary/behavioral variable), and $\tau_i$ is a characteristic relaxation time. 

The social force term in Eq.~\eqref{eq:SFMapp} introduces a short-range repulsive interaction between pedestrians to maintain social distances and avoid physical contact \cite{helbing1995social}. Its expression between two pedestrians $i$ and $j$ is given by:
\begin{equation}
        \mathbf{F}_{ij}(\bm{x}_i,\bm{x}_j) = 
    A_i \exp\left( \frac{r_{ij} - d_{ij}}{B_i} \right) \bm{n}_{ij} 
    + k \, g(r_{ij} - d_{ij}) \, \bm{n}_{ij} 
    + \kappa \, g(r_{ij} - d_{ij}) \, \Delta v_{ij}^{\mathrm{tang}} \, \bm{t}_{ij},
    \label{eq:Fij}
\end{equation}
where the first term represents a repulsive force modelling the physiological desire for personal space, the second term represents a physical contact (body) force that is only active when pedestrians are in contact, and the third term represents a sliding friction force that occurs during physical contact. The repulsive force is modelled in Eq.~\eqref{eq:Fij} by an exponential term, where $A_i\in \mathbb{R}^+$ is the strength of the repulsive interaction, $d_{ij} =\|\bm{x}_i - \bm{x}_j\|_2$ is the euclidean distance between pedestrians, $r_{ij}=r_i+r_j\in \mathbb{R}$ is the sum of their radii (minimal comfortable distance \cite{Helbing2000}), $B_i\in \mathbb{R}^+$ is the parameter controlling the spatial decay rate of the repulsive force, and $\bm{n}_{ij}=\big(n^1_{ij},n^2_{ij}\big)=(\bm{x}_i - \bm{x}_j)/d_{ij}$ is the unit vector pointing from pedestrian $j$ to pedestrian $i$. The physical contact force is modelled in Eq.~\eqref{eq:Fij} by a linear term, where $k>0$ controls the strength of the contact (body) force, and $g(x)=\text{max}(0,x)$ is a ramp function ensuring that physical forces are only active when pedestrians overlap, i.e., when $(d_{ij}<r_{ij})$. The friction force term in Eq.~\eqref{eq:Fij} is modelled by the same ramp function $g(x)$, scaled by a friction coefficient $\kappa>0$ and difference in the tangential velocities $\Delta v_{ij}^{\mathrm{tang}}=\|(\bm{v}_i-\bm{v}_j) -  \left((\bm{v}_i-\bm{v}_j) \cdot \bm{n}_{ij}\right) \bm{n}_{ij} \|_2$; this component is projected onto the unit tangential vector $\mathbf{t}_{ij}=\big(-n^2_{ij},n^1_{ij}\big)$. 

The repulsive force in Eq.~\eqref{eq:SFMapp}, accounting for the repulsion that a pedestrian might experience due to static objects $B$, follows a similar form to Eq.~\eqref{eq:Fij}, reading:
\begin{equation}
    \mathbf{F}_{iB}(\bm{x}_i) = C_i \exp\left(-\frac{d_{iB}}{D_i}\right) \bm{n}_{iB} 
+ k \, g(r_i - d_{iB}) \, \bm{n}_{iB},
\end{equation}
where $C_i\in \mathbb{R}^+$ is the strength of the repulsive interaction against a wall or an obstacle, $D_i$ is its corresponding decay parameter, and $d_{iB}$ is the euclidean distance between the pedestrian and the closest point of the object $B$; $\bm{n}_{iB}$ is the respective the unit vector. Unlike pedestrian–pedestrian interactions, no tangential (sliding) friction is typically included for static obstacles, as they do not exert dynamic resistance.
\begin{table}[htbp]
\centering
\caption{Social Force Model simulation parameters.}
\label{tab:SFM_parameters}
\begin{tabular}{l l l l}
\toprule
\textbf{Simulation Parameters} & \textbf{Value} & \textbf{Simulation Parameters} & \textbf{Value} \\
\midrule
Number of Pedestrians ($N$) & 100 & Repulsion coefficient ($k$) & $1.2 \times 10^5\,\frac{\text{kg}\cdot\text{m}}{\text{s}^2}$ \\
Corridor length ($L_x$) & $48\,\text{m}$ & Sliding friction coefficient ($\kappa$) & $2.4 \times 10^5\,\frac{\text{kg}\cdot\text{m}}{\text{s}^2}$ \\
Corridor width ($L_y$) & $12\,\text{m}$ & Interaction strength ($A_i$) & $2 \times 10^3\,\text{N}$ \\
Obstacle height & $3.6\,\text{m}$ & Interaction range ($B_i$) & $0.08\,\text{m}$ \\
Obstacle width & $3.6\,\text{m}$ & Interaction strength ($C_i$) & $2 \times 10^3\,\text{N}$ \\
Obstacle $x$-coord. center & $25.8\,\text{m}$ & Interaction range ($D_i$) & $0.08\,\text{m}$ \\
Obstacle $y$-coord. center & $1.8\,\text{m}$ & Pedestrian mass ($m_i$) & $80\,\text{kg}$ \\
Initial velocity in $x$-direction ($v_x$) & $0\,\text{m/s}$ & Relaxation time ($\tau$) & $0.5\,\text{s}$ \\
Initial velocity in $y$-direction ($v_y$) & $0\,\text{m/s}$ & Pedestrian radius ($r_i$) & $0.2\,\text{m}$ \\
Simulation time & $250\,\text{s}$ & Time Step ($\delta_t$) & $2.5 \times 10^{-2}\,\text{s}$ \\
\bottomrule
\end{tabular}
\end{table}
For our simulations, the model parameters were calibrated based on established values from previous studies \cite{helbing1995social,Helbing2000}. These works describe pedestrians with a standard average mass of $80\ kg$, for which a relaxation time $\tau = 0.5\,\text{s}$ provides adequate behavior under normal conditions. The complete set of parameters governing repulsive potentials, including pedestrian-pedestrian and pedestrian-obstacle interactions, along with the simulation environment specifications, are detailed in Table \ref{tab:SFM_parameters}. This parameter selection, which is uniform for all pedestrians, ensures consistency with validated bottleneck evacuation scenarios. Although individual variations exist in reality, using a common set of parameters is a standard approach in the literature \cite{Helbing2000} to maintain a manageable calibration process while enabling the reliable simulation of diverse crowd dynamics, including the present case study.


\renewcommand{\theequation}{B.\arabic{equation}}
\renewcommand{\thefigure}{B.\arabic{figure}}
\renewcommand{\thetable}{B.\arabic{table}}
\renewcommand{\thealgorithm}{B.\arabic{algorithm}}
\setcounter{equation}{0}
\setcounter{figure}{0}
\setcounter{table}{0}
\setcounter{algorithm}{0}
\section{Density fields extraction via Kernel Density Estimation}

\label{app:kde}
Here, we provide additional details for the first step of the proposed framework, regarding the extraction of continuous density profiles from discrete pedestrians' positions. As already discussed in Section~\ref{sbsb:pos2den}, we used KDE for this task, following the general form in Eq.~\eqref{eq:pos2den}; here, we elaborate on the kernel selection, bandwidth tuning, and adjustments for boundaries and obstacles.

The choice of kernel function directly influences the smoothness and continuity of the density estimate. Common options include the uniform, triangular, Epanechnikov, and Gaussian kernels \cite{Silverman1986,WandJones1995}. The Epanechnikov kernel is theoretically optimal for minimizing mean integrated squared error \cite{Epanechnikov1969}, but its bounded support can limit flexibility in practice. In contrast, the Gaussian kernel is widely preferred due to its smoothness, infinite support, and analytical tractability. Its multivariate form is given by:

\begin{equation}
K_h(\mathbf{x}-\mathbf{x}_i(t)) = \frac{1}{2\pi \det(\mathbf{H})^{1/2}} \exp\left(-\tfrac{1}{2} (\mathbf{x}-\mathbf{x}_i(t))^\top \mathbf{H}^{-1} (\mathbf{x}-\bm{x}_i(t))\right),
\label{eq:gaussian_kernel_full}
\end{equation}
where $\mathbf{H}$ is the symmetric positive-definite bandwidth matrix, $\mathbf{x} \in \mathbb{R}^2$ is the spatial domain and $\bm{x}_i(t) \in \mathbb{R}^{2}$ is the position of the $i$-th pedestrian.

The choice of the bandwidth matrix, which in this case takes the form $\mathbf{H}=\diag(h_x,h_y)$, plays a central role in the KDE estimate. Small values of $h_x/h_y$ result in undersmoothing and high variance, while large values can oversmooth the density and obscure relevant features. Silverman's rule of thumb \cite{Silverman1986}, originally formulated for univariate distributions, extends naturally to multivariate settings, indicating an optimal bandwidth:
\begin{equation}
h_i = \sigma_i \left(\frac{4}{N(d + 2)}\right)^{1/(d + 4)},
\end{equation}
where $N$ is the number of data points, $d$ is the number of spatial dimensions, in this case $d=2$, and $\sigma_i$ is the standard deviation of the pedestrian positions along $i=x,y$ spatial directions. However, in our case, Silverman's rule resulted in wide bandwidths, which did not allow us to consider sufficient detail in the density profiles. We thus tuned $h_x/h_y$ to the values reported in Section~\ref{sbsb:pos2den}.

The standard symmetric kernels are known to introduce bias near domain boundaries, when handling cases where \textit{periodic boundary conditions} are imposed, such as in the $x$-direction of pedestrian flow in our configuration. These artifacts arise because the kernel's support extends beyond the domain, failing to reflect the periodic structure of the system. To mitigate these artifacts, we define an extended support for the KDE in the $x$-direction. Let the physical domain boundaries be $x = a$ and $x = b$, and introduce a small displacement $\Delta x = \frac{b - a}{n}$, where $n$ defines the extension of the domain along $x$-direction; here, we tuned $n=5$. We then define the extended domain as:
\begin{equation}
    \Omega^{\text{ext}} = [a - \Delta x, b + \Delta x]\times[c,d],
    \label{eq:omega_ext}
\end{equation}
while the support in the $y$-direction remains unbounded. To enforce continuity and reduce boundary-induced artifacts in the $x$-direction, we further augment the observed positions with ``ghost" particles near the domain edges. Let $\epsilon = \Delta x$ define a threshold for proximity to the boundary. For each position $\bm{x}_i(t) = (x_i(t), y_i(t))$ such that $x_i(t) \leq a + \epsilon$, we introduce a ``mirrored" pedestrians at
\begin{equation*}
\bm{x}_i^{\text{aug1}}(t) = (x_i(t) - (b - a), y_i(t)).
\end{equation*}
Similarly, for every $\bm{x}_i(t)$ such that $x_i(t) \geq b - \epsilon$, we introduce a ``mirrored" pedestrians at
\begin{equation*}
\bm{x}_i^{\text{aug2}}(t) = (x_i(t) + (b - a), y_i(t)).
\end{equation*}
The resulting augmented set of pedestrian positions used for KDE is then given by
\begin{equation}
\{\bm{x}_i'(t)\} = \{\bm{x}_i(t)\} \cup \{\bm{x}_i^{\text{aug1}}(t)\} \cup \{\bm{x}_i^{\text{aug2}}(t)\}. 
\label{eq:aug_pos}
\end{equation}

A pseudo-code for the estimation of the density in $\Omega \subset \mathbb{R}^2$ is given in Algorithm \ref{alg:KDE}.
\begin{algorithm}
\caption{Kernel Density Estimation (KDE) for extracting density fields from pedestrian positions.}
\begin{algorithmic}[1]
\Require Number of pedestrians $N$, computational domain $\Omega = [a, b] \times [c, d]$, spatial discretization $n_x$/$n_y$, obstacle subdomain $\Gamma=[a_o,b_o] \times [c_o,d_o]\subset \Omega$ and pedestrian positions $\{ \mathbf{x}_i(t)\}_{i=1}^{N} \in \Omega \backslash \Gamma$ 
\State Set bandwidth $\mathbf{H}=\diag(h_x,h_y)$.
\State Set augmentation on $x$-direction $\Delta x= (b-a)/n$
\State Define extended support $\Omega^{\text{ext}}$ to handle periodic boundary conditions \Comment{Use Eq.~\eqref{eq:omega_ext}}
\State Discretize extended domain $\Omega^{\text{ext}}$ in $(n_x+2n_x/n) \times n_y$ cells
\State Introduce ``mirrored'' pedestrians to form $\{\mathbf{x}_i'(t)\}$ \Comment{Use Eq.~\eqref{eq:aug_pos}}
\State Initialize zero density field $\rho((x_i,y_j),t)$ across the extended domain $\Omega^{\text{ext}}$
\For{each pedestrian $p=1,2,\ldots,N$} \Comment{Account for ``mirrored'' pedestrians}
    \State Compute kernel contribution $K_h(\mathbf{x}-\mathbf{x}_p(t))$ over the extended support $\Omega^{\text{ext}}$ \Comment{Use Eq.~\eqref{eq:gaussian_kernel_full}}
    \State Add contribution to density field $\rho((x_i,y_j),t)$
\EndFor
\State Normalize $\rho((x_i,y_j),t)$ with total number of pedestrians, $N$ and ``mirrored'' ones
\State Retrieve density $\rho((x_i,y_j),t)$ in computational domain $\Omega$
\State Employ binary mask function to ensure zero density in obstacle subdomain $\Gamma$, by setting \Comment{See Eq.~\eqref{eq:mask}}
\Statex $\rho((x_i,y_i),t) \gets \mathcal{M}(x_i,y_i)\cdot \rho((x_i,y_i),t)$ 
\State Compute total estimated density over the entire domain $S(t) = \sum_{i,j=1}^{n_x,n_y} \rho((x_i,y_j),t) \delta x \delta y$ 
\State Normalize density field $\rho((x_i,y_j),t) \gets \rho((x_i,y_j),t)/S(t)$
\Statex \Return
Normalized density $\rho((x_i,y_j),t)$ across computational domain $\Omega$
\end{algorithmic}
\label{alg:KDE}
\end{algorithm}
In the presence of obstacles, formally described as a subdomain $\Gamma \subset \Omega$, the standard KDE must be adjusted to ensure that the density vanishes within and along the boundaries of these regions. This is typically achieved by introducing a binary mask function $\mathcal{M}(x, y)$:
\begin{align}
M(x,y)= 
\begin{cases}
0, & \text{if } (x,y) \in \Gamma \subset \Omega, \\
1, & \text{otherwise}.
\end{cases}
\label{eq:mask}
\end{align}
which is then applied to the initially estimated densities to enforce zero density in $\Gamma$, implying:
\begin{equation}
\rho((x_i,y_j),t)= \mathcal{M}(x_i,y_j) \, \rho((x_i,y_j),t), \quad \forall (x_i,y_j) \in \Omega. \label{AppxC:Mask}
\end{equation}


\renewcommand{\theequation}{C.\arabic{equation}}
\renewcommand{\thefigure}{C.\arabic{figure}}
\renewcommand{\thetable}{C.\arabic{table}}
\setcounter{equation}{0}
\setcounter{figure}{0}
\setcounter{table}{0}
\section{Initial Conditions of Microscopic Distributions}
\label{app:intial_conditions}
Here, we describe the initial conditions considered in the SFM to construct the datasets, upon which the proposed framework is trained and tested. We considered zero initial velocities for all pedestrians, while we used the following distributions for initializing the position $\bm{x}_i=(x_i,y_i)\in \Omega$ of the $i$-th pedestrian. Positions are sampled from:
\begin{itemize}
\item Uniform distributions over the domain $[x_{\text{min}},x_{\text{max}}] \times [y_{\text{min}},y_{\text{max}}] \subset \Omega$: 
\begin{equation}
p(x_i,y_i) = \mathcal{U}\left([x_{\text{min}},x_{\text{max}}] \times [y_{\text{min}},y_{\text{max}}] \right).
\label{eq:uniform}
\end{equation}

\item Two dimensional normal distributions with means $(\mu_x, \mu_y)$ and standard deviations $(\sigma_x, \sigma_y)$:
\begin{equation}
p(x_i,y_i) =\frac{1}{2 \pi \sigma_x \sigma_y} \exp \left( -\frac{(x_i - \mu_x)^2}{2 \sigma_x^2} - \frac{(y_i - \mu_y)^2}{2 \sigma_y^2} \right),
\label{eq:Gaussian}
\end{equation}
\item ``Double bell-shaped'' independent joint distributions, which are comprised by normal distributions along the $y$–axis with mean $\mu_y$ and standard deviation $\sigma_y$, and a mixture of normal distributions along the $x$–axis with means $(\mu_{x_1},\mu_{x_2})$ and standard deviations $(\sigma_{x_1},\sigma_{x_2})$:
\begin{equation}
\begin{aligned}
p(x_i,y_i) &= \frac{1}{\sqrt{2\pi}\,\sigma_y}\,
\exp\!\left(-\frac{(y-\mu_y)^2}{2\sigma_y^2}\right) \\
&\quad \times \frac{1}{2}\left[
\frac{1}{\sqrt{2\pi}\,\sigma_{x_1}}
\exp\!\left(-\frac{(x-\mu_{x_1})^2}{2\sigma_{x_1}^2}\right)
+
\frac{1}{\sqrt{2\pi}\,\sigma_{x_2}}
\exp\!\left(-\frac{(x-\mu_{x_2})^2}{2\sigma_{x_2}^2}\right)
\right].
\end{aligned}
\label{eq:double_gaussian}
\end{equation}
\item Cosine distributions centered at $(c_x, c_y)$ with scale parameters $(s_x, s_y)$:
\begin{equation}
p(x_i,y_i) = \frac{1}{\pi s_x s_y}
\cos\left( \frac{x_i - c_x}{s_x} \right)
\cos\left( \frac{y_i - c_y}{s_y} \right).
 \label{eq:cosine}
\end{equation}
\end{itemize}
The training and testing data sets were generated from $C = 20$ initial conditions, with 5 samples drawn from each of the four distributions described above. To ensure plausible and varied crowd configurations, the parameters for these distributions were selected heuristically to provide a broad spatial coverage while mitigating boundary artifacts. Furthermore, during the sampling process, pedestrian positions were resampled if they were within a critical distance equal to $2r_i$ of an existing pedestrian or a wall/obstacle to prevent initial overlaps and collisions. The specific parameter configurations used for the training and testing data sets are detailed Table~\ref{tab:Microscopic_distributions} and Table~\ref{tab:Microscopic_distributions_test}, respectively.
\begin{table}[htbp]
\centering
\caption{Microscopic initial distribution configurations and parameters for constructing the training dataset.}
\label{tab:Microscopic_distributions}
\begin{tabular}{lll}
\toprule
\textbf{Case c} & \textbf{Distribution Type} & \textbf{Parameters} \\
\midrule
1 & Uniform, Eq.~\eqref{eq:uniform} & $x_{\text{min}}= 2,x_{\text{max}}=15,y_{\text{min}}= 3,y_{\text{max}}=9$ \\
2 & Uniform, Eq.~\eqref{eq:uniform} & $x_{\text{min}}= 33,x_{\text{max}}=46,y_{\text{min}}= 3,y_{\text{max}}=9$\\
3 & Uniform, Eq.~\eqref{eq:uniform} & $x_{\text{min}}= 2,x_{\text{max}}=20,y_{\text{min}}= 4,y_{\text{max}}=8$ \\
4 & Uniform, Eq.~\eqref{eq:uniform} & $x_{\text{min}}= 28,x_{\text{max}}=46,y_{\text{max}}= 4,y_{\text{max}}=8$ \\
5 & Uniform, Eq.~\eqref{eq:uniform} & $x_{\text{min}}= 3,x_{\text{max}}=17,y_{\text{max}}= 4,y_{\text{max}}=8$ \\
6 & Gaussian, Eq.~\eqref{eq:Gaussian} & $\mu_x = 10,\; \mu_y = 6, \sigma_x = 1.5,\; \sigma_y = 1.5$ \\
7 & Gaussian, Eq.~\eqref{eq:Gaussian} & $\mu_x = 10,\; \mu_y = 6,\sigma_x = 2,\; \sigma_y = 1.5$ \\
8 & Gaussian, Eq.~\eqref{eq:Gaussian} & $\mu_x = 38,\; \mu_y = 6,\sigma_x = 2,\; \sigma_y = 1.5$ \\
9 & Gaussian, Eq.~\eqref{eq:Gaussian} & $\mu_x = 38,\; \mu_y = 6, \sigma_x = 1.5,\; \sigma_y = 1.5$ \\
10 & Gaussian, Eq.~\eqref{eq:Gaussian} & $\mu_x = 37,\; \mu_y = 5, \sigma_x = 1.5,\; \sigma_y = 1.5$ \\
11 & Double Gaussian, Eq.~\eqref{eq:double_gaussian} & $\mu_{x_1} = 8,\; \mu_y = 7, \sigma_{x_1} = 1.2,\; \sigma_y = 1.5,\; \mu_{x_2} = 12, \; \sigma_{x_2} = 0.8$ \\
12 & Double Gaussian, Eq.~\eqref{eq:double_gaussian} & $\mu_{x_1} = 10,\; \mu_y = 6, \sigma_{x_1} = 1.2,\; \sigma_y = 1.5,\; \mu_{x_2} = 14, \; \sigma_{x_2} = 1$ \\
13 & Double Gaussian, Eq.~\eqref{eq:double_gaussian} & $\mu_{x_1} = 28,\; \mu_y = 6, \sigma_{x_1} = 1.3,\; \sigma_y = 1.5,\; \mu_{x_2} = 32, \; \sigma_{x_2} = 1.1$ \\
14 & Double Gaussian, Eq.~\eqref{eq:double_gaussian} & $\mu_{x_1} = 36,\; \mu_y = 6, \sigma_{x_1} = 1.4,\; \sigma_y = 1.5,\; \mu_{x_2} = 40, \; \sigma_{x_2} = 1.2$ \\
15 & Double Gaussian, Eq.~\eqref{eq:double_gaussian} & $\mu_{x_1} = 9,\; \mu_y = 6, \sigma_{x_1} = 0.8,\; \sigma_y = 1.5,\; \mu_{x_2} = 11, \; \sigma_{x_2} = 1$ \\
16 & Cosine, Eq.~\eqref{eq:cosine} & $c_x = 10,\; c_y = 6, s_x = 10,\; s_y = 3$ \\
17 & Cosine, Eq.~\eqref{eq:cosine} & $c_x = 20,\; c_y = 8,s_x = 10,\; s_y = 3$ \\
18 & Cosine, Eq.~\eqref{eq:cosine} & $c_x = 30,\; c_y = 7,s_x = 10,\; s_y = 3$ \\
19 & Cosine, Eq.~\eqref{eq:cosine} & $c_x = 40,\; c_y = 6, s_x = 10,\; s_y = 3$ \\
20 & Cosine, Eq.~\eqref{eq:cosine} & $c_x = 10,\; c_y = 6, s_x = 9,\; s_y = 3$ \\
\bottomrule
\end{tabular}
\end{table}
\begin{table}[htbp]
\centering
\caption{Microscopic initial distribution configurations and parameters for constructing the testing dataset.}
\label{tab:Microscopic_distributions_test}
\begin{tabular}{lll}
\toprule
\textbf{Case c} & \textbf{Distribution Type} & \textbf{Parameters} \\
\midrule
1 & Uniform, Eq.~\eqref{eq:uniform} & $x_{\text{min}}=2,\; x_{\text{max}}=16,\; y_{\text{min}}=3,\; y_{\text{max}}=10$ \\
2 & Uniform, Eq.~\eqref{eq:uniform} & $x_{\text{min}}=32,\; x_{\text{max}}=46,\; y_{\text{min}}=3,\; y_{\text{max}}=10$ \\
3 & Uniform, Eq.~\eqref{eq:uniform} & $x_{\text{min}}=2,\; x_{\text{max}}=19,\; y_{\text{min}}=4,\; y_{\text{max}}=8$ \\
4 & Uniform, Eq.~\eqref{eq:uniform} & $x_{\text{min}}=29,\; x_{\text{max}}=46,\; y_{\text{min}}=4,\; y_{\text{max}}=8$ \\
5 & Uniform, Eq.~\eqref{eq:uniform} & $x_{\text{min}}=4,\; x_{\text{max}}=17,\; y_{\text{min}}=4,\; y_{\text{max}}=8$ \\
6 & Gaussian, Eq.~\eqref{eq:Gaussian} & $\mu_x=12,\; \mu_y=7,\; \sigma_x=2,\; \sigma_y=1.5$ \\
7 & Gaussian, Eq.~\eqref{eq:Gaussian} & $\mu_x=11,\; \mu_y=5,\; \sigma_x=1.5,\; \sigma_y=2$ \\
8 & Gaussian, Eq.~\eqref{eq:Gaussian} & $\mu_x=38,\; \mu_y=6,\; \sigma_x=2,\; \sigma_y=1.5$ \\
9 & Gaussian, Eq.~\eqref{eq:Gaussian} & $\mu_x=39,\; \mu_y=5,\; \sigma_x=1.5,\; \sigma_y=1.5$ \\
10 & Gaussian, Eq.~\eqref{eq:Gaussian} & $\mu_x=37,\; \mu_y=7,\; \sigma_x=1.5,\; \sigma_y=1.5$ \\
11 & Double Gaussian, Eq.~\eqref{eq:double_gaussian} & $\mu_{x_1}=7,\; \mu_y=7,\; \sigma_{x_1}=1.1,\; \sigma_y=1.5,\; \mu_{x_2}=11,\; \sigma_{x_2}=0.9$ \\
12 & Double Gaussian, Eq.~\eqref{eq:double_gaussian} & $\mu_{x_1}=10,\; \mu_y=7,\; \sigma_{x_1}=1.2,\; \sigma_y=1.5,\; \mu_{x_2}=14,\; \sigma_{x_2}=1$ \\
13 & Double Gaussian, Eq.~\eqref{eq:double_gaussian} & $\mu_{x_1}=29,\; \mu_y=6,\; \sigma_{x_1}=1.3,\; \sigma_y=1.5,\; \mu_{x_2}=33,\; \sigma_{x_2}=1.1$ \\
14 & Double Gaussian, Eq.~\eqref{eq:double_gaussian} & $\mu_{x_1}=37,\; \mu_y=6,\; \sigma_{x_1}=1.4,\; \sigma_y=1.5,\; \mu_{x_2}=41,\; \sigma_{x_2}=1.2$ \\
15 & Double Gaussian, Eq.~\eqref{eq:double_gaussian} & $\mu_{x_1}=10,\; \mu_y=6,\; \sigma_{x_1}=0.8,\; \sigma_y=1.5,\; \mu_{x_2}=12,\; \sigma_{x_2}=1$ \\
16 & Cosine, Eq.~\eqref{eq:cosine} & $c_x=15,\; c_y=4,\; s_x=8,\; s_y=3$ \\
17 & Cosine, Eq.~\eqref{eq:cosine} & $c_x=25,\; c_y=6,\; s_x=10,\; s_y=3$ \\
18 & Cosine, Eq.~\eqref{eq:cosine} & $c_x=35,\; c_y=5,\; s_x=10,\; s_y=3$ \\
19 & Cosine, Eq.~\eqref{eq:cosine} & $c_x=45,\; c_y=7,\; s_x=5,\; s_y=3$ \\
20 & Cosine, Eq.~\eqref{eq:cosine} & $c_x=15,\; c_y=6,\; s_x=7,\; s_y=3$ \\
\bottomrule
\end{tabular}
\end{table}


\renewcommand{\theequation}{D.\arabic{equation}}
\renewcommand{\thefigure}{D.\arabic{figure}}
\renewcommand{\thetable}{D.\arabic{table}}
\renewcommand{\thealgorithm}{D.\arabic{algorithm}}
\setcounter{equation}{0}
\setcounter{figure}{0}
\setcounter{table}{0}
\setcounter{algorithm}{0}
\section{Lag window size selection via the Bayesian and the Akaike Information Criteria}
\label{app:bic_aic}
As discussed in Section~\ref{sb:ROMs_impl}, to determine the optimal lag window size $w$ for the MVAR model, we used two information-theoretic criteria: the Bayesian Information Criterion (BIC) and the Akaike Information Criterion (AIC) \cite{lutkepohl2005new, akaike1974new, schwarz1978estimating}
. These criteria provide a balance between model fit and complexity, and are used to avoid underfitting or overfitting the temporal dynamics. Both tests work under the assumption of data stationarity. The BIC is defined as:
\begin{equation}
\text{BIC}(w) = -2\ln(\mathcal{L}(w)) + k \ln(n_t),
\label{BIC}
\end{equation}
where $\mathcal{L}$ is the maximized likelihood of the model. For an MVAR model with $d$ latent dimensions and $w$ order of the model, the total number of parameters is then $k = d^2 \cdot w$. The AIC, which provides a different trade-off between model complexity and fit, is given by:
\begin{equation}
\text{AIC}(w) = -2\ln(\mathcal{L}(w)) + 2k.
\label{AIC}
\end{equation}
While AIC tends to favor more complex models and is better suited for short-term predictive accuracy in smaller datasets, BIC imposes a stronger penalty on complexity, making it preferable for larger sample sizes.
The log-likelihood $\ln(\mathcal{L}(w))$ is computed under the assumption of Gaussian noise, and is derived from the residual covariance matrix of the MVAR model. 

Finally, to select the optimal lag window $w^*$, we evaluated the BIC and AIC for all candidate lags~$w \in \{1, 2, \dots, w_{\text{max}}\}$. The chosen lag~$w^*$ corresponds to the value of~$w$ that yielded the smallest BIC or AIC in this range, i.e.,:
\begin{equation}
     w^*_{\text{BIC}} = \underset{w \in \{1, \dots, w_{\text{max}}\}}{\arg\min} \, \text{BIC}(w), \quad 
     w^*_{\text{AIC}} = \underset{w \in \{1, \dots, w_{\text{max}}\}}{\arg\min} \, \text{AIC}(w).
 \end{equation}


\renewcommand{\theequation}{E.\arabic{equation}}
\renewcommand{\thefigure}{E.\arabic{figure}}
\renewcommand{\thetable}{E.\arabic{table}}
\renewcommand{\thealgorithm}{E.\arabic{algorithm}}
\setcounter{equation}{0}
\setcounter{figure}{0}
\setcounter{table}{0}
\setcounter{algorithm}{0}
\section{Open-loop prediction errors and closed-loop reconstructions for the unidirectional flow case} 
\label{app:UniFlow}

Here, we provide the detailed open-loop prediction errors and the closed-loop reconstructions of the proposed framework for the unidirectional flow case discussed in Section~\ref{sb:UniFlow_res}. 

The open-loop setting prediction of the proposed restrict-evolve with ROM-lift framework in Eq.~\eqref{eq:LER_basic} reads:
\begin{equation}
    \hat{x}^{(c)}(t_k) = \mathcal{L}\left(\hat{\boldsymbol{y}}^{(c)}(t_k)\right), \qquad
    \hat{\boldsymbol{y}}^{(c)}(t_k)=\Phi_{\text{ROM}}\left(
    \boldsymbol{y}^{(c)}(t_{k-1}),\ldots, \boldsymbol{y}^{(c)}(t_{k-w});\mathbf{p}\right),
    \label{eq:lift_ROMev_rest_OpenLoop}
\end{equation}
where $\boldsymbol{y}^{(c)}(t_{k-j})=\mathcal{R}(x^{(c)}(t_{k-j}))$ are restricted embeddings of observed densities for $j=1,\ldots,w$. Similarly to the close-loop predictions in Eq.~\eqref{eq:lift_ROMev_rest}, the open-loop prediction expression is used in Eq.~\eqref{eq:ForcastErr} to quantify the one-step prediction accuracy of the proposed framework. 

Table \ref{tab:errors_OL} provides a summary of the relative reconstructed errors in the ambient space, for one-step predictions, for all four ROMs, including the mean error and the 10--90\% error percentiles, aggregated \emph{over all cases and time steps} in the testing set.
\begin{table}[htbp]
\centering
\caption{Open-loop (one-step) prediction errors in the ambient--density profile--space, for the testing set using the trained MVAR and LSTM (best out of 50 training runs) models for the unidirectional flow case. The relative reconstructed $L_1$, $L_2$ and $L_\infty$ errors $e^{1,(c)}_k,\ e^{2,(c)}_k,\ e^{\infty,(c)}_k$ in Eq.~\eqref{eq:ForcastErr} are reported for the MVAR(4), MVAR(9), LSTM(4), and LSTM(9) latent dynamics models. Mean and 10--90\% error percentiles are shown \emph{over all the} $C=20$ cases and $K=\{991,996\}$ time steps (for width $w=\{9,4\}$, respectively).}
\label{tab:errors_OL}
\begin{tabular}{lccc}
\toprule
\textbf{Model} &
$L_1$ error, $e^{1,(c)}_k$ &
$L_2$ error, $e^{2,(c)}_k$ &
$L_\infty$ error, $e^{\infty,(c)}_k$ \\
\midrule
MVAR(4) &
\shortstack{$0.048\ (0.033,\,0.065)$} &
\shortstack{$0.038\ (0.027,\,0.051)$} &
\shortstack{$0.042\ (0.030,\,0.057)$} \\
MVAR(9) &
\shortstack{$0.047\ (0.032,\,0.063)$} &
\shortstack{$0.037\ (0.026,\,0.049)$} &
\shortstack{$0.041\ (0.030,\,0.055)$} \\
LSTM(4) &
\shortstack{$0.048\ (0.034,\,0.065)$} &
\shortstack{$0.038\ (0.027,\,0.051)$} &
\shortstack{$0.043\ (0.032,\,0.058)$}  \\
LSTM(9) &
\shortstack{$0.047\ (0.033,\,0.063)$} &
\shortstack{$0.037\ (0.026,\,0.049)$} &
\shortstack{$0.042\ (0.031,\,0.056)$} \\
\bottomrule
\end{tabular}

\end{table}

Figure~\ref{fig:Errors_openloop} displays the relative reconstruction errors $e^{2,(c)}_k,\ e^{\infty,(c)}_k$ and their 10--90\% percentile ranges for one-step predictions via the open-loop time-stepper (Eq.~\eqref{eq:lift_ROMev_rest_OpenLoop}). Errors are averaged over the $C=20$ cases of different, unseen initial conditions considered, across the time steps $k=1,\ldots,1000$. 
\begin{figure}[!b]
    \centering
    \begin{subfigure}[b]{0.47\textwidth}
        \includegraphics[width=\textwidth]{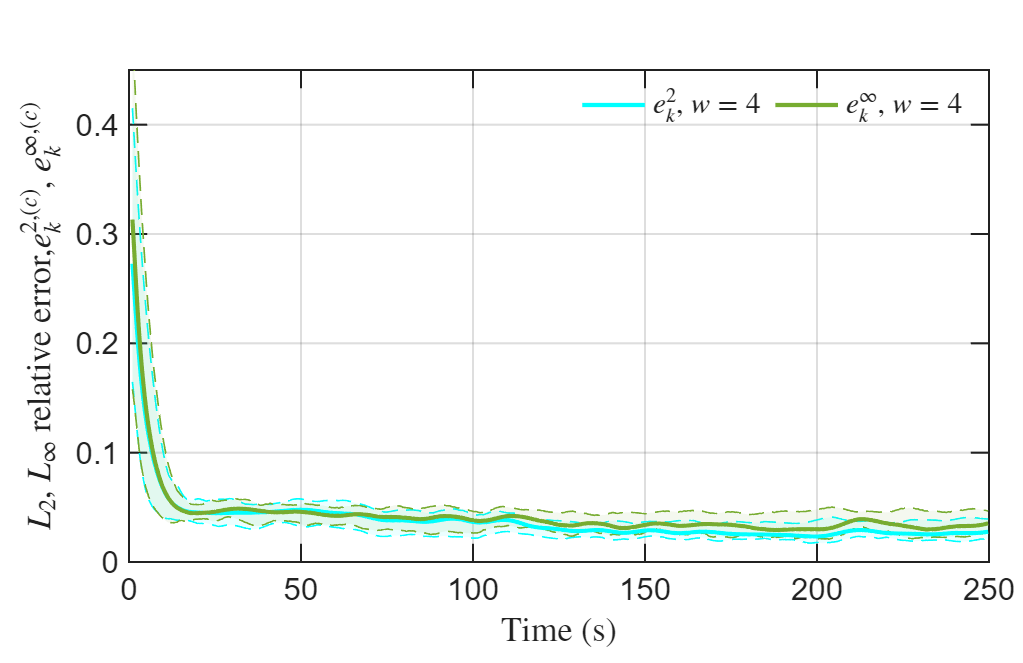}
        \caption{Relative Errors with MVAR(4)}
    \end{subfigure}
    \hfill
    \begin{subfigure}[b]{0.47\textwidth}
        \includegraphics[width=\textwidth]{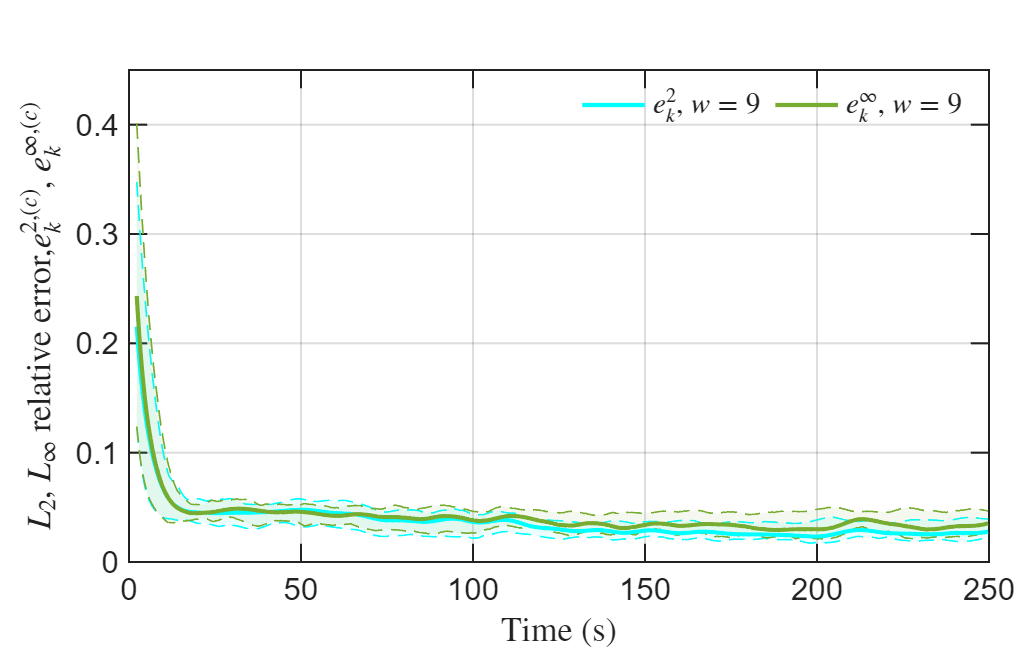}
        \caption{Relative Errors with MVAR(9)}
    \end{subfigure}
    \begin{subfigure}[b]{0.47\textwidth}
        \includegraphics[width=\textwidth]{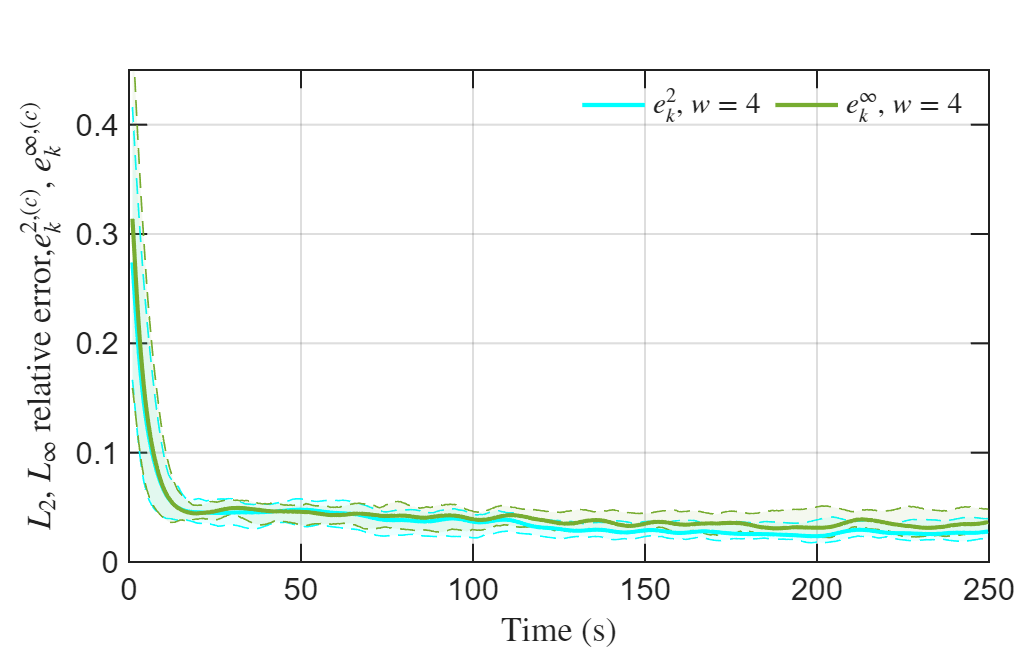}
        \caption{Relative Errors with LSTM(4)}
    \end{subfigure}
    \hfill
    \begin{subfigure}[b]{0.47\textwidth}
        \includegraphics[width=\textwidth]{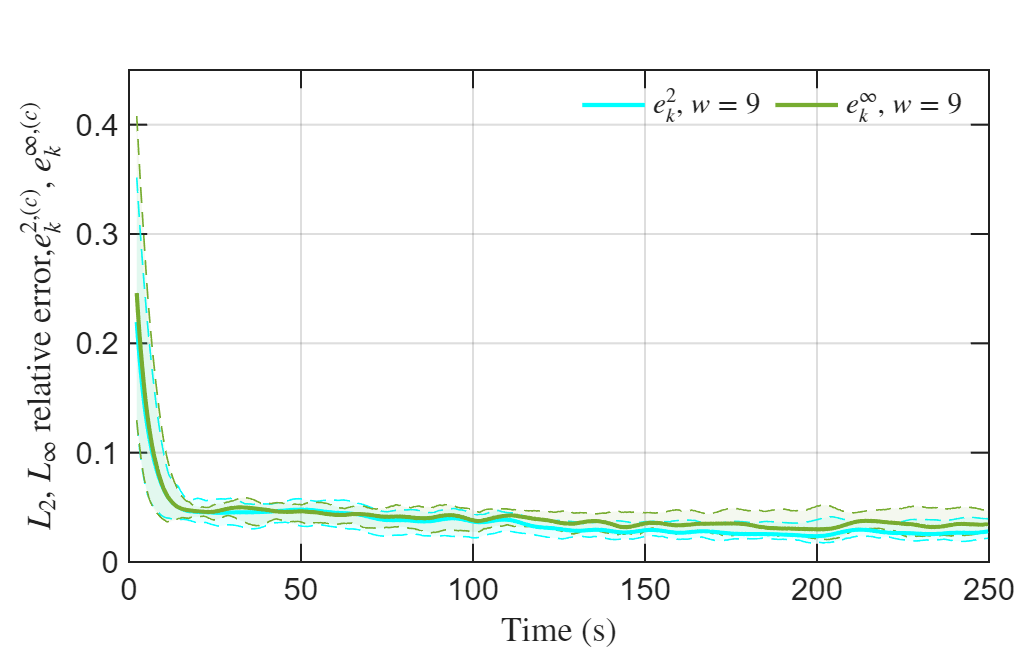}
        \caption{Relative Errors with LSTM(9)}
    \end{subfigure}
    \caption{One-step (open-loop) prediction errors, in the ambient--density profile--space, for the testing set over time, using trained MVAR and LSTM models for the unidirectional flow case. The relative reconstructed $L_2$ (cyan) and $L_\infty$ (green) errors $e^{2,(c)}_k$ and $ e^{\infty,(c)}_k$ Eq.~\eqref{eq:ForcastErr} (using the one-step prediction in Eq.~\eqref{eq:lift_ROMev_rest_OpenLoop}) are shown for the MVAR(4) panel (a), MVAR(9) panel (b), LSTM(4) panel (c), and LSTM(9) panel (d) latent dynamics models. Mean relative errors (solid) and 10--90\% error percentiles (dashed) are shown over $C=20$ cases per time step. For the LSTM models, results correspond to the model achieving the best training performance out of the 50 independent runs.
    \label{fig:Errors_openloop}}
\end{figure}

Figures~\ref{fig:MVARw4den}, \ref{fig:MVARw9den} \ref{fig:LSTMw4den} and \ref{fig:LSTMw9den} compare the ground truth normalized density fields $x^{(c)}(t_k)$ and the predicted ones $\hat{x}^{(c)}(t_k)$ for closed-loop predictions (Eq.~\eqref{eq:lift_ROMev_rest}) using MVAR(4), MVAR(9), LSTM(4), and LSTM(9) models, respectively. The test case uses a Gaussian initialization (6th row of Table~\ref{tab:Microscopic_distributions_test}) at four time steps ($k=250,500,780,930$), during which the pedestrian crowd does not pass through the corridor boundary. Panels (a,c,e,g) show the ground truth, while panels (b,d,f,h) present the corresponding closed-loop predictions. The spatial distribution of the resulting absolute errors is shown in Figure~\ref{fig:Unidirectional_flow_error} at time steps $k=250,930$. 
\begin{figure}[htbp]
    \centering
    \begin{subfigure}[b]{0.48\textwidth}
        \includegraphics[width=\textwidth]{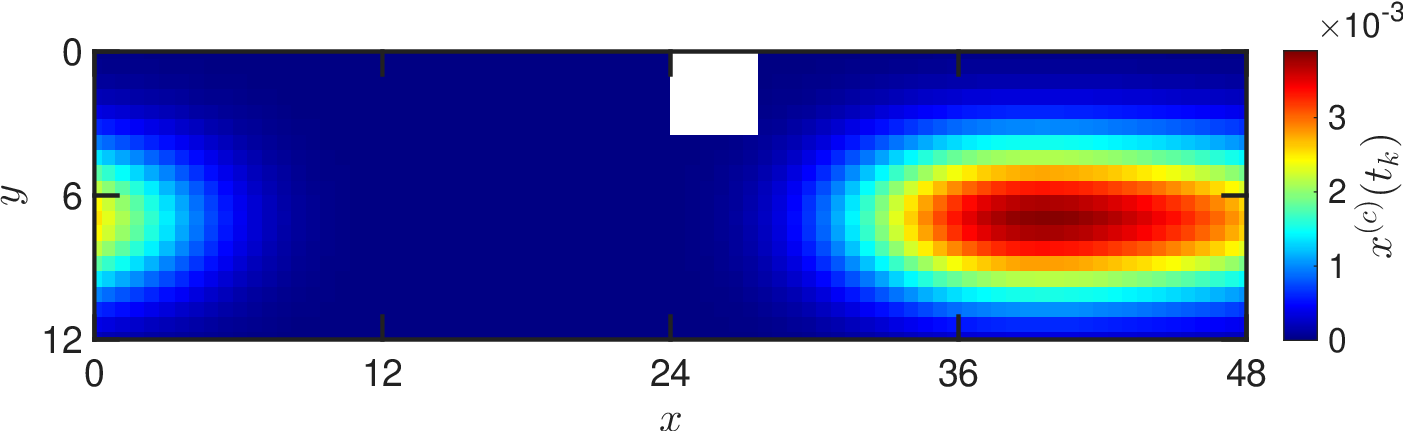}
        \caption{Ground truth density at $t_{250} = 62.5\,\text{s}$}
    \end{subfigure}
    \hfill
    \begin{subfigure}[b]{0.48\textwidth}
        \includegraphics[width=\textwidth]{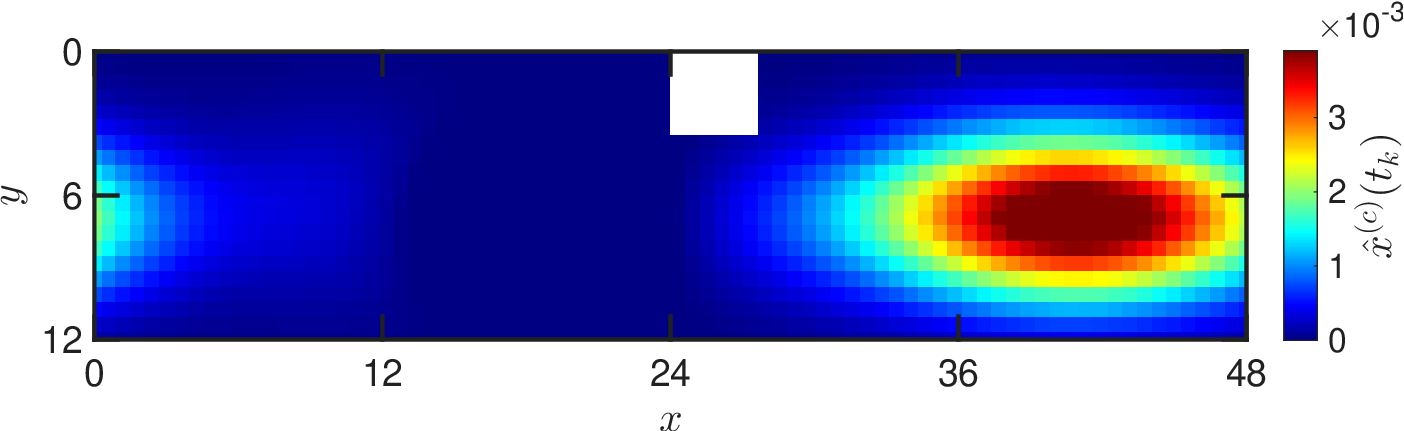}
        \caption{Predicted density at $t_{250} = 62.5\,\text{s}$}
    \end{subfigure}
    \\
    \begin{subfigure}[b]{0.48\textwidth}
        \includegraphics[width=\textwidth]{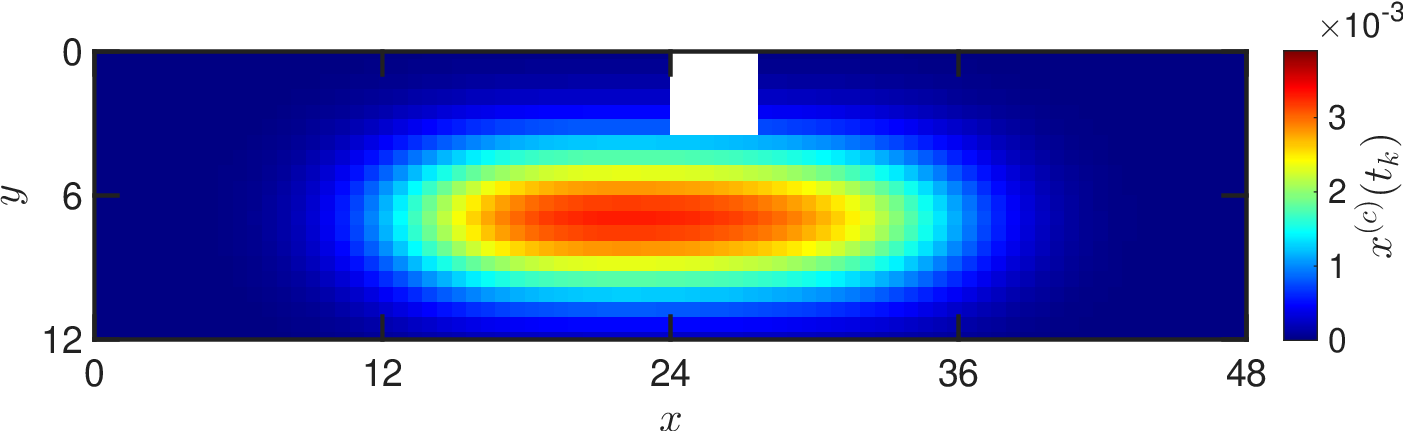}
        \caption{Ground truth density at $t_{500} = 125\,\text{s}$}
    \end{subfigure}
    \hfill
    \begin{subfigure}[b]{0.48\textwidth}
        \includegraphics[width=\textwidth]{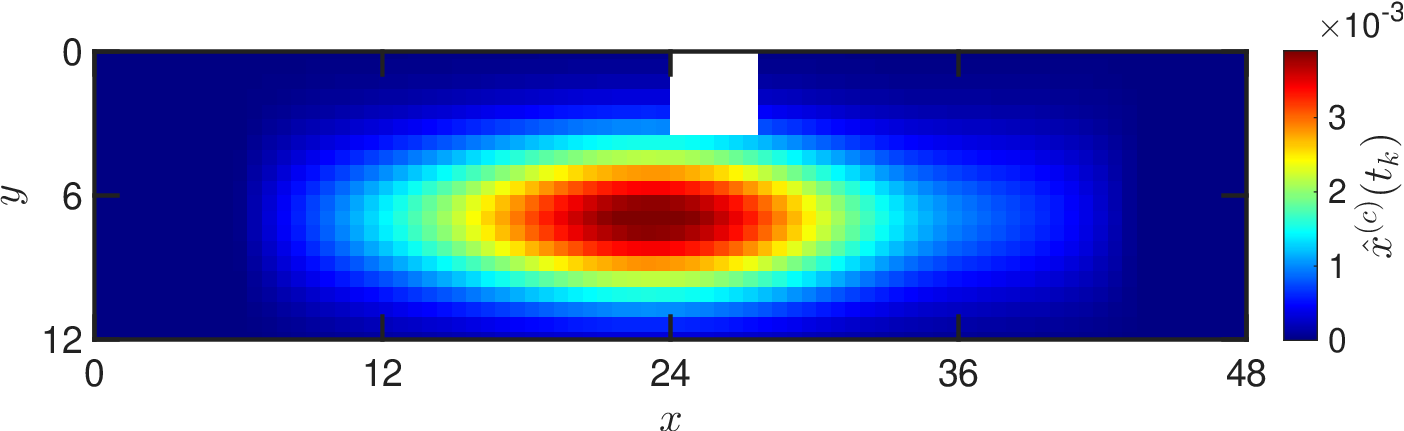}
        \caption{Predicted density at $t_{500} = 125\,\text{s}$}
    \end{subfigure}
    \\
    \begin{subfigure}[b]{0.48\textwidth}
        \includegraphics[width=\textwidth]{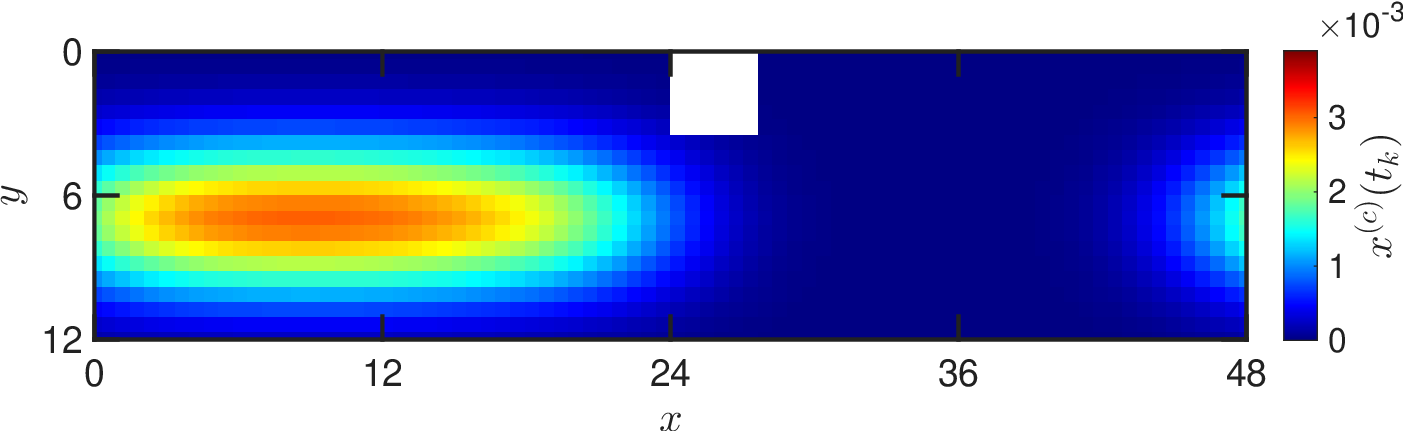}
        \caption{Ground truth density at $t_{780} = 195\,\text{s}$}
    \end{subfigure}
    \hfill
    \begin{subfigure}[b]{0.48\textwidth}
        \includegraphics[width=\textwidth]{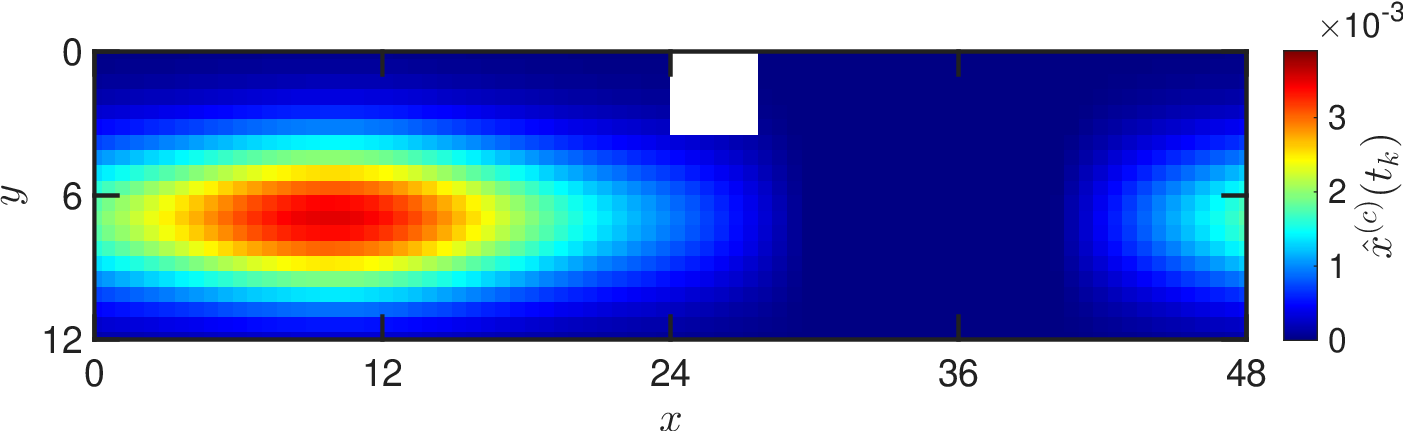}
        \caption{Predicted density at $t_{780} = 195\,\text{s}$}
    \end{subfigure}
    \\
    \begin{subfigure}[b]{0.48\textwidth}
        \includegraphics[width=\textwidth]{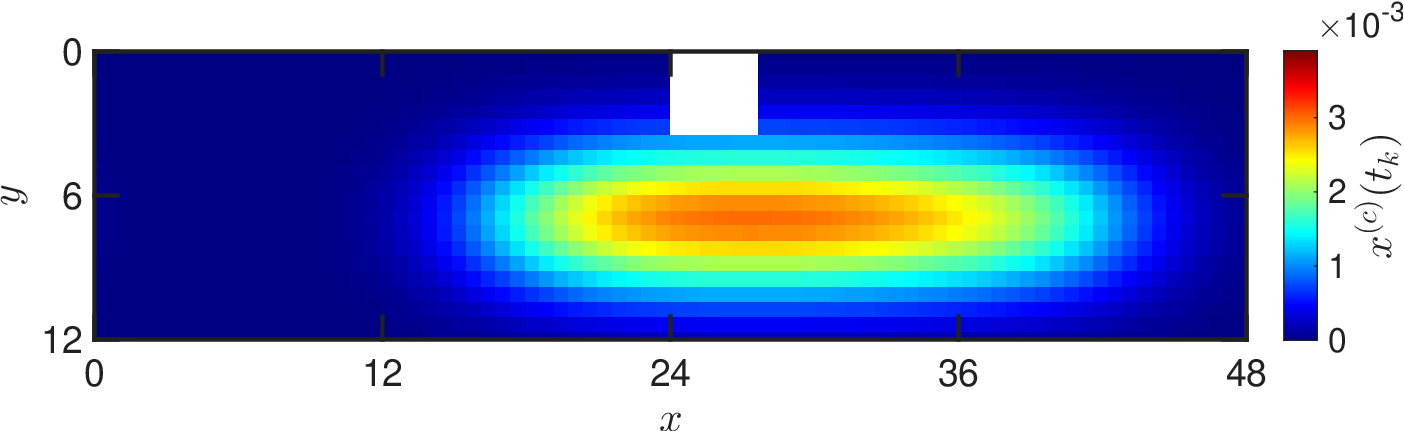}
        \caption{Ground truth density at $t_{930} =232.5\,\text{s}$}
    \end{subfigure}
    \hfill
    \begin{subfigure}[b]{0.48\textwidth}
        \includegraphics[width=\textwidth]{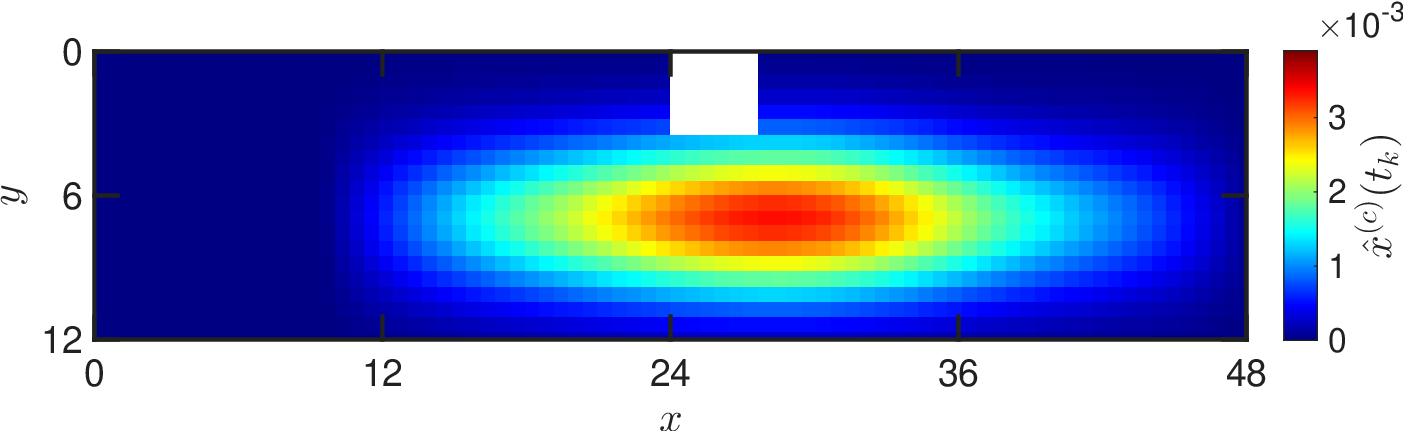}
        \caption{Predicted density at $t_{930} =232.5\,\text{s}$}
    \end{subfigure}
    \caption{Comparison of the ``ground truth" $x^{(c)}(t_k)$ and the predicted $\hat{x}^{(c)}(t_k)$ normalized densities via the closed-loop time-stepper in Eq.~\eqref{eq:lift_ROMev_rest} for the unidirectional flow case, using the MVAR(4) model for an unseen case of the testing set; initialization at the 6th row of Table~\ref{tab:Microscopic_distributions_test}. Panels (a), (c), (e), and (g) show the ``ground truth"  density distribution in $\Omega$, while panels (b), (d), (f), and (h) show the predicted one, for the time steps $t_{250} = 62.5\,\text{s}$, $t_{500} = 125\,\text{s}$, $t_{780} = 195\,\text{s}$, and $t_{930} =232.5\,\text{s}$, respectively.}
    \label{fig:MVARw4den}
\end{figure}

\begin{figure}[htbp]
    \centering
    \begin{subfigure}[b]{0.48\textwidth}
        \includegraphics[width=\textwidth]{Figure_E4a.eps}
        \caption{Ground truth density at $t_{250} = 62.5\,\text{s}$.}
    \end{subfigure}
    \hfill
    \begin{subfigure}[b]{0.48\textwidth}
        \includegraphics[width=\textwidth]{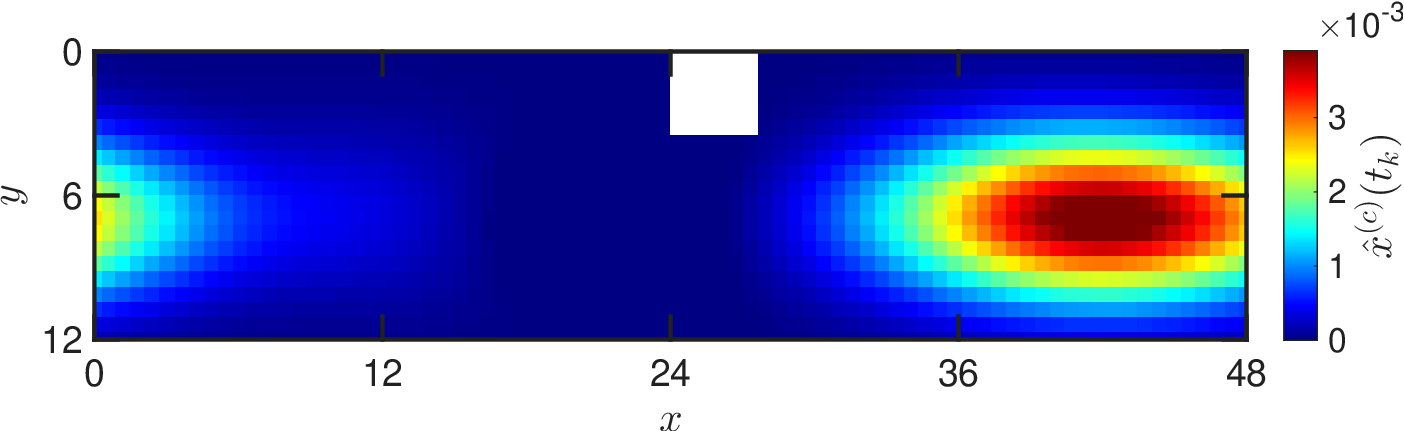}
        \caption{Predicted density at $t_{250} = 62.5\,\text{s}$}
    \end{subfigure}
    \\
    \begin{subfigure}[b]{0.48\textwidth}
        \includegraphics[width=\textwidth]{Figure_E4c.eps}
        \caption{Ground truth density at $t_{500} = 125\,\text{s}$}
    \end{subfigure}
    \hfill
    \begin{subfigure}[b]{0.48\textwidth}
        \includegraphics[width=\textwidth]{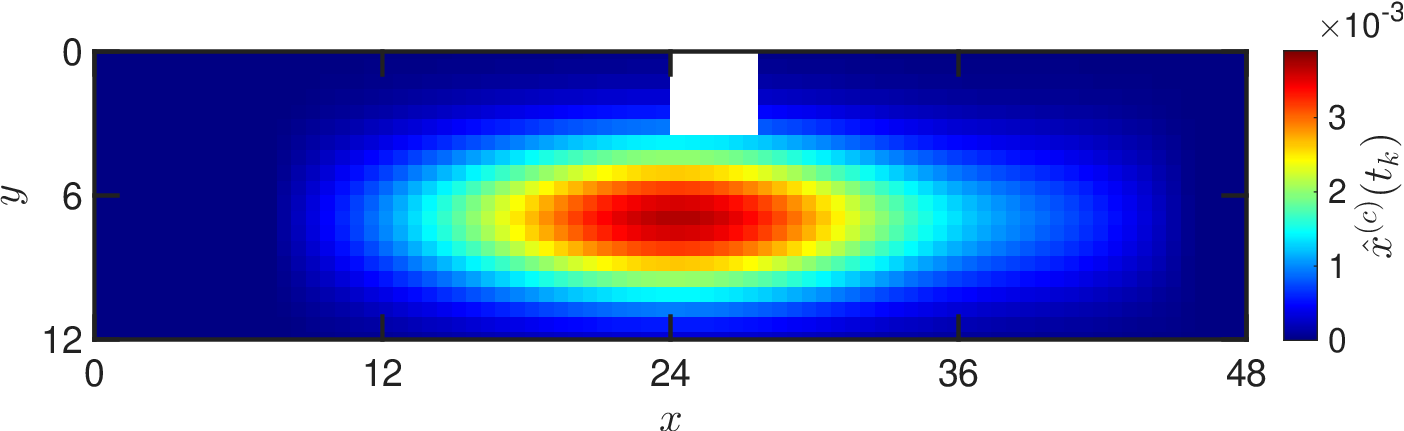}
        \caption{Predicted density at $t_{500} = 125\,\text{s}$}
    \end{subfigure}
    \\
    \begin{subfigure}[b]{0.48\textwidth}
        \includegraphics[width=\textwidth]{Figure_E4e.eps}
        \caption{Ground truth density at $t_{780} = 195\,\text{s}$}
    \end{subfigure}
    \hfill
    \begin{subfigure}[b]{0.48\textwidth}
        \includegraphics[width=\textwidth]{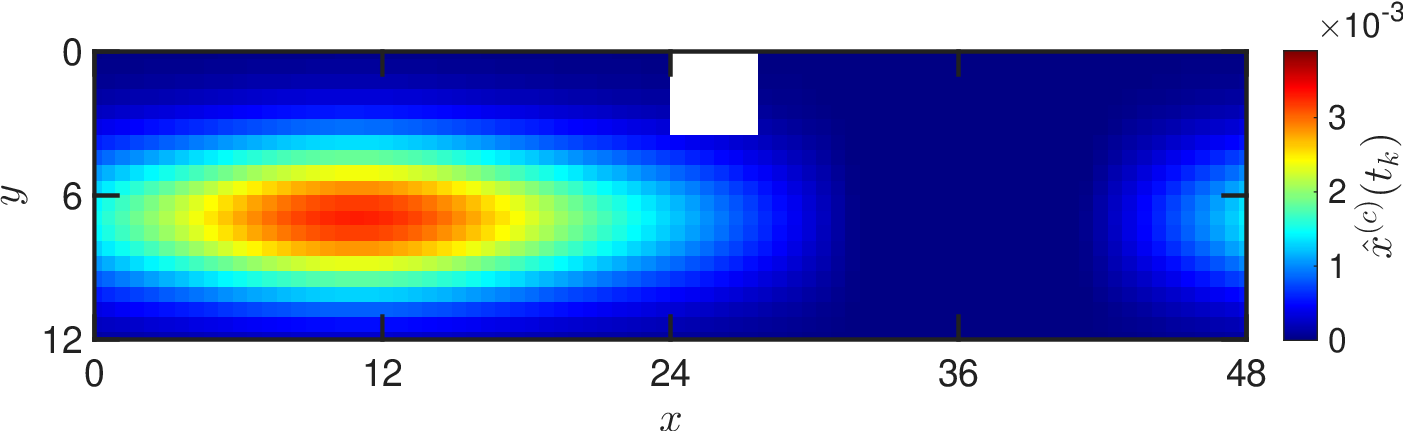}
        \caption{Predicted density at $t_{780} = 195\,\text{s}$}
    \end{subfigure}
    \\
    \begin{subfigure}[b]{0.48\textwidth}
        \includegraphics[width=\textwidth]{Figure_E4g.eps}
        \caption{Ground truth density at $t_{930} =232.5\,\text{s}$}
    \end{subfigure}
    \hfill
    \begin{subfigure}[b]{0.48\textwidth}
        \includegraphics[width=\textwidth]{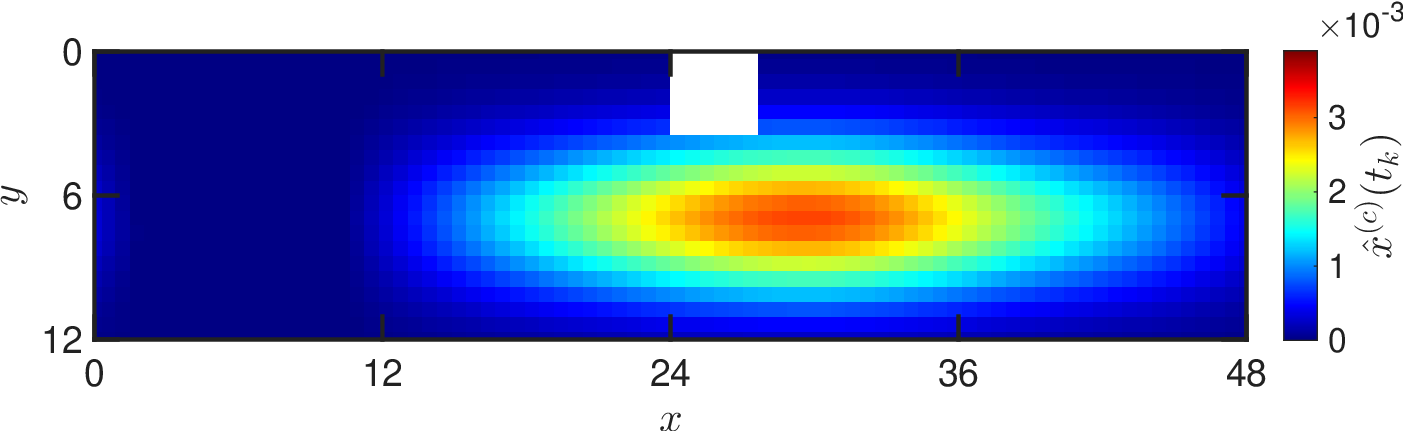}
        \caption{Predicted density at $t_{930} =232.5\,\text{s}$}
    \end{subfigure}
    \caption{Comparison of the ``ground truth" $x^{(c)}(t_k)$ and the predicted $\hat{x}^{(c)}(t_k)$ normalized densities via the closed-loop time-stepper in Eq.~\eqref{eq:lift_ROMev_rest} for the unidirectional flow case, using the MVAR(9) model for an unseen case of the testing set; initialization at the 6th row of Table~\ref{tab:Microscopic_distributions_test}. Panels (a), (c), (e), and (g) show the ``ground truth"  density distribution in $\Omega$, while panels (b), (d), (f), and (h), show the predicted one, for the time steps $t_{250} = 62.5\,\text{s}$, $t_{500} = 125\,\text{s}$, $t_{780} = 195\,\text{s}$, and $t_{930} =232.5\,\text{s}$, respectively.}
    \label{fig:MVARw9den}
\end{figure}

\begin{figure}[htbp]
    \centering
    \begin{subfigure}[b]{0.48\textwidth}
        \includegraphics[width=\textwidth]{Figure_E4a.eps}
        \caption{Ground truth density at $t_{250} = 62.5\,\text{s}$}
    \end{subfigure}
    \hfill
    \begin{subfigure}[b]{0.48\textwidth}
        \includegraphics[width=\textwidth]{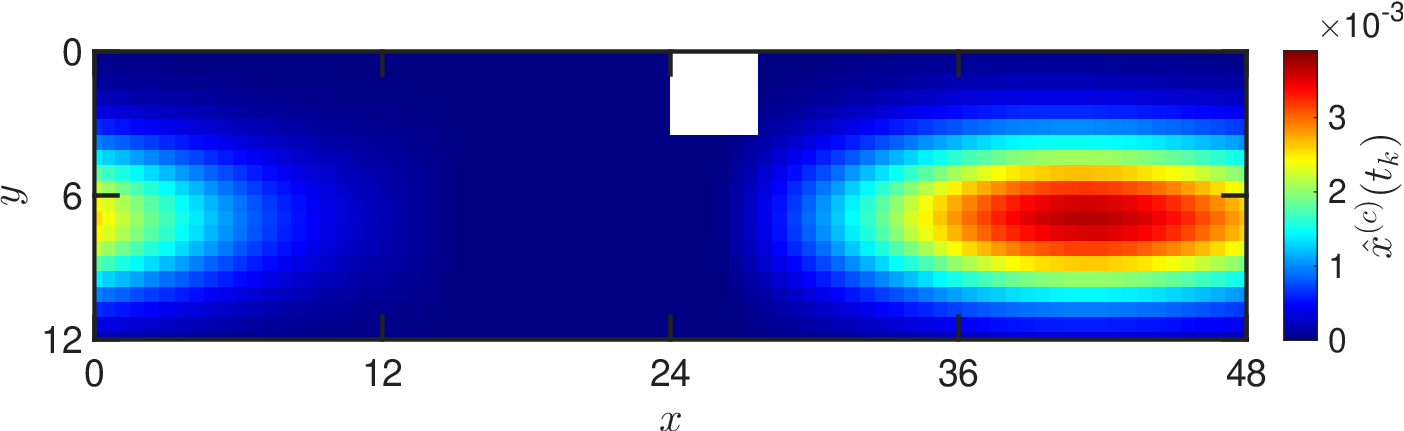}
        \caption{Predicted density at $t_{250} = 62.5\,\text{s}$}
    \end{subfigure}
    \\
    \begin{subfigure}[b]{0.48\textwidth}
        \includegraphics[width=\textwidth]{Figure_E4c.eps}
        \caption{Ground truth density at $t_{500} = 125\,\text{s}$}
    \end{subfigure}
    \hfill
    \begin{subfigure}[b]{0.48\textwidth}
        \includegraphics[width=\textwidth]{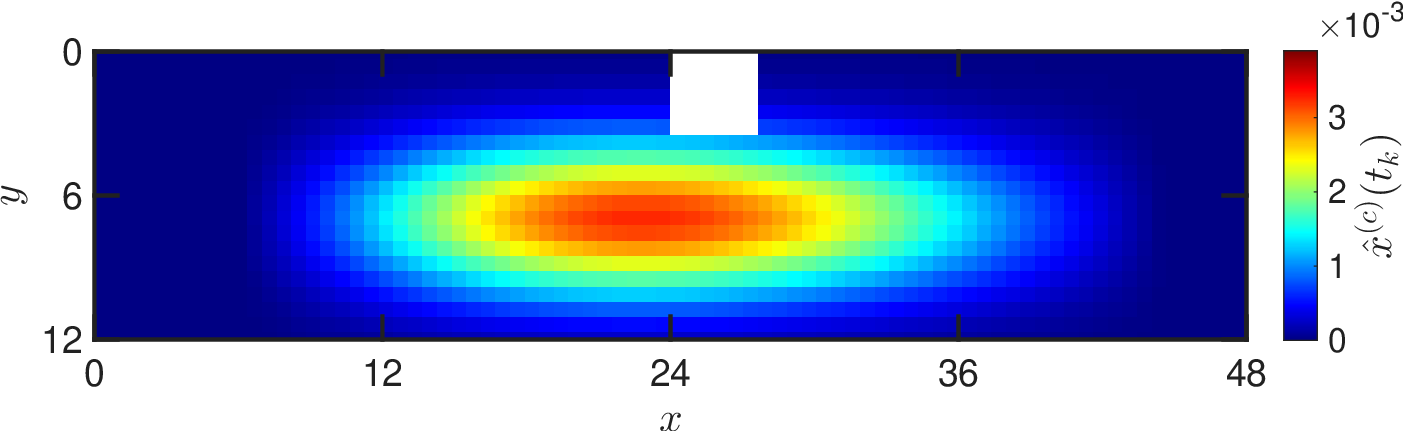}
        \caption{Predicted density at $t_{500} = 125\,\text{s}$}
    \end{subfigure}
    \\
    \begin{subfigure}[b]{0.48\textwidth}
        \includegraphics[width=\textwidth]{Figure_E4e.eps}
        \caption{Ground truth density at $t_{780} = 195\,\text{s}$}
    \end{subfigure}
    \hfill
    \begin{subfigure}[b]{0.48\textwidth}
        \includegraphics[width=\textwidth]{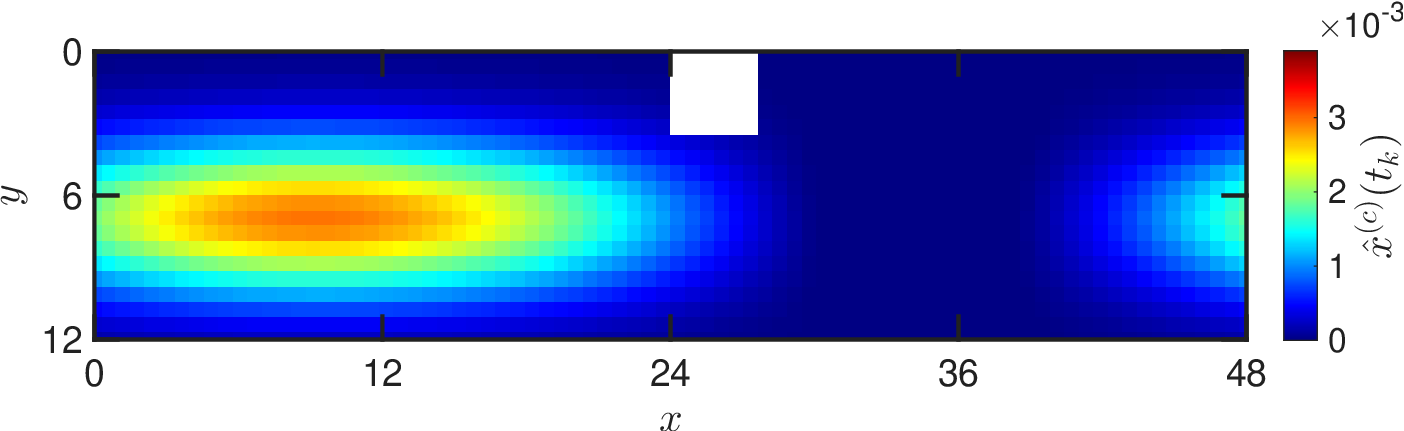}
        \caption{Predicted density at $t_{780} = 195\,\text{s}$}
    \end{subfigure}
    \\
    \begin{subfigure}[b]{0.48\textwidth}
        \includegraphics[width=\textwidth]{Figure_E4g.eps}
        \caption{Ground truth density at $t_{930} =232.5\,\text{s}$}
    \end{subfigure}
    \hfill
    \begin{subfigure}[b]{0.48\textwidth}
        \includegraphics[width=\textwidth]{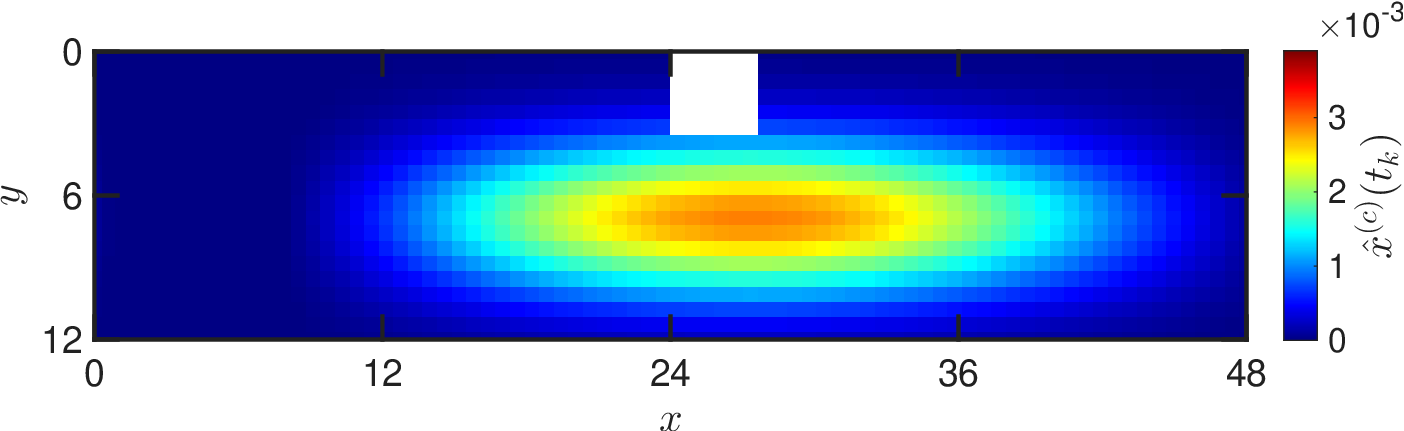}
        \caption{Predicted density at $t_{930} =232.5\,\text{s}$}
    \end{subfigure}
    \caption{Comparison of the ``ground truth" $x^{(c)}(t_k)$ and the predicted $\hat{x}^{(c)}(t_k)$ normalized densities via the closed-loop time-stepper in Eq.~\eqref{eq:lift_ROMev_rest} for the unidirectional flow case, using the LSTM(4) model for an unseen case of the testing set; initialization at the 6th row of Table~\ref{tab:Microscopic_distributions_test}. Panels (a), (c), (e), and (g) show the ``ground truth"  density distribution in $\Omega$, while panels (b), (d), (f), and (h), show the predicted one, for the time steps $t_{250} = 62.5\,\text{s}$, $t_{500} = 125\,\text{s}$, $t_{780} = 195\,\text{s}$, and $t_{930} =232.5\,\text{s}$, respectively.}
    \label{fig:LSTMw4den}
\end{figure}

\begin{figure}[htbp]
    \centering
    \begin{subfigure}[b]{0.48\textwidth}
        \includegraphics[width=\textwidth]{Figure_E4a.eps}
        \caption{Ground truth density at $t_{250} = 62.5\,\text{s}$}
    \end{subfigure}
    \hfill
    \begin{subfigure}[b]{0.48\textwidth}
        \includegraphics[width=\textwidth]{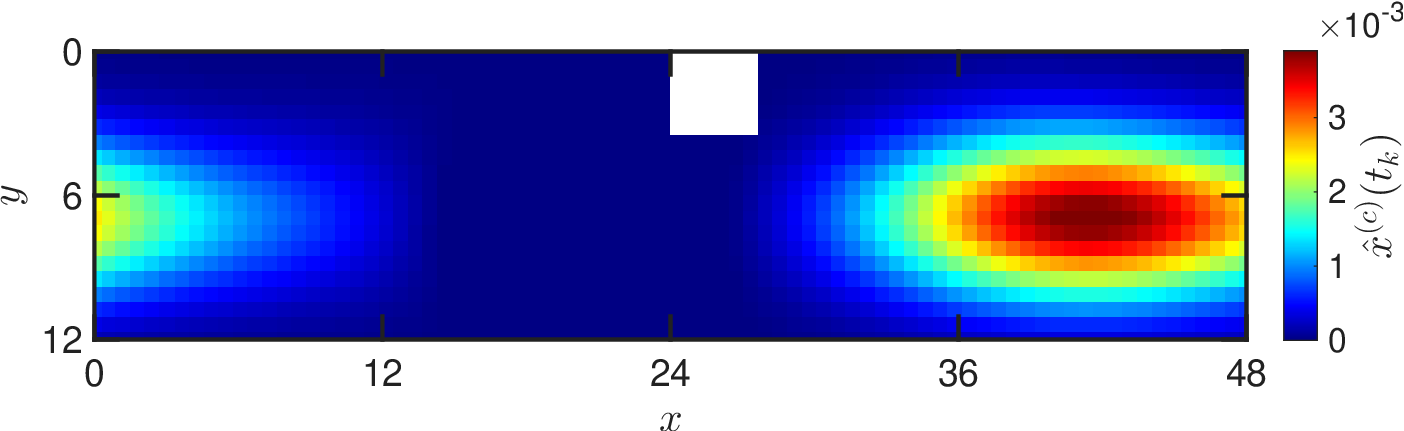}
        \caption{Predicted density at $t_{250} = 62.5\,\text{s}$}
    \end{subfigure}   
    \\
    \begin{subfigure}[b]{0.48\textwidth}
        \includegraphics[width=\textwidth]{Figure_E4c.eps}
        \caption{Ground truth density at $t_{500} = 125\,\text{s}$}
    \end{subfigure}
    \hfill
    \begin{subfigure}[b]{0.48\textwidth}
        \includegraphics[width=\textwidth]{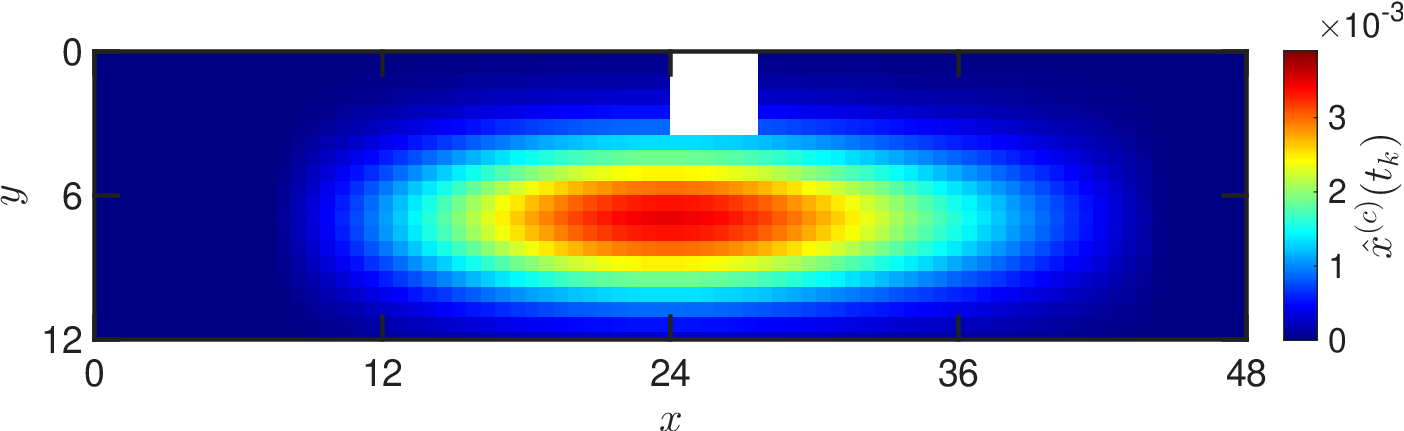}
        \caption{Predicted density at $t_{500} = 125\,\text{s}$}
    \end{subfigure}
    \\
    \begin{subfigure}[b]{0.48\textwidth}
        \includegraphics[width=\textwidth]{Figure_E4e.eps}
        \caption{Ground truth density at $t_{780} = 195\,\text{s}$}
    \end{subfigure}
    \hfill
    \begin{subfigure}[b]{0.48\textwidth}
        \includegraphics[width=\textwidth]{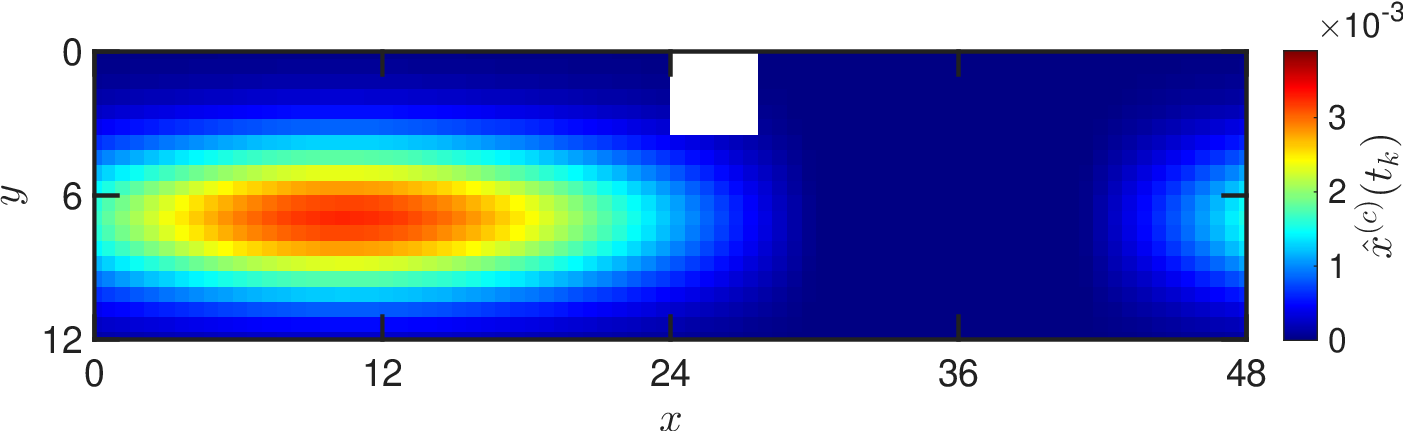}
        \caption{Predicted density at $t_{780} = 195\,\text{s}$}
    \end{subfigure}
    \\
    \begin{subfigure}[b]{0.48\textwidth}
        \includegraphics[width=\textwidth]{Figure_E4g.eps}
        \caption{Ground truth density at $t_{930} =232.5\,\text{s}$}
    \end{subfigure}
    \hfill
    \begin{subfigure}[b]{0.48\textwidth}
        \includegraphics[width=\textwidth]{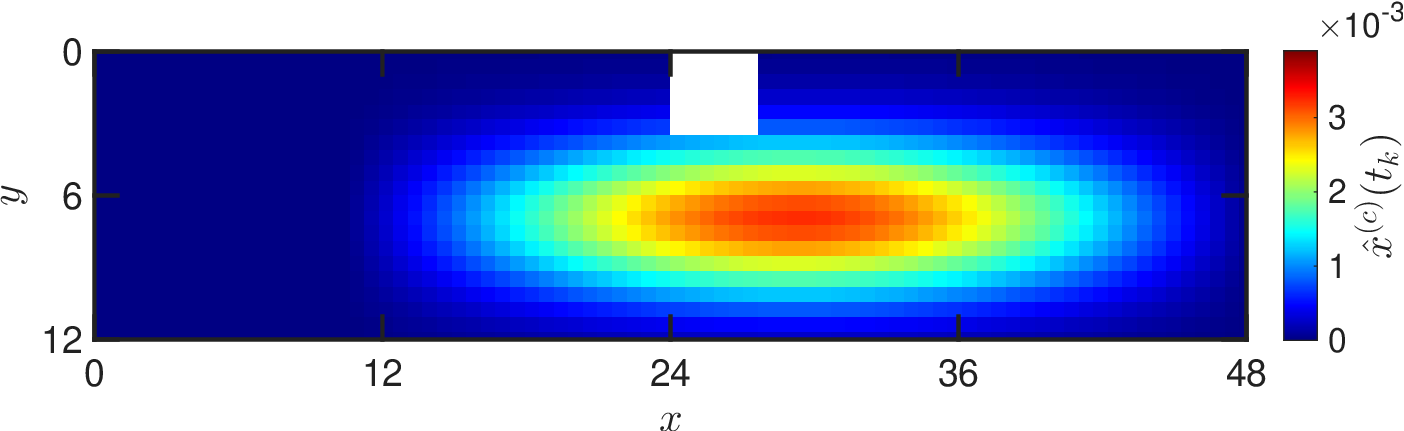}
        \caption{Predicted density at $t_{930} =232.5\,\text{s}$}
    \end{subfigure}
    \caption{Comparison of the ``ground truth" $x^{(c)}(t_k)$ and the predicted $\hat{x}^{(c)}(t_k)$ normalized densities via the closed-loop time-stepper in Eq.~\eqref{eq:lift_ROMev_rest} for the unidirectional flow case, using the LSTM(9) model for an unseen case of the testing set; initialization at the 6th row of Table~\ref{tab:Microscopic_distributions_test}. Panels (a), (c), (e), and (g) show the ``ground truth"  density distribution in $\Omega$, while panels (b), (d), (f), and (h) show the predicted one, for the time steps $t_{250} = 62.5\,\text{s}$, $t_{500} = 125\,\text{s}$, $t_{780} = 195\,\text{s}$, and $t_{930} =232.5\,\text{s}$, respectively.}
    \label{fig:LSTMw9den}
\end{figure}

\begin{figure}[htbp]
    \centering

    \begin{subfigure}[t]{0.48\textwidth}
        \vspace{0pt}\centering
        \includegraphics[width=\textwidth]{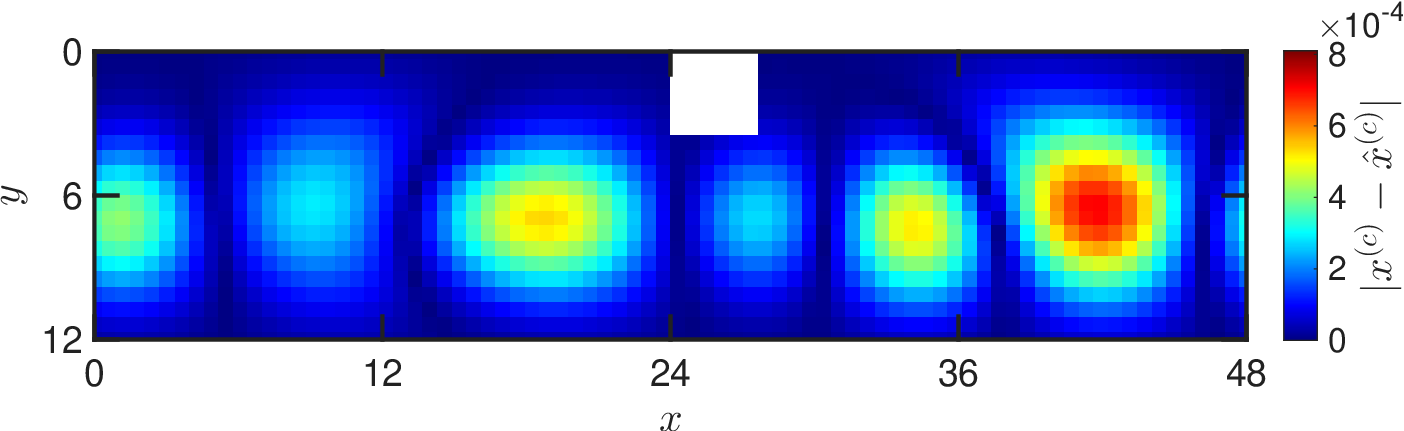}
        \caption{MVAR(4) absolute error at time $t_{250}=62.5\,\mathrm{s}$}
    \end{subfigure}\hfill
    \begin{subfigure}[t]{0.48\textwidth}
        \vspace{0pt}\centering
        \includegraphics[width=\textwidth]{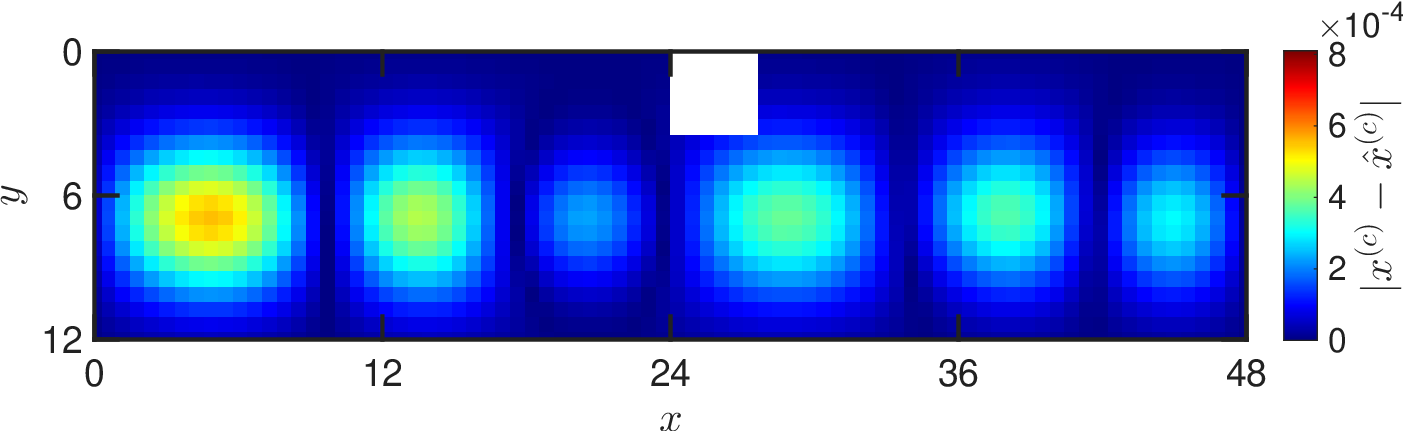}
        \caption{MVAR(4) absolute error at time $t_{930}=232.5\,\mathrm{s}$}
    \end{subfigure}

    \vspace{0.6em}

    \begin{subfigure}[t]{0.48\textwidth}
        \vspace{0pt}\centering
        \includegraphics[width=\textwidth]{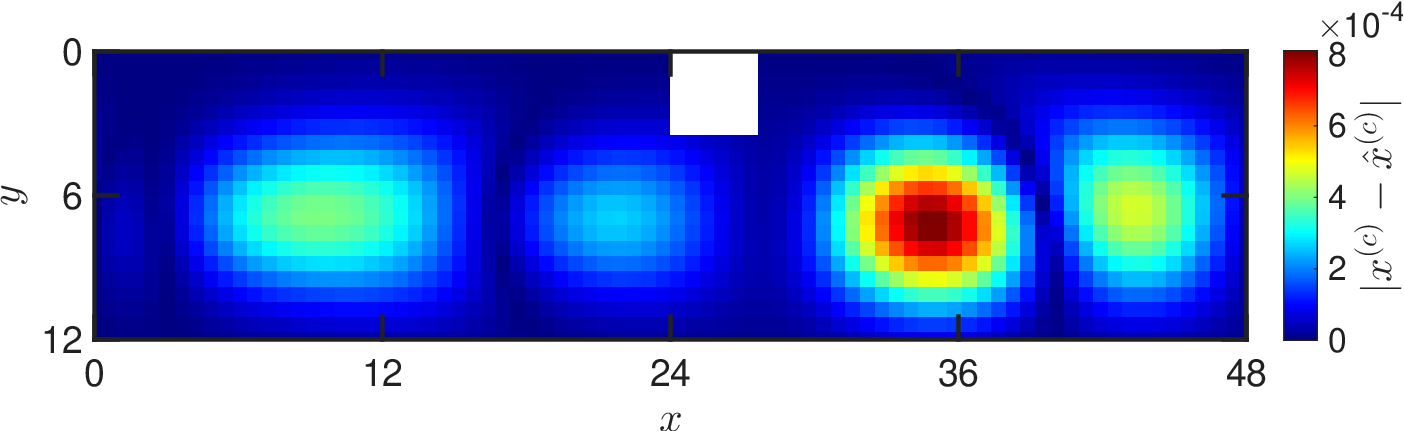}
        \caption{MVAR(9) absolute error at time $t_{250}=62.5\,\mathrm{s}$}
    \end{subfigure}\hfill
    \begin{subfigure}[t]{0.48\textwidth}
        \vspace{0pt}\centering
        \includegraphics[width=\textwidth]{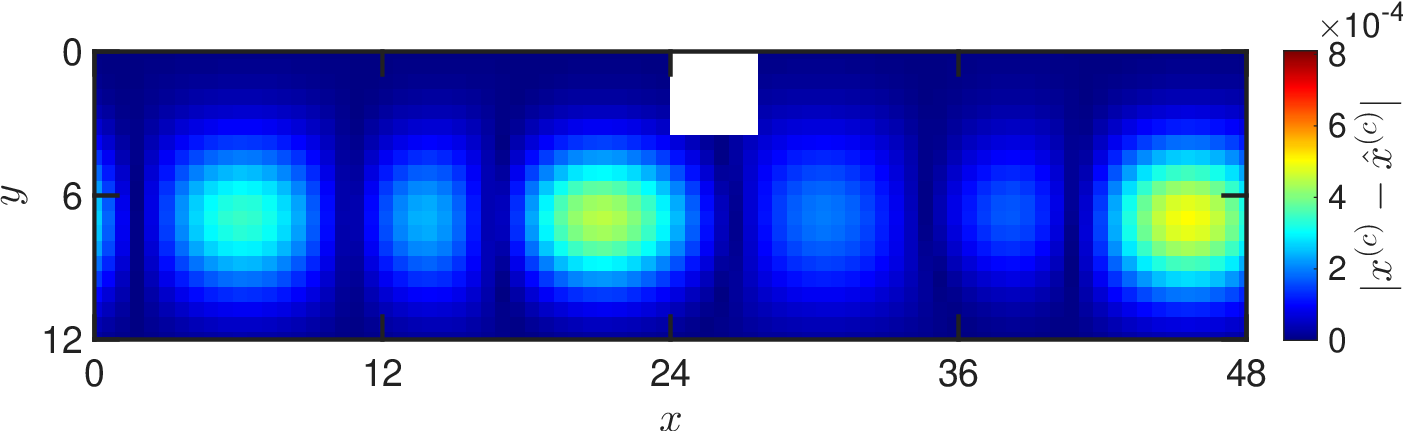}
        \caption{MVAR(9) absolute error at time $t_{930}=232.5\,\mathrm{s}$}
    \end{subfigure}

    \vspace{0.6em}

    \begin{subfigure}[t]{0.48\textwidth}
        \vspace{0pt}\centering
        \includegraphics[width=\textwidth]{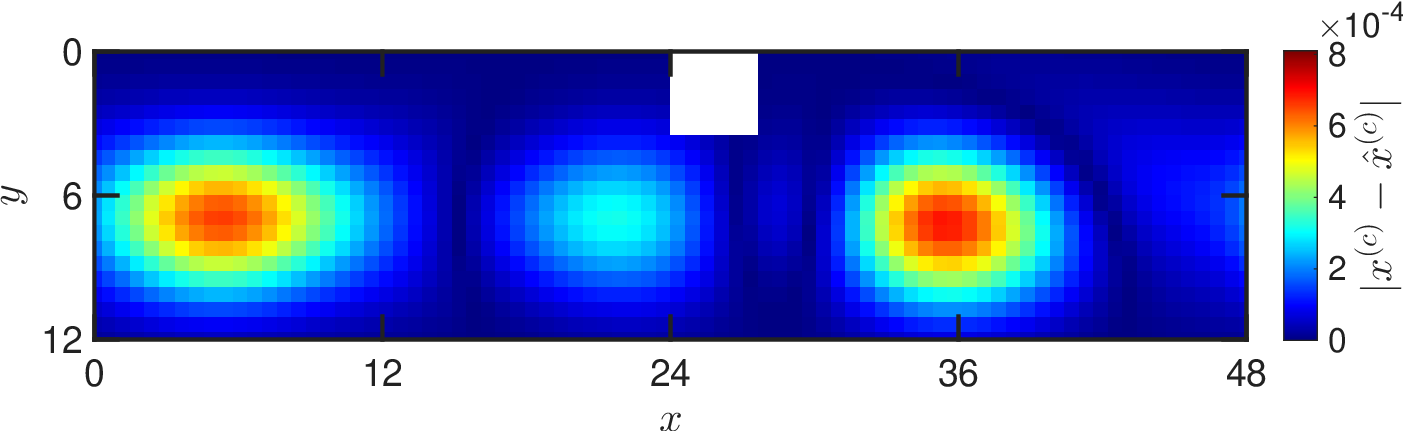}
        \caption{LSTM(4) absolute error at time $t_{250}=62.5\,\mathrm{s}$}
    \end{subfigure}\hfill
    \begin{subfigure}[t]{0.48\textwidth}
        \vspace{0pt}\centering
        \includegraphics[width=\textwidth]{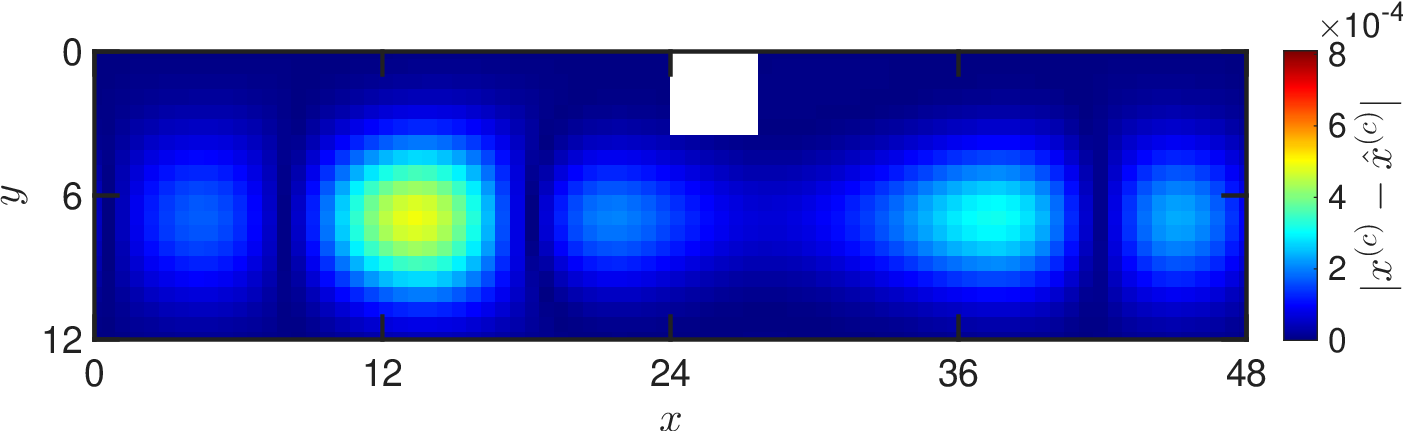}
        \caption{LSTM(4) absolute error at time $t_{930}=232.5\,\mathrm{s}$}
    \end{subfigure}

    \vspace{0.6em}

    \begin{subfigure}[t]{0.48\textwidth}
        \vspace{0pt}\centering
        \includegraphics[width=\textwidth]{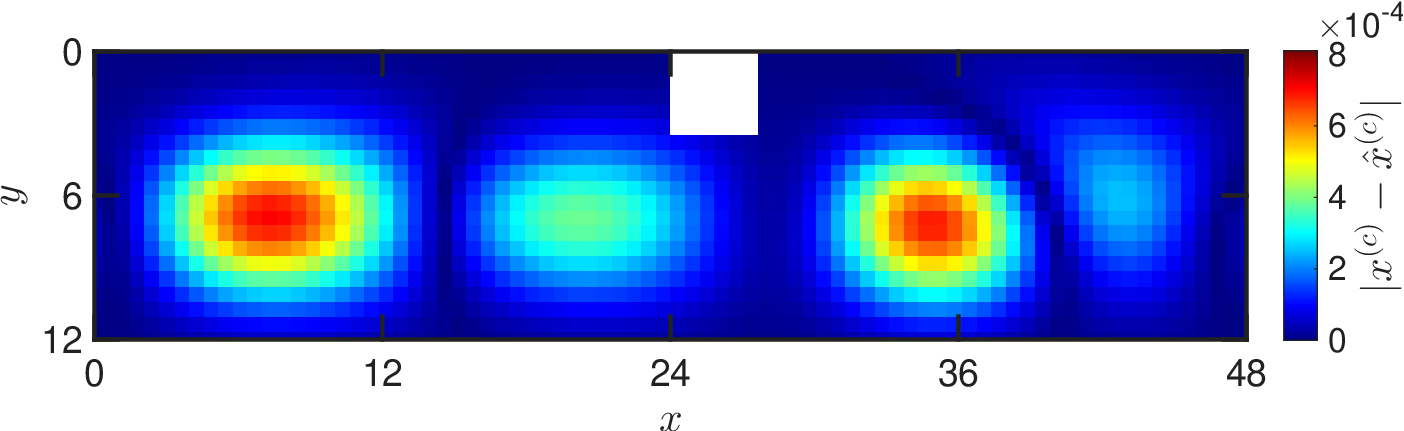}
        \caption{LSTM(9) absolute error at time $t_{250}=62.5\,\mathrm{s}$}
    \end{subfigure}\hfill
    \begin{subfigure}[t]{0.48\textwidth}
        \vspace{0pt}\centering
        \includegraphics[width=\textwidth]{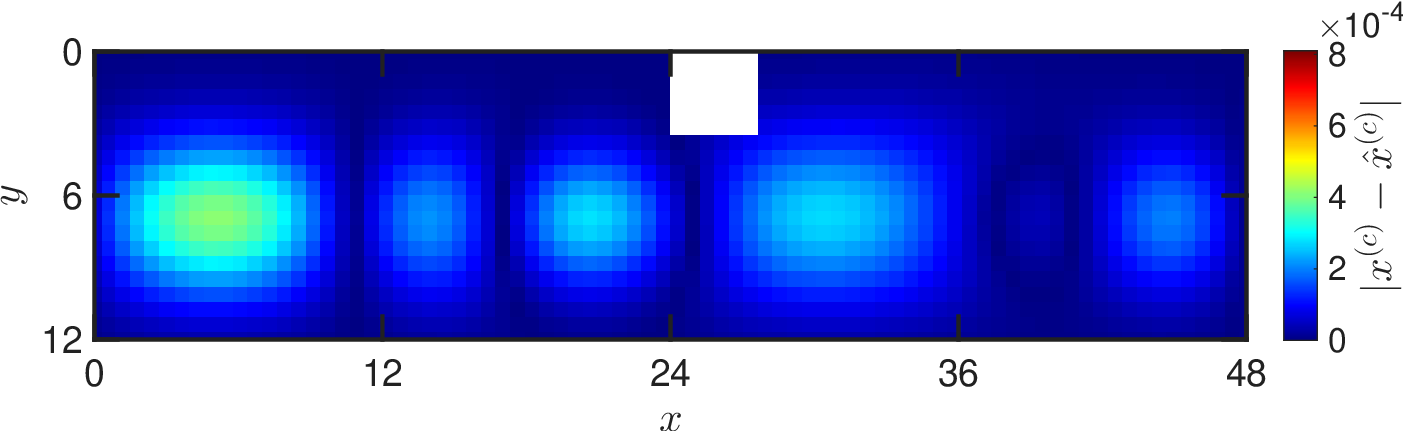}
        \caption{LSTM(9) absolute error at time $t_{930}=232.5\,\mathrm{s}$}
    \end{subfigure}

    \caption{Absolute error between the ``ground truth'' $x^{(c)}(t_k)$ and the predicted $\hat{x}^{(c)}(t_k)$ normalized densities for the unidirectional flow case, corresponding to the predictions in Figs.~\ref{fig:MVARw4den}, \ref{fig:MVARw9den}, \ref{fig:LSTMw4den}, and \ref{fig:LSTMw9den}. Panels (a)--(h) show the spatial distribution of the absolute error $\lvert x^{(c)}(t_k)-\hat{x}^{(c)}(t_k)\rvert$ in $\Omega$ at two indicative time steps $t_{250}=62.5\,\mathrm{s}$ (left column) and $t_{930}=232.5\,\mathrm{s}$ (right column). Panels (a,b) correspond to MVAR(4), (c,d) to MVAR(9), (e,f) to LSTM(4), and (g,h) to LSTM(9) models.}
    \label{fig:Unidirectional_flow_error}
\end{figure}

\clearpage
\newpage
\renewcommand{\theequation}{F.\arabic{equation}}
\renewcommand{\thefigure}{F.\arabic{figure}}
\renewcommand{\thetable}{F.\arabic{table}}
\renewcommand{\thealgorithm}{F.\arabic{algorithm}}
\setcounter{equation}{0}
\setcounter{figure}{0}
\setcounter{table}{0}
\setcounter{algorithm}{0}
\section{Detailed results for the counterflow case} \label{app:CounterFlow}

Here, we provide the detailed results of the proposed framework for the counterflow case discussed in Section~\ref{sb:Counterflow_res}, including baseline errors of the restriction and lifting operators, training of the MVAR and LSTM model, open-loop and closed-loop reconstruction performance. 

Figure~\ref{fig:PODrecon_co} shows the accuracy of the restriction and lifting operators in Eqs.~\eqref{eq:Rest2} and \eqref{eq:Lift2} constructed via the augmented basis in Eq.~\eqref{eq:Aug_Basis}, displaying the evolution of the relative $L_2$ reconstruction error in time per group.

Table~\ref{tab:performance_counter} reports the training results of the MVAR(10) and LSTM(10) models, while Table~\ref{tab:errors_CL_counterflow} summarizes the open- and closed-loop end-to-end performance of the proposed framework.

Figure~\ref{fig:Errors_openloop_co} displays the relative reconstruction errors $e^{2,(c)}_k,\ e^{\infty,(c)}_k$ per group and their 10--90\% percentile ranges for one-step predictions via the open-loop time-stepper (Eq.~\eqref{eq:lift_ROMev_rest_OpenLoop}). Errors are averaged over the $C=20$ cases of different, unseen initial conditions considered, across the time steps $k=1,\ldots,1000$. 

Figures~\ref{fig:Counterflow_densities_MVAR} and \ref{fig:Counterflow_densities_LSTM} compare the ground truth normalized density fields $x^{(l,c)}(t_k)$ and the predicted ones $\hat{x}^{(,c)}(t_k)$ for closed-loop simulations using MVAR(10) and LSTM(10) models, respectively. The test case uses a double gaussian initialization (15th row of Table~\ref{tab:Microscopic_distributions_test}) at four time steps ($k=250,560,780,950$), during which the pedestrian groups do not pass through the corridor boundary. The two interacting populations are shown in red (group~1) and blue (group~2). Panels (a,c,e,g) show the ground truth, while panels (b,d,f,h) present the corresponding closed-loop predictions. The spatial distribution of the resulting absolute errors is shown in Figure~\ref{fig:Counterflow_Error} at time steps $k=250,950$.

\begin{figure}[htbp]
    \centering
    \begin{subfigure}[b]{0.47\textwidth}
        \includegraphics[width=\textwidth]{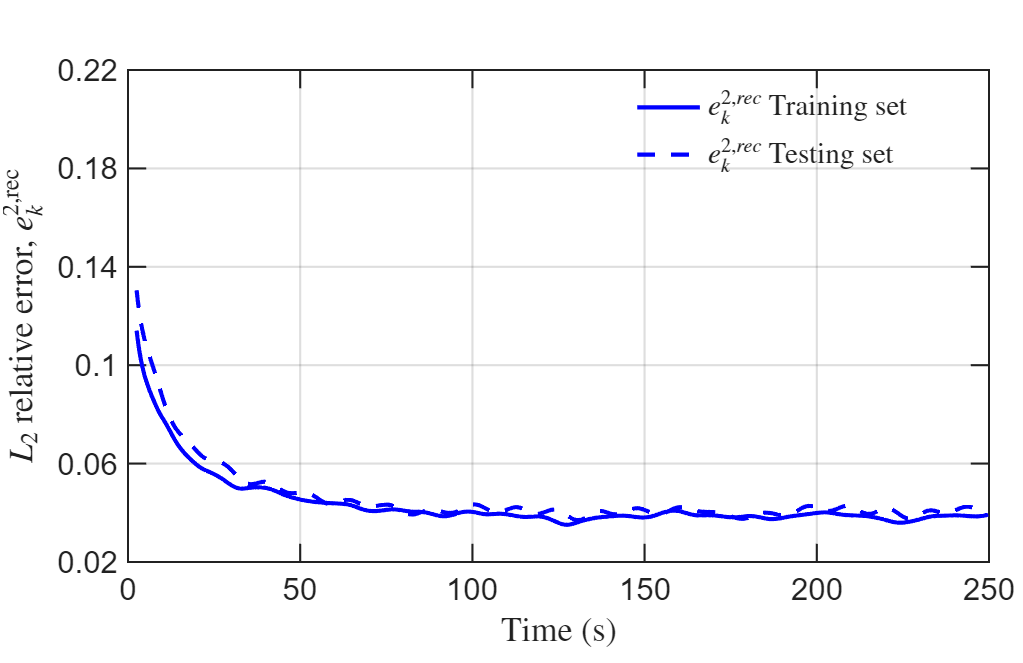}
        \caption{Augmented POD reconstruction error (group~1)}
    \end{subfigure}
    \hfill
    \begin{subfigure}[b]{0.47\textwidth}
        \includegraphics[width=\textwidth]{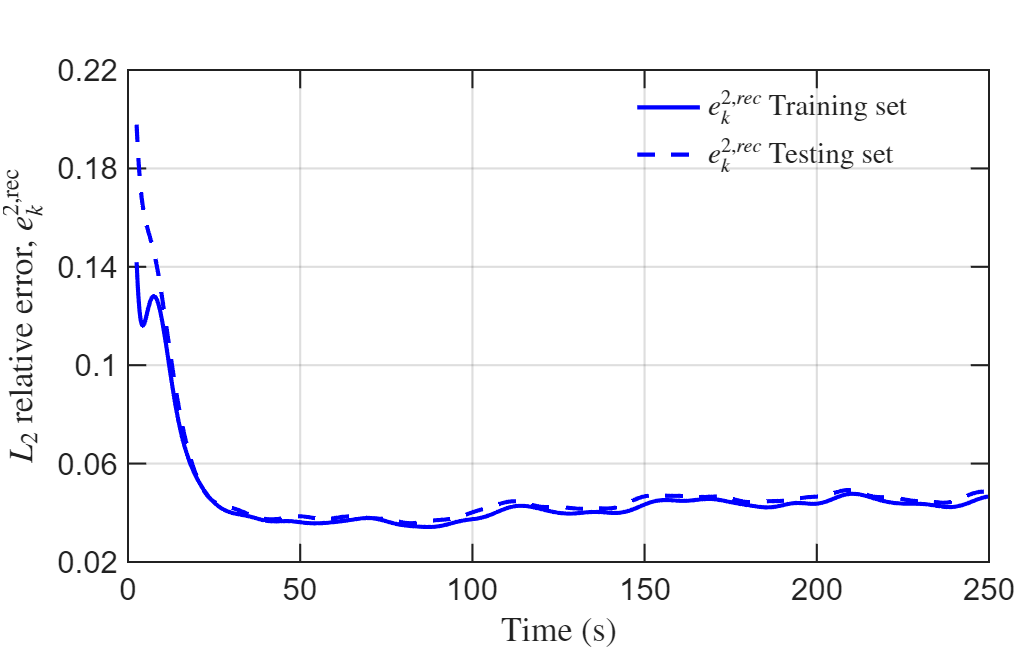}
        \caption{Augmented POD reconstruction error (group~2)}
    \end{subfigure}
    \caption{Accuracy of the restriction and lifting operators in Eqs.~\eqref{eq:Rest2} and \eqref{eq:Lift2} for the counterflow case. Panels (a) and (b) depict the average (over $C=20$ cases with different initial conditions) reconstruction errors $e_k^{2,rec}$ in Eq.~\eqref{eq:PODrecon} for groups 1 and 2, respectively, in time for the training and testing datasets. The augmented basis in Eq.~\eqref{eq:Aug_Basis} is constructed with $d_1=6$, $d_2=8$ POD modes retained for the group-specific matrices $X^{(1)}_{tr},X^{(2)}_{tr}$ and $m=4$ modes retained for cross-covariance matrix $C^{(12)}$.}
    \label{fig:PODrecon_co}
\end{figure}

\begin{table}[htbp]
\centering
\caption{Training results of MVAR(10) and LSTM(10) models at the latent space for the counterflow case. Loss functions $\mathcal{L}_{\text{MVAR}}(\cdot)$ in Eq.~\eqref{eq:mvar_loss} for the MVAR model and $\mathcal{L}_{\text{LSTM}}(\cdot)$ in Eq.~\eqref{eq:lstm_loss} for the LSTM model are reported at the end of training, along with the computational times (in seconds) required for training. For the LSTM, the mean and 10--90\% error percentiles (in parentheses) are shown across the 50 independent training runs with different random initializations. Additionally, the MSE loss function of the MVAR model in Eq.~\eqref{eq:mvar_loss_MSE} is reported for direct comparison to the MSE-based loss $\mathcal{L}_{\text{LSTM}}(\cdot)$.}
\label{tab:performance_counter}
\begin{tabular}{l c c c c}
\toprule
\textbf{Model} 
& $\mathcal{L}_{\text{MVAR}}(\cdot)$ 
& $MSE_{\text{MVAR}}(\cdot)$ 
& $\mathcal{L}_{\text{LSTM}}(\cdot)$ 
& \textbf{Comput. Time (s)} \\
\midrule
MVAR(10) 
& \shortstack{$5.096 \times 10^{-6}$} 
& \shortstack{$1.07 \times 10^{-11}$} 
& \shortstack{-} 
& \shortstack{$0.031$} \\
LSTM(10) 
& \shortstack{-} 
& \shortstack{-} 
& \shortstack{$1.44\ (1.38, 1.71)\times 10^{-7}$} 
& \shortstack{$104\ (84,226)$} \\
\bottomrule
\end{tabular}
\end{table}

\begin{table}[htbp]
\centering
\caption{Open-loop/one-step and closed-loop/recursive prediction errors in the ambient--density profile--space, for the testing set using the trained MVAR and LSTM (best out of 50 training runs) models for the counterflow case. The relative reconstructed $L_1$, $L_2$, and $L_\infty$ errors $e^{1,(c)}_k$, $e^{2,(c)}_k$, and $e^{\infty,(c)}_k$ in Eq.~\eqref{eq:ForcastErr} are reported separately for group~1 and group~2, using the MVAR(10) and LSTM(10) latent dynamics models. Mean values and 10--90\% error percentiles (in parentheses) are shown \emph{over all the} $C=20$ cases and $K=990$ time steps.}
\label{tab:errors_CL_counterflow}
\begin{tabular}{lcccccc}
\toprule
& \multicolumn{3}{c}{\textbf{group 1}} & \multicolumn{3}{c}{\textbf{group 2}} \\
\cmidrule(lr){2-4} \cmidrule(lr){5-7}
&
$L_1$ error, $e^{1,(c)}_k$ & $L_2$ error, $e^{2,(c)}_k$ & $L_\infty$ error, $e^{\infty,(c)}_k$ &
$L_1$ error, $e^{1,(c)}_k$ & $L_2$ error, $e^{2,(c)}_k$ & $L_\infty$ error, $e^{\infty,(c)}_k$ \\
\midrule

\multicolumn{7}{l}{\textbf{Open-Loop}} \\
\midrule
MVAR(10)   &
\shortstack{$0.063$\\$(0.041,\,0.100)$} &
\shortstack{$0.046$\\$(0.027,\,0.076)$} &
\shortstack{$0.046$\\$(0.024,\,0.078)$} &
\shortstack{$0.070$\\$(0.049,\,0.089)$} &
\shortstack{$0.048$\\$(0.034,\,0.063)$} &
\shortstack{$0.051$\\$(0.034,\,0.073)$} \\

LSTM(10)  &
\shortstack{$0.075$\\$(0.049,\,0.109)$} &
\shortstack{$0.057$\\$(0.034,\,0.086)$} &
\shortstack{$0.060$\\$(0.032,\,0.095)$} &
\shortstack{$0.082$\\$(0.061,\,0.101)$} &
\shortstack{$0.056$\\$(0.041,\,0.071)$} &
\shortstack{$0.064$\\$(0.043,\,0.092)$} \\

\midrule
\multicolumn{7}{l}{\textbf{Closed-Loop}} \\
\midrule
MVAR(10) &
\shortstack{$0.108$\\$(0.064,\,0.164)$} &
\shortstack{$0.083$\\$(0.046,\,0.123)$} &
\shortstack{$0.091$\\$(0.046,\,0.151)$} &
\shortstack{$0.118$\\$(0.086,\,0.152)$} &
\shortstack{$0.088$\\$(0.058,\,0.115)$} &
\shortstack{$0.099$\\$(0.060,\,0.144)$} \\
 
LSTM(10)  &
\shortstack{$0.132$\\$(0.090,\,0.201)$} &
\shortstack{$0.104$\\$(0.067,\,0.168)$} &
\shortstack{$0.113$\\$(0.061,\,0.197)$} &
\shortstack{$0.123$\\$(0.095,\,0.157)$} &
\shortstack{$0.095$\\$(0.068,\,0.123)$} &
\shortstack{$0.115$\\$(0.070,\,0.152)$} \\

\bottomrule
\end{tabular}
\end{table}


\begin{figure}[htbp]
    \centering
    \begin{subfigure}[b]{0.47\textwidth}
        \includegraphics[width=\textwidth]{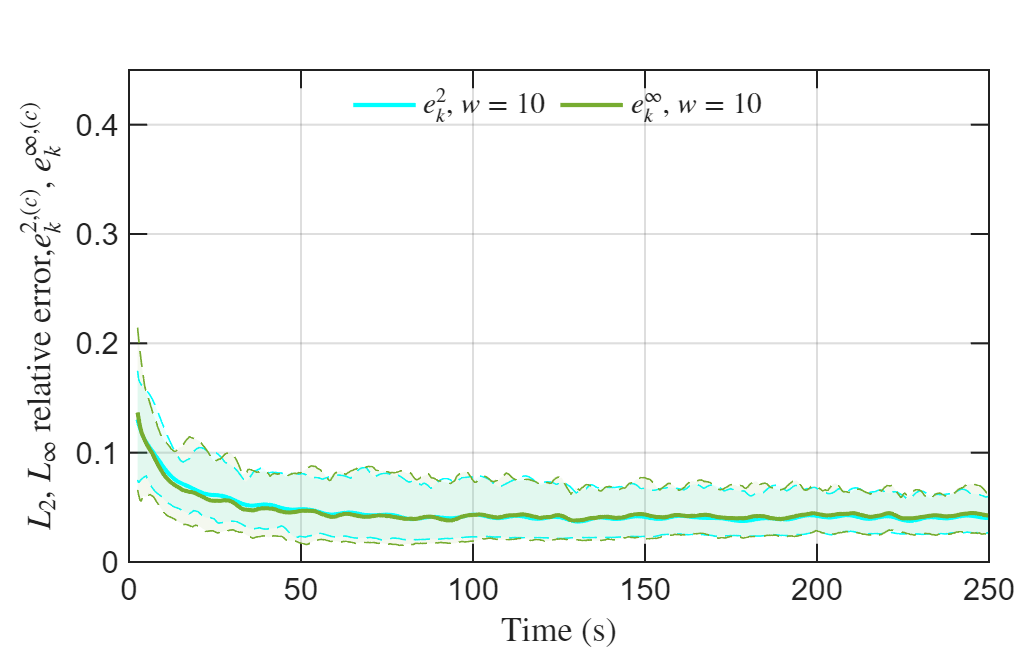}
        \caption{Relative Errors with MVAR(10) (group~1)}
    \end{subfigure}
    \hfill
    \begin{subfigure}[b]{0.47\textwidth}
        \includegraphics[width=\textwidth]{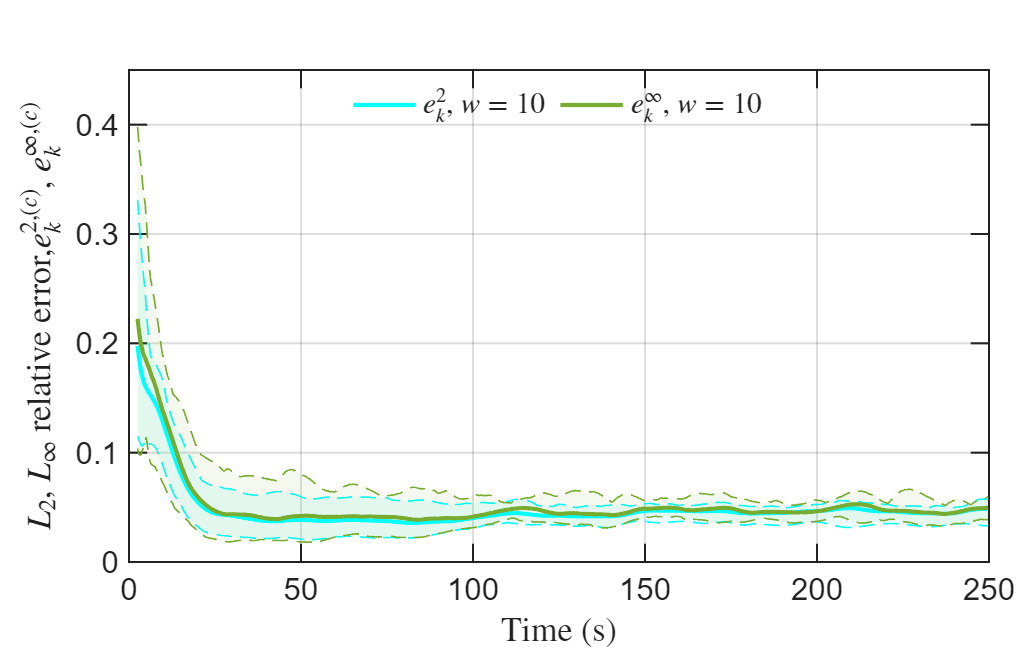}
        \caption{Relative Errors with MVAR(10) (group~2)}
    \end{subfigure}
    
    \vspace{0.3cm}
    
    \begin{subfigure}[b]{0.47\textwidth}
        \includegraphics[width=\textwidth]{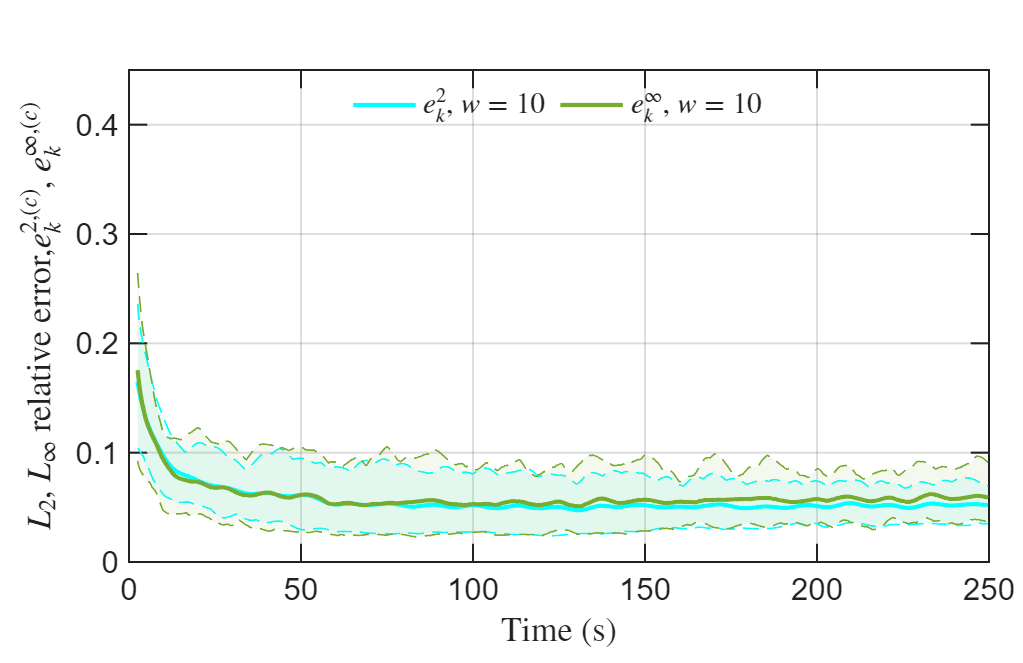}
        \caption{Relative Errors with LSTM(10) (group~1)}
    \end{subfigure}
    \hfill
    \begin{subfigure}[b]{0.47\textwidth}
        \includegraphics[width=\textwidth]{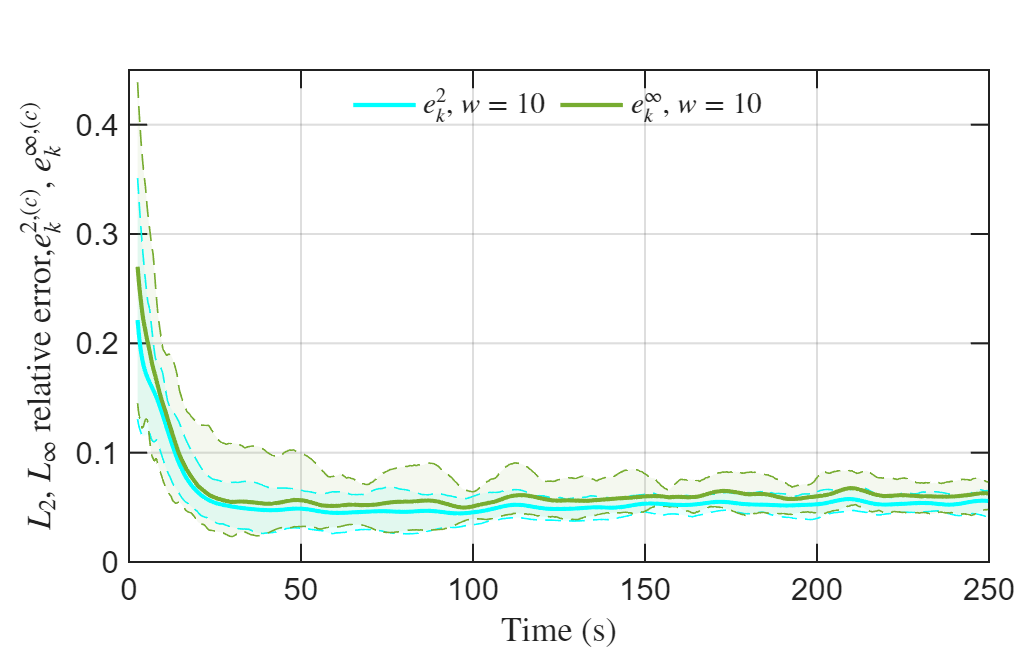}
        \caption{Relative Errors with LSTM(10) (group~2)}
    \end{subfigure}
    
    \caption{One-step (open-loop) prediction errors, in the ambient–-density profile–-space, for the testing set over time, using trained MVAR and LSTM models for the counterflow case. The relative reconstructed $L_2$ (cyan) and $L_\infty$ (green) errors $e^{2,(c)}_k$ and $ e^{\infty,(c)}_k$ Eq.~\eqref{eq:ForcastErr} (using the one-step prediction in Eq.~\eqref{eq:lift_ROMev_rest_OpenLoop}) are shown for the MVAR(10) and LSTM(10) latent dynamics models in panels (a,b) (c,d) per group, respectively. Mean relative errors (solid) and 10--90\% error percentiles (dashed) are shown over $C=20$ cases per time step. For the LSTM models, results correspond to the model achieving the best training performance out of the 50 independent runs.}
    \label{fig:Errors_openloop_co}
\end{figure}

\begin{figure}[htbp]
    \centering
    \begin{subfigure}[t]{0.48\textwidth}
        \vspace{0pt}\centering
        \includegraphics[width=\textwidth]{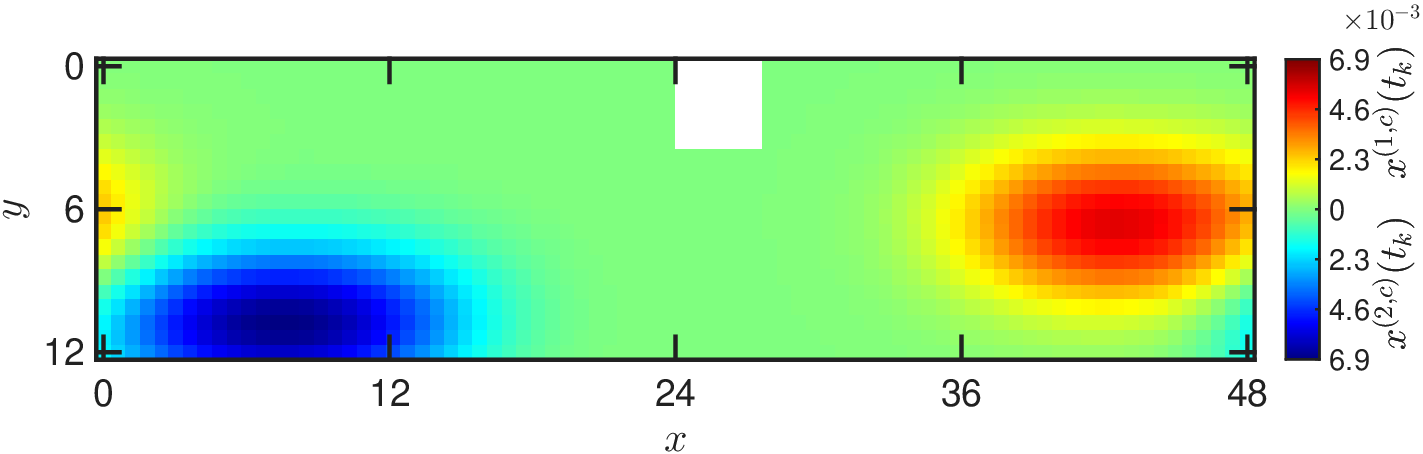}
        \caption{Ground truth density at $t_{250}=62.5\,\mathrm{s}$}
    \end{subfigure}\hfill
    \begin{subfigure}[t]{0.48\textwidth}
        \vspace{0pt}\centering
        \includegraphics[width=\textwidth]{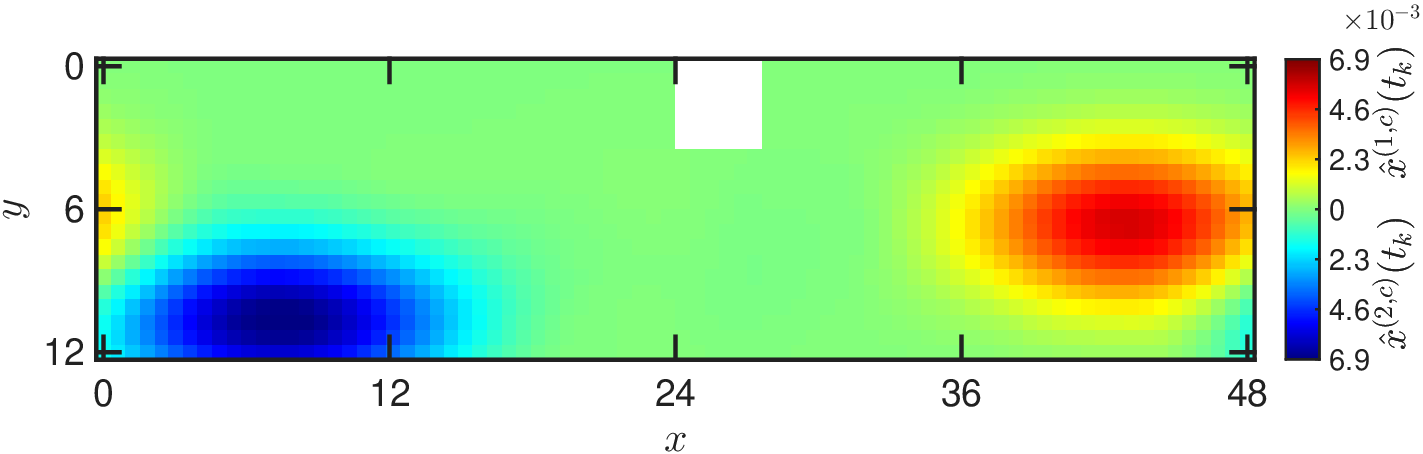}
        \caption{Predicted density at $t_{250}=62.5\,\mathrm{s}$}
    \end{subfigure}

    \vspace{0.6em}

    \begin{subfigure}[t]{0.48\textwidth}
        \vspace{0pt}\centering
        \includegraphics[width=\textwidth]{Figure_F3c}
        \caption{Ground truth density at $t_{560}=140\,\mathrm{s}$}
    \end{subfigure}\hfill
    \begin{subfigure}[t]{0.48\textwidth}
        \vspace{0pt}\centering
        \includegraphics[width=\textwidth]{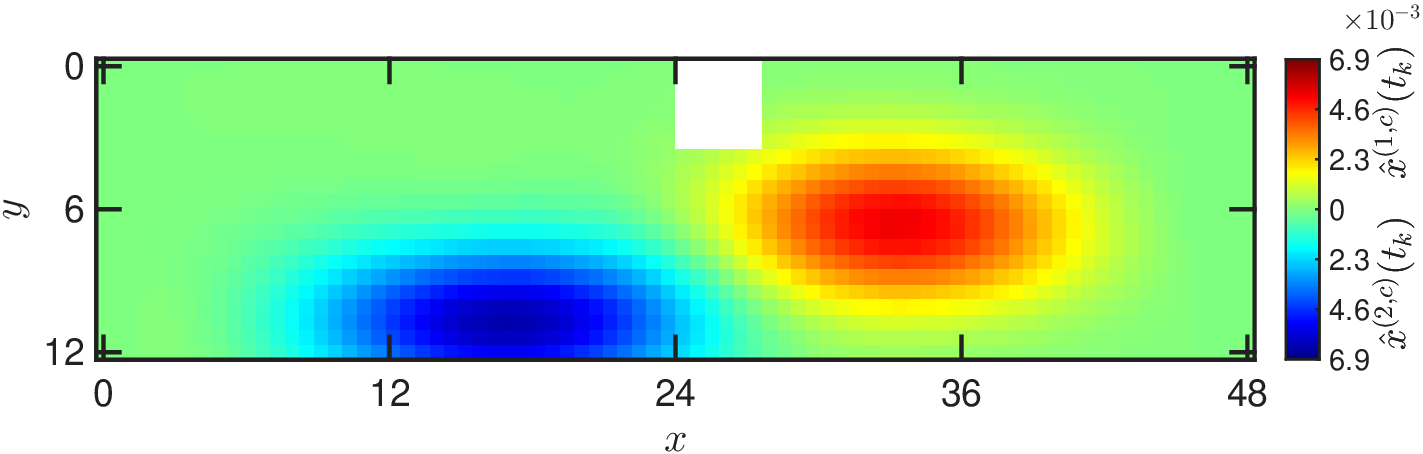}
        \caption{Predicted density at $t_{560}=140\,\mathrm{s}$}
    \end{subfigure}

    \vspace{0.6em}

    \begin{subfigure}[t]{0.48\textwidth}
        \vspace{0pt}\centering
        \includegraphics[width=\textwidth]{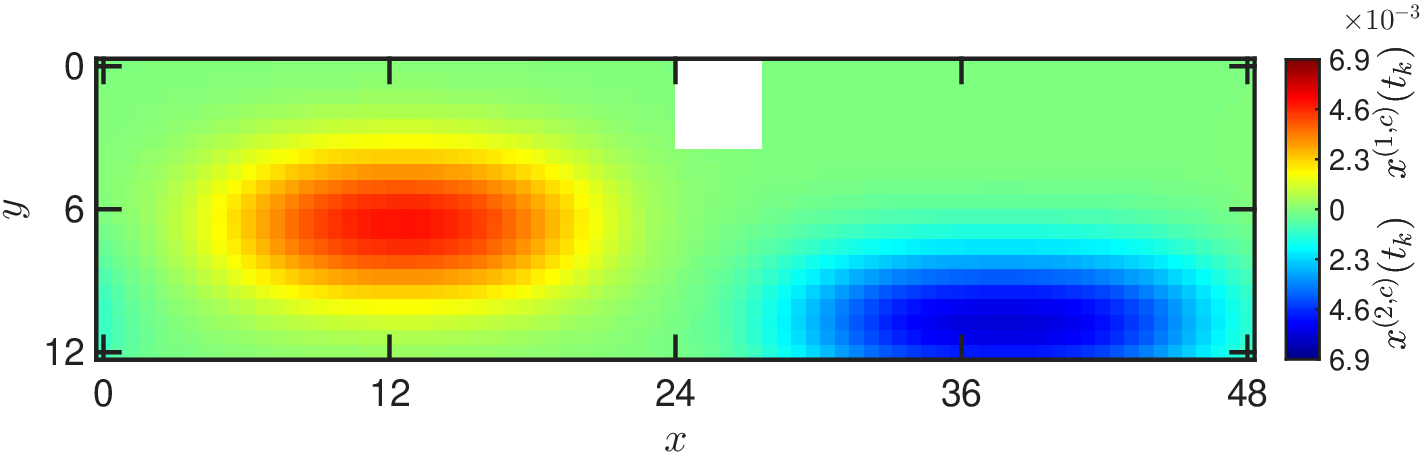}
        \caption{Ground truth density at $t_{780}=195\,\mathrm{s}$}
    \end{subfigure}\hfill
    \begin{subfigure}[t]{0.48\textwidth}
        \vspace{0pt}\centering
        \includegraphics[width=\textwidth]{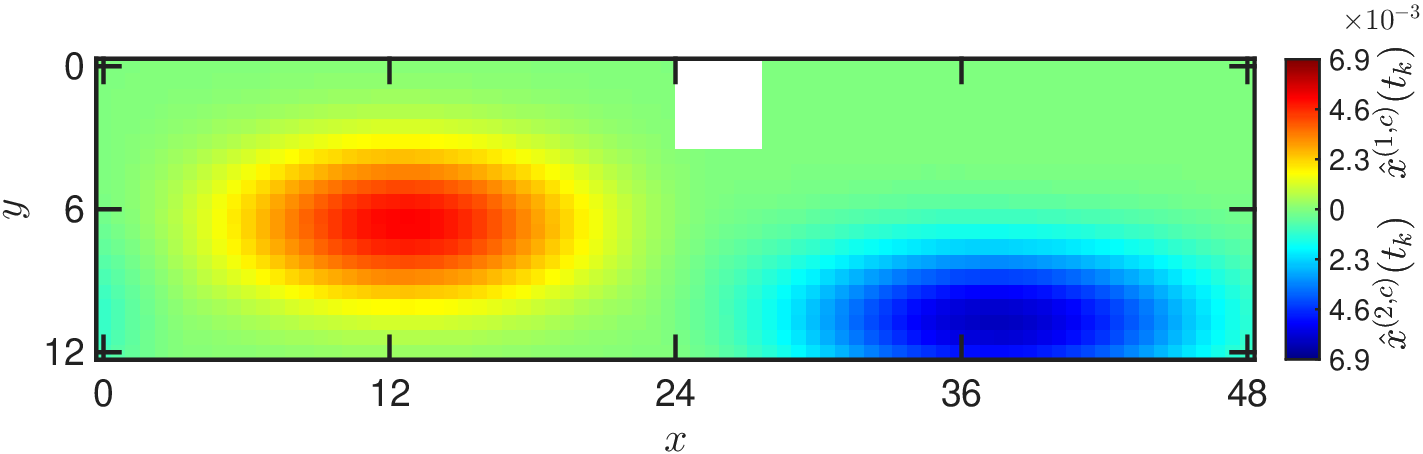}
        \caption{Predicted density at $t_{780}=195\,\mathrm{s}$}
    \end{subfigure}
    \begin{subfigure}[t]{0.48\textwidth}
        \vspace{0pt}\centering
        \includegraphics[width=\textwidth]{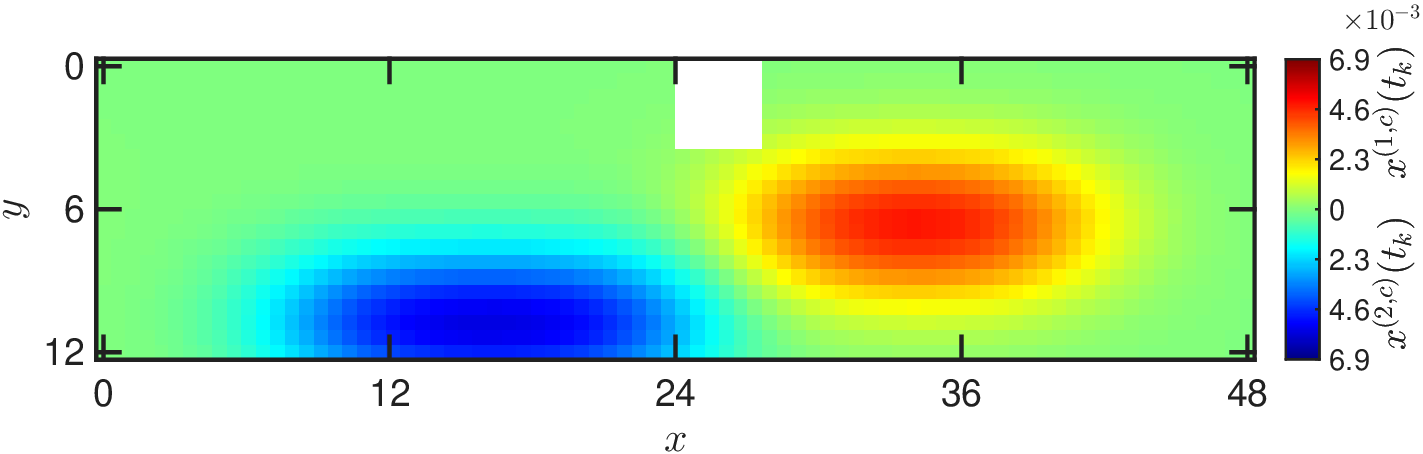}
        \caption{Ground truth density at $t_{950}=237.5\,\mathrm{s}$}
    \end{subfigure}\hfill
    \begin{subfigure}[t]{0.48\textwidth}
        \vspace{0pt}\centering
        \includegraphics[width=\textwidth]{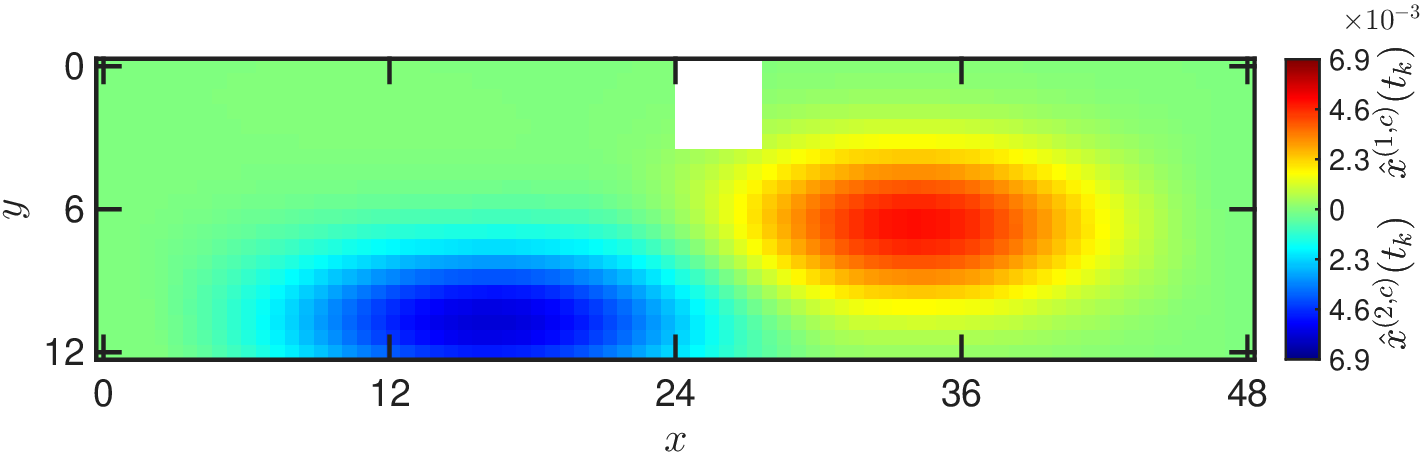}
        \caption{Predicted density at $t_{950}=237.5\,\mathrm{s}$}
    \end{subfigure}
    \caption{Comparison of the ``ground truth'' $x^{(l,c)}(t_k)$ and the predicted $\hat{x}^{(l,c)}(t_k)$ normalized densities via the closed-loop time-stepper in Eq.~\eqref{eq:lift_ROMev_rest} for groups $l=1,2$ in the counterflow case, using the MVAR(10) model for an unseen case of the testing set; initialization at the 15th row of Table~\ref{tab:Microscopic_distributions_test}. Panels (a), (c), (e) and (g) show the ``ground truth'' density distribution in $\Omega$, while panels (b), (d), (f) and (h) show the predicted one, for the timesteps $t_{250}=62.5\,\mathrm{s}$, $t_{560}=140\,\mathrm{s}$, $t_{780}=195\,\mathrm{s}$, and $t_{950}=237.5\,\mathrm{s}$, respectively. The group-specific densities are superimposed with the blue and red lobes corresponding to groups 1 and 2, respectively.}
    \label{fig:Counterflow_densities_MVAR}
\end{figure}

\begin{figure}[htbp]
    \centering

    \begin{subfigure}[t]{0.48\textwidth}
        \vspace{0pt}\centering
        \includegraphics[width=\textwidth]{Figure_F3a.eps}
        \caption{Ground truth density at $t_{250}=62.5\,\mathrm{s}$}
    \end{subfigure}\hfill
    \begin{subfigure}[t]{0.48\textwidth}
        \vspace{0pt}\centering
        \includegraphics[width=\textwidth]{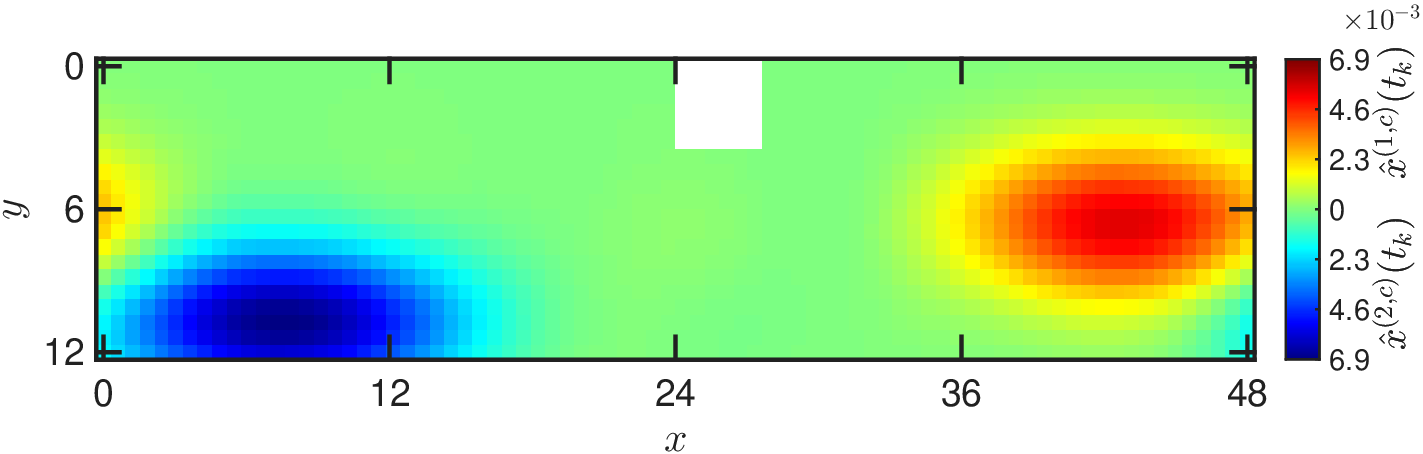}
        \caption{Predicted density at $t_{250}=62.5\,\mathrm{s}$}
    \end{subfigure}
    
    \vspace{0.6em}

    \begin{subfigure}[t]{0.48\textwidth}
        \vspace{0pt}\centering
        \includegraphics[width=\textwidth]{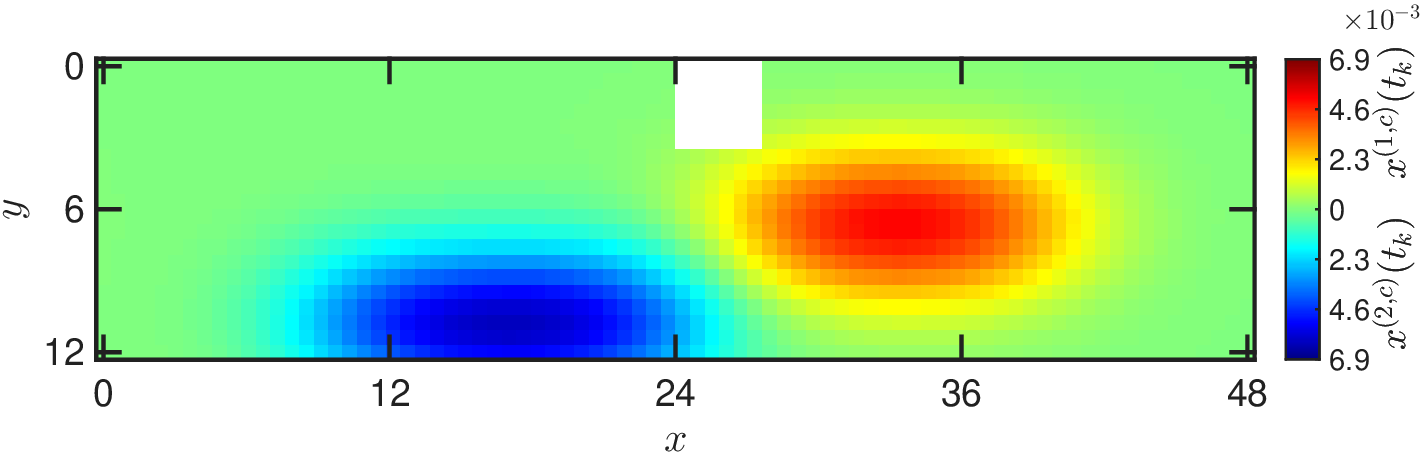}
        \caption{Ground truth density at $t_{560}=140\,\mathrm{s}$}
    \end{subfigure}\hfill
    \begin{subfigure}[t]{0.48\textwidth}
        \vspace{0pt}\centering
        \includegraphics[width=\textwidth]{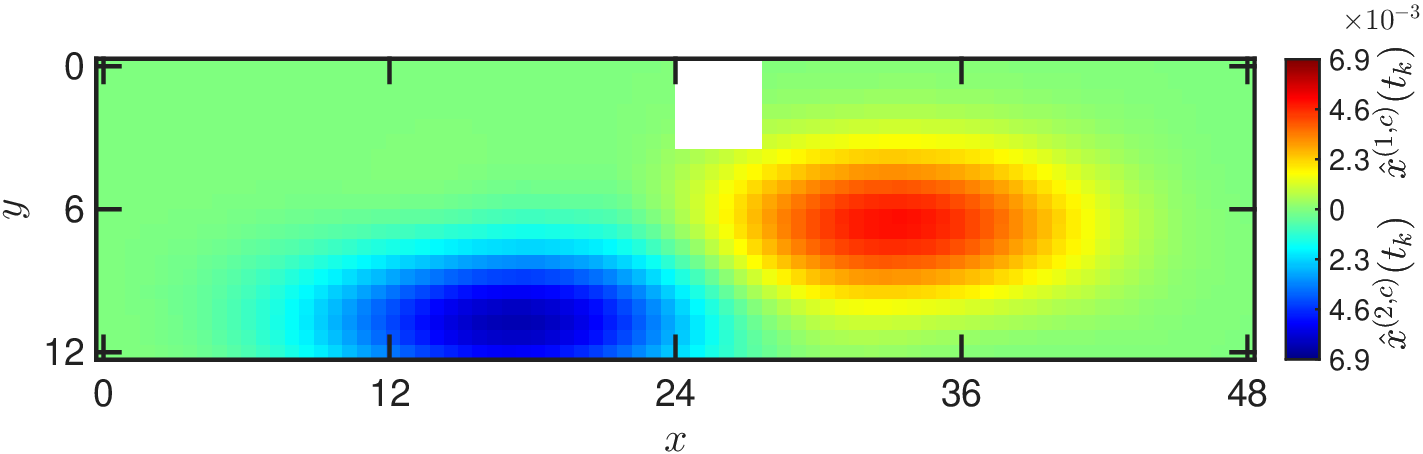}
        \caption{Predicted density at $t_{560}=140\,\mathrm{s}$}
    \end{subfigure}

    \vspace{0.6em}

    \begin{subfigure}[t]{0.48\textwidth}
        \vspace{0pt}\centering
        \includegraphics[width=\textwidth]{Figure_F3e.eps}
        \caption{Ground truth density at $t_{780}=195\,\mathrm{s}$}
    \end{subfigure}\hfill
    \begin{subfigure}[t]{0.48\textwidth}
        \vspace{0pt}\centering
        \includegraphics[width=\textwidth]{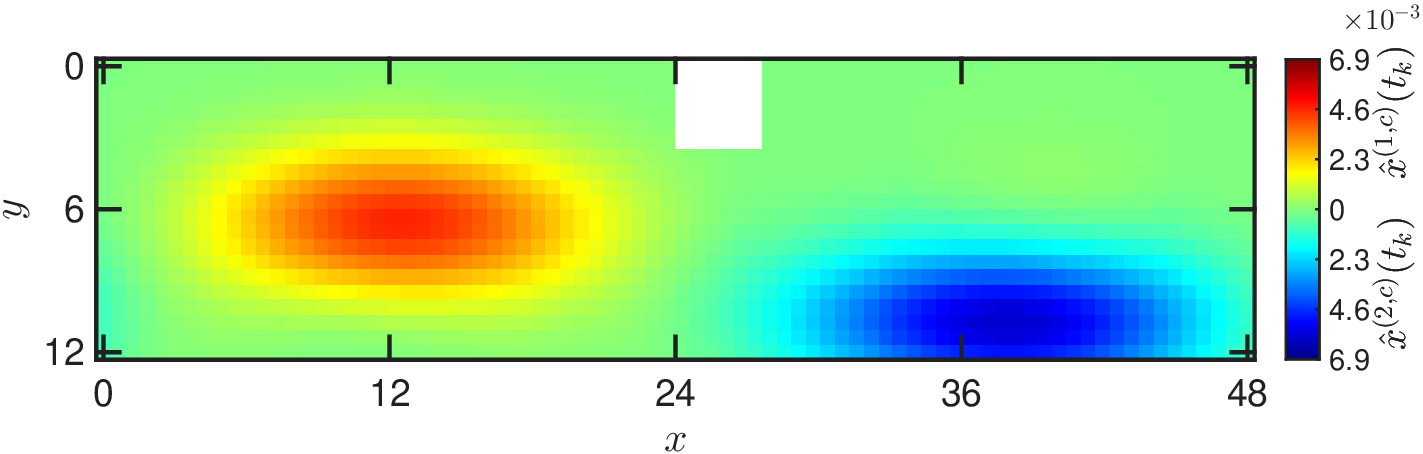}
        \caption{Predicted density at $t_{780}=195\,\mathrm{s}$}
    \end{subfigure}
    \begin{subfigure}[t]{0.48\textwidth}
        \vspace{0pt}\centering
        \includegraphics[width=\textwidth]{Figure_F3g.eps}
        \caption{Ground truth density at $t_{950}=237.5\,\mathrm{s}$}
    \end{subfigure}\hfill
    \begin{subfigure}[t]{0.48\textwidth}
        \vspace{0pt}\centering
        \includegraphics[width=\textwidth]{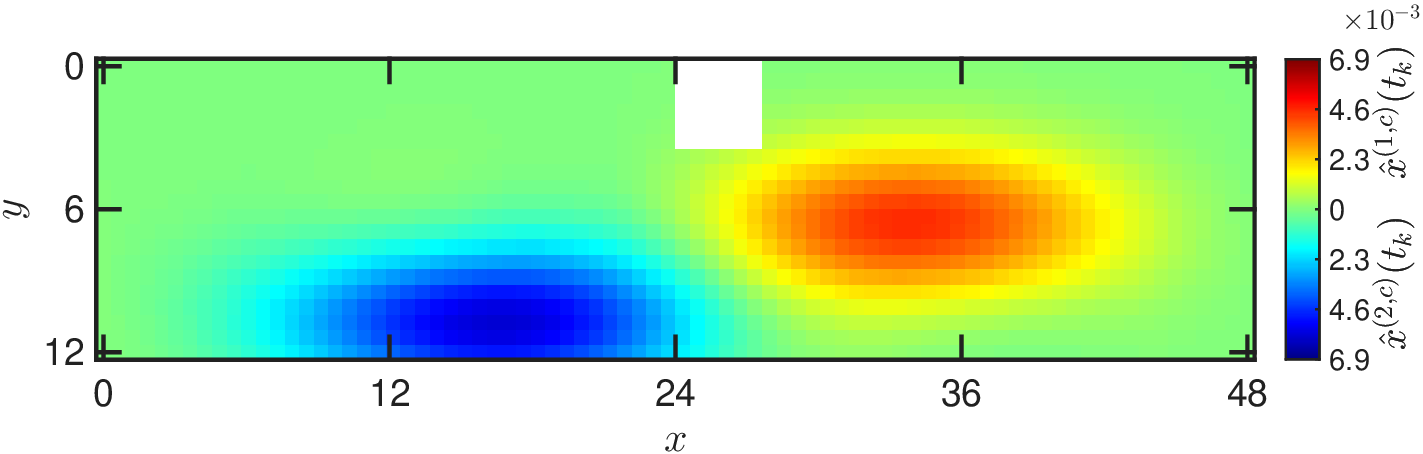}
        \caption{Predicted density at $t_{950}=237.5\,\mathrm{s}$}
    \end{subfigure}
    \caption{Comparison of the ``ground truth'' $x^{(l,c)}(t_k)$ and the predicted $\hat{x}^{(l,c)}(t_k)$ normalized densities via the closed-loop time-stepper in Eq.~\eqref{eq:lift_ROMev_rest} for groups $l=1,2$ in the counterflow case, using the LSTM(10) model for an unseen case of the testing set; initialization at the 15th row of Table~\ref{tab:Microscopic_distributions_test}. Panels (a), (c), (e) and (g) show the ``ground truth'' density distribution in $\Omega$, while panels (b), (d), (f) and (h) show the predicted one, for the timesteps $t_{250}=62.5\,\mathrm{s}$, $t_{560}=140\,\mathrm{s}$, $t_{780}=195\,\mathrm{s}$, and $t_{950}=237.5\,\mathrm{s}$, respectively. The group-specific densities are superimposed with the blue and red lobes corresponding to groups 1 and 2, respectively.}
    \label{fig:Counterflow_densities_LSTM}
\end{figure}

\begin{figure}[htbp]
    \centering
    \begin{subfigure}[t]{0.48\textwidth}
        \vspace{0pt}\centering
        \includegraphics[width=\textwidth]{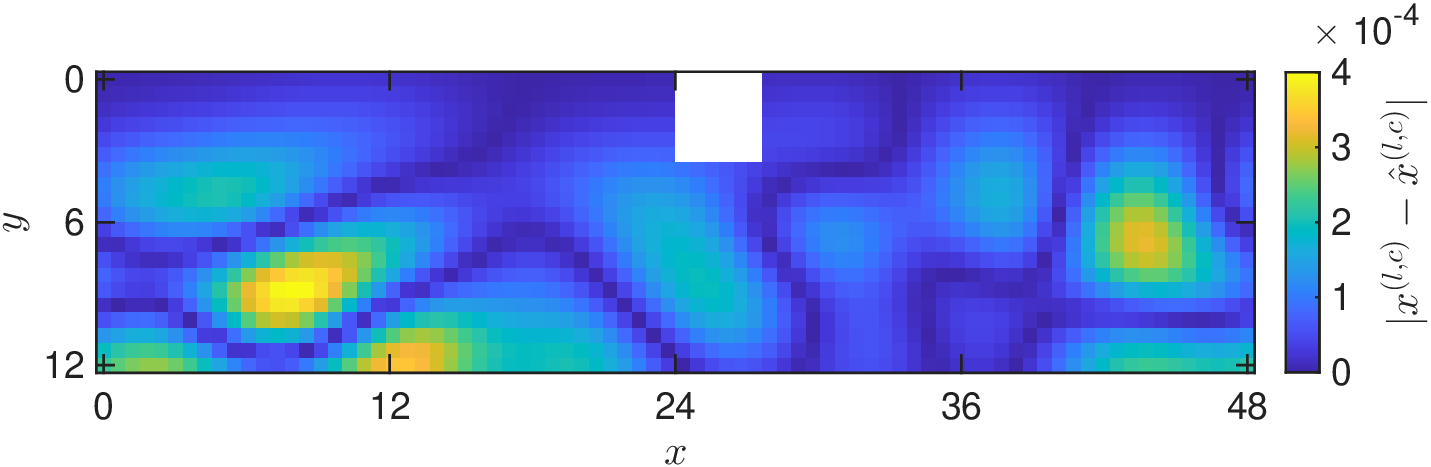}
        \caption{MVAR(10) absolute error at time $t_{250}=62.5\,\mathrm{s}$}
    \end{subfigure}\hfill
    \begin{subfigure}[t]{0.48\textwidth}
        \vspace{0pt}\centering
        \includegraphics[width=\textwidth]{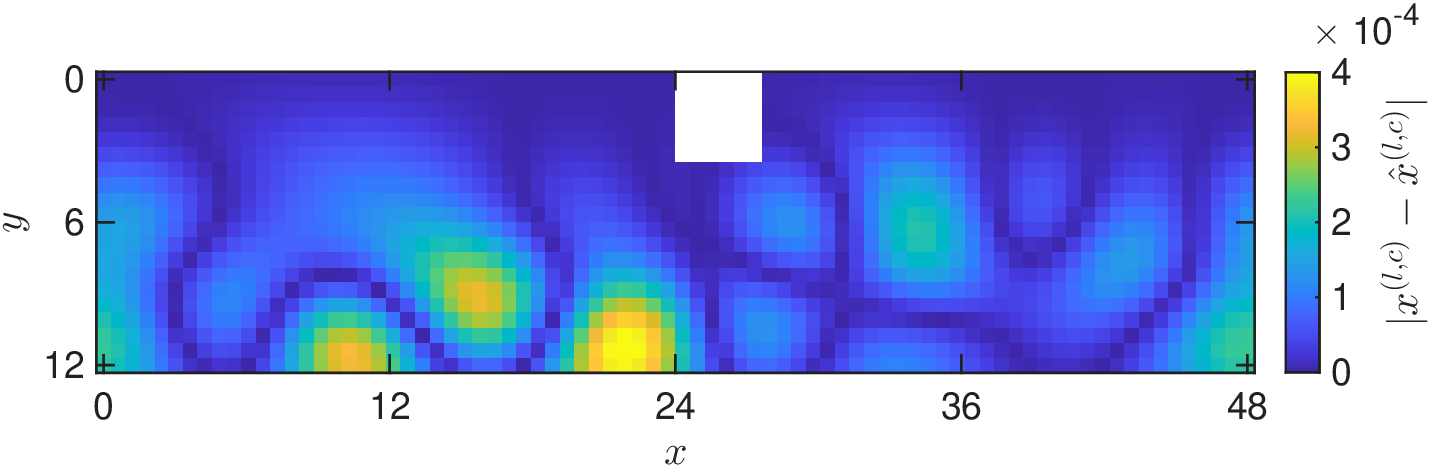}
        \caption{MVAR(10) absolute error at time $t_{950}=237.5\,\mathrm{s}$}
    \end{subfigure}

    \vspace{0.6em}

    \begin{subfigure}[t]{0.48\textwidth}
        \vspace{0pt}\centering
        \includegraphics[width=\textwidth]{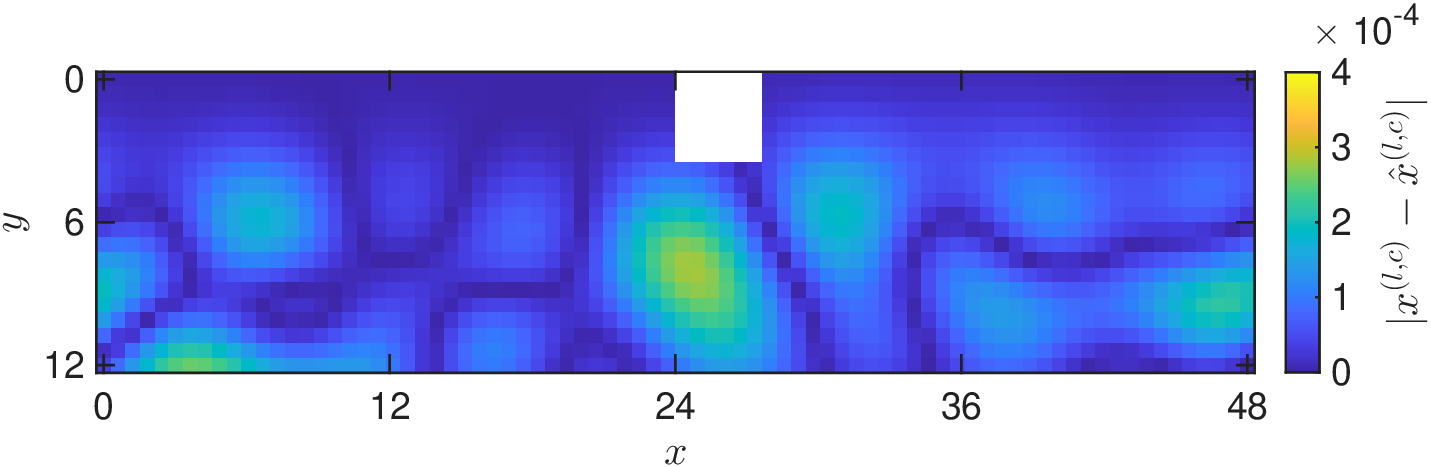}
        \caption{LSTM(10) absolute error at time $t_{250}=62.5\,\mathrm{s}$}
    \end{subfigure}\hfill
    \begin{subfigure}[t]{0.48\textwidth}
        \vspace{0pt}\centering
        \includegraphics[width=\textwidth]{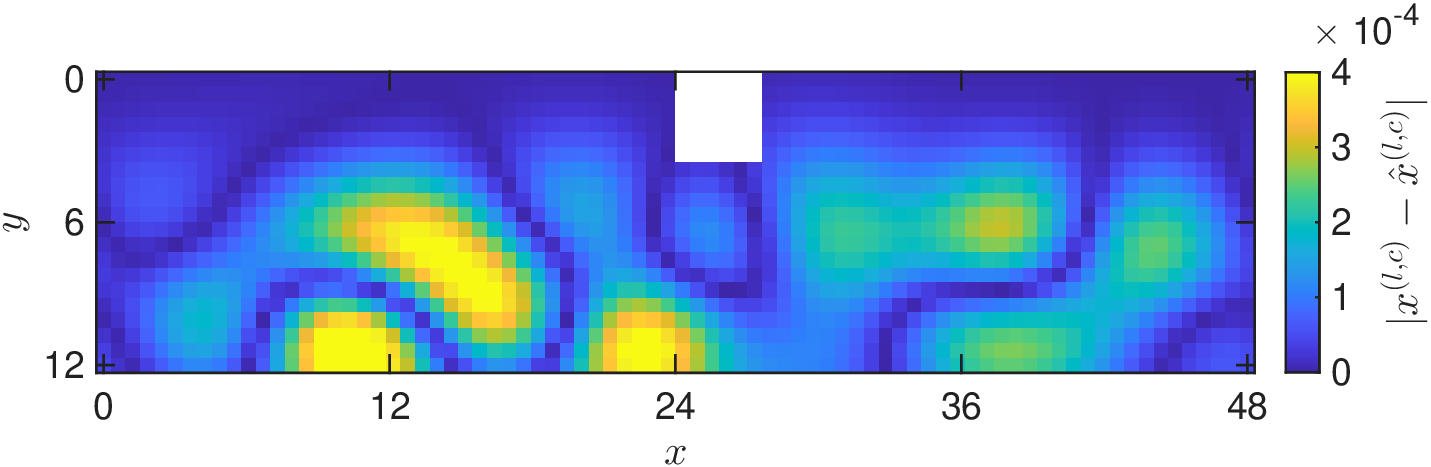}
        \caption{LSTM(10) absolute error at time $t_{950}=237.5\,\mathrm{s}$}
    \end{subfigure}
    \caption{Absolute error between the ``ground truth'' $x^{(l,c)}(t_k)$ and the predicted $\hat{x}^{(l,c)}(t_k)$ normalized densities in the counterflow case, corresponding to the predictions for groups $l=1,2$ in Figs.~\ref{fig:Counterflow_densities_MVAR} and \ref{fig:Counterflow_densities_LSTM}. Panels (a)--(d) show the spatial distribution of the absolute error $\lvert x^{(l,c)}(t_k)-\hat{x}^{(l,c)}(t_k)\rvert$ in $\Omega$ at two indicative time steps $t_{250}=62.5\,\mathrm{s}$ (left column) and $t_{950}=237.5\,\mathrm{s}$ (right column). Panels (a,b) correspond to MVAR(10), and panels (c,f) to LSTM(10) models.}
    \label{fig:Counterflow_Error}
\end{figure}

\end{document}